\documentclass[]{fairmeta}
\usepackage{amsmath}
\usepackage{amssymb}
\usepackage{array}
\usepackage{tabularx}
\usepackage{enumitem}
\usepackage{wrapfig}
\usepackage{rotating}
\newcommand{\ms}[2]{$#1_{\pm #2}$}
\newcommand{\emj}[1]{{\scriptsize\textsf{[#1]}}}
\newcommand{\method}{\textsc{MIMESIS}}

\newtcbtheorem{Example}{\bfseries Example}{enhanced,drop shadow={black!50!white},
  coltitle=black,
  top=0.3in,
  attach boxed title to top left=
  {xshift=1.5em,yshift=-\tcboxedtitleheight/2},
  boxed title style={size=small,colback=blue!20}
}{Example}
\newtcolorbox{TranscriptBox}[1]{%
  enhanced,breakable,colback=white,colframe=black!30,boxrule=0.4pt,
  arc=1pt,left=6pt,right=6pt,top=5pt,bottom=5pt,
  title={#1},fonttitle=\bfseries,coltitle=black,
  attach boxed title to top left={xshift=6pt,yshift=-\tcboxedtitleheight/2},
  boxed title style={size=small,colback=blue!10,colframe=black!30}}
\newcommand{\bms}[2]{%
  \ensuremath{\mathbf{#1}_{\boldsymbol{\pm}\mathbf{#2}}}%
}

\usepackage{amsmath,amssymb,amsthm}
\theoremstyle{plain}
\newtheorem{csdtheorem}{Theorem}[section]

\newtheorem{csdproposition}[csdtheorem]{Proposition}

\title{{\raisebox{-0.32em}{\includegraphics[width=.08\linewidth]{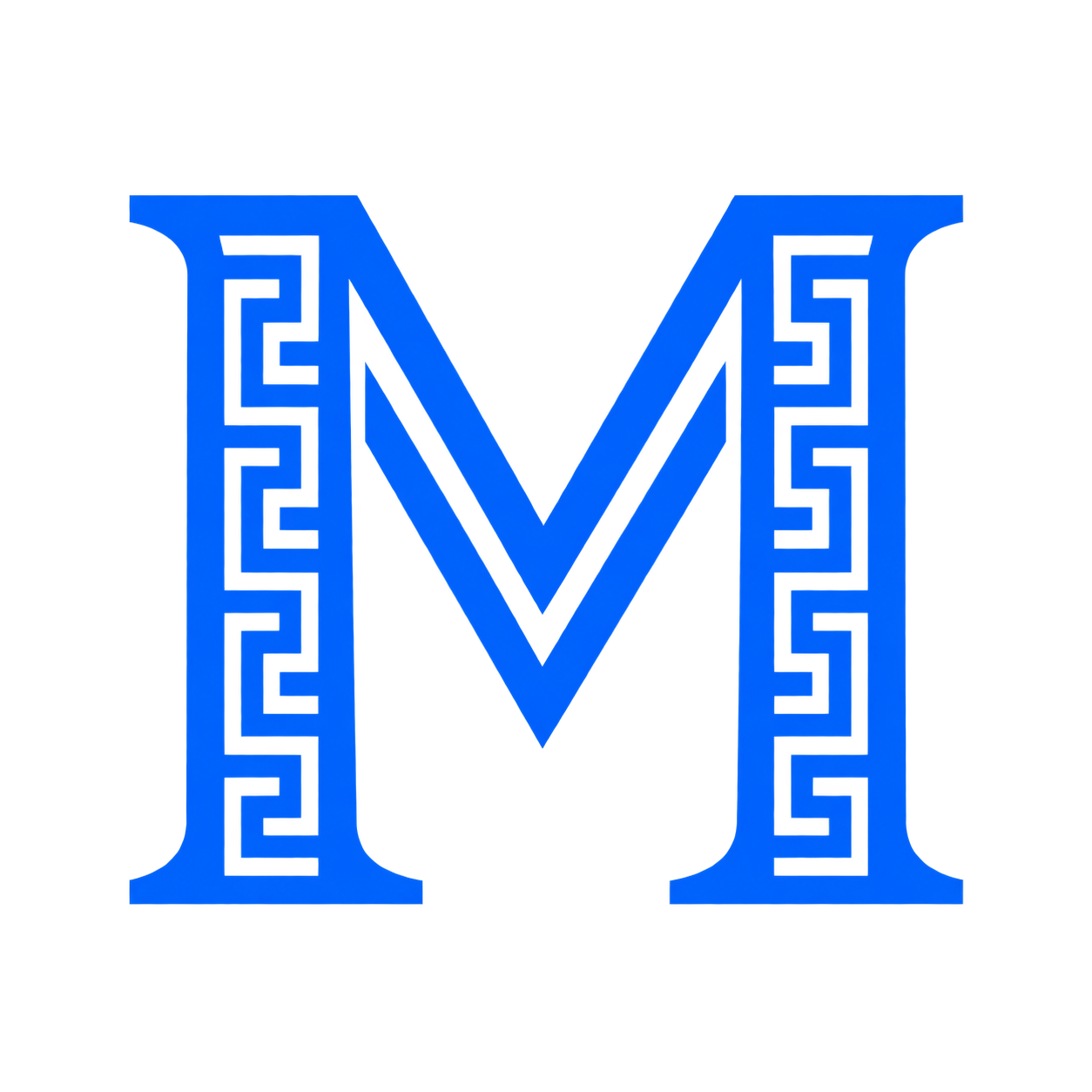}}}IMESIS: Learning User Simulators as Training Environments for Interactive Agents}

\author[1,2,*]{Hoang Phan}
\author[1]{Dat Huynh}
\author[]{Andrey Zhmoginov}
\author[1]{Qi Zeng}
\author[1]{Wancen Mu}
\author[1]{Yue Cao}
\author[1]{Shengjie Bi}
\author[1]{Yun He}
\author[1,3,*]{Changdae Oh}
\author[1]{Deren Lei}
\affiliation[1]{Meta Superintelligence Labs}
\affiliation[2]{New York University}
\affiliation[3]{University of Wisconsin - Madison}
\contribution[*]{Work done at Meta}

\abstract{{
Training and evaluating interactive language agents typically requires rich user interactions, yet collecting human feedback is expensive and difficult to scale. Simulated users offer a scalable alternative, but they must both resemble real user behavior and provide useful learning experiences for agents. In contrast, most agent-training frameworks rely on off-the-shelf assistant LLMs, whose helpfulness can make them overly cooperative, explicit, and behaviorally homogeneous compared with real users. We introduce MIMESIS, a purpose-built user simulator trained on human conversations with explicit reasoning supervision and 13 realistic behavioral patterns derived from real user interactions. Empirically, our 9B model achieves a SOUL-Index of 65.7, surpassing the strongest frontier model. Compared with  Claude-Opus-5, the strongest baseline on RealUserSim and SimulatorArena, MIMESIS improves behavioral fidelity by 13.4 points and reduces Turing distance by 3.6 points, respectively. We then freeze the simulator and train an agent by interacting with the frozen simulator using multi-turn reinforcement learning.  Across eight environments, training with MIMESIS yields better agent performance than training with GPT-5.5 under all nine unseen user simulators, demonstrating stronger generalization to new  user simulators. Moreover, we propose Coached On-Policy Self-Distillation (CSD), which leverages simulator-generated private reasoning traces and subsequent utterances as feedback on how well the agent addresses user needs. A coach converts this information into concise coaching notes that describe how the agent can better anticipate user needs and adapt its behavior over the course of an interaction. CSD turns this feedback into dense, token-level supervision beyond sparse task rewards, yielding further gains across all nine evaluation user models. %Our code is available \href{https://anonymous.4open.science/r/mimesis-DD3C/}{here}. 
}}

\correspondence{Hoang Phan at \email{hvp2011@nyu.edu}}
\metadata[Project Page]{\url{https://viethoang1512.github.io/mimesis/}}
\metadata[Code]{\url{https://github.com/VietHoang1512/mimesis}}
\begin{document}
\maketitle

\begin{figure}[ht]
\centering
\vspace*{-2mm}
\includegraphics[width=1.\textwidth]{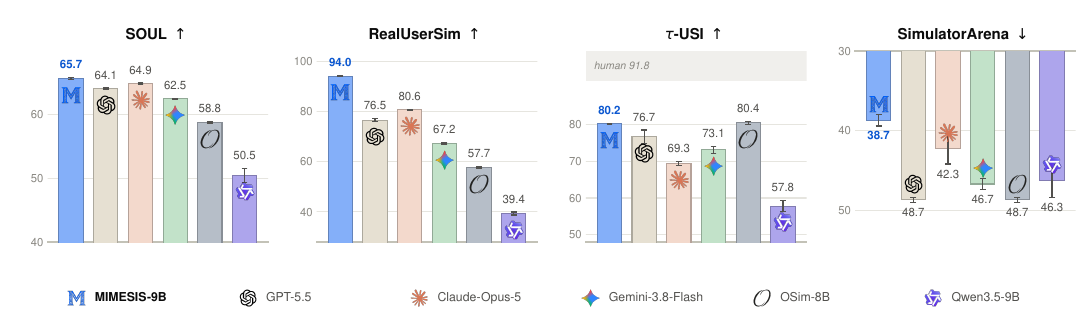}
\caption{\textbf{User-simulation performance across four benchmarks.}
We compare MIMESIS-9B with frontier API models (GPT-5.5, Claude-Opus-5, and Gemini-3.8-Flash) and open-weight simulators (Osim-8B and Qwen3.5-9B)  on SOUL~\citep{zhou2026odyssim}, RealUserSim~\citep{zhu2026realusersim}, $\tau$-USI~\citep{zhou2026sim2real}, and SimulatorArena~\citep{dou2025simulatorarena}. Higher is better for SOUL, RealUserSim, and $\tau$-USI, while lower Turing distance \citep{wang2026learning} is better for SimulatorArena. The shaded band in $\tau$-USI denotes the human inter-annotator score (91.8).}
\label{fig:main}
\end{figure}

\clearpage
\section{Introduction}
\label{sec:introduction}
Large language models are increasingly deployed as interactive agents that operate over
extended trajectories in partially observable and stateful environments
\citep{liu2024agentbench,ma2024agentboard,guan2026evaluating}. In these settings, success
cannot be reduced to producing a correct response at an isolated turn. Agents must
maintain conversational and environmental state, preserve instructions and constraints,
identify missing information, incorporate feedback, and condition later actions on
earlier decisions
\citep{lu2025toolsandbox,yao2024taubench}. Multi-turn interaction therefore provides a
more faithful evaluation of agentic competence than single-turn tasks, exposing
limitations in long-horizon reasoning, state tracking, instruction retention, and error
recovery
\citep{ma2024agentboard,lu2025toolsandbox,deshpande2025multichallenge}.  Recent work has therefore increasingly evaluated agents through multi-turn interaction with users, tools, and environments, spanning general reasoning and tool use~\citep{wang2024mint}, conversational web navigation~\citep{deng2024multi}, realistic tool--agent--user workflows~\citep{yao2024taubench}, and stateful tool execution~\citep{lu2025toolsandbox}.

% place at the START of a paragraph, not inside a list or right before \section
\begin{wrapfigure}{r}{0.59\textwidth}
  \centering
  \vspace*{-8mm}
\includegraphics[width=0.59\textwidth]{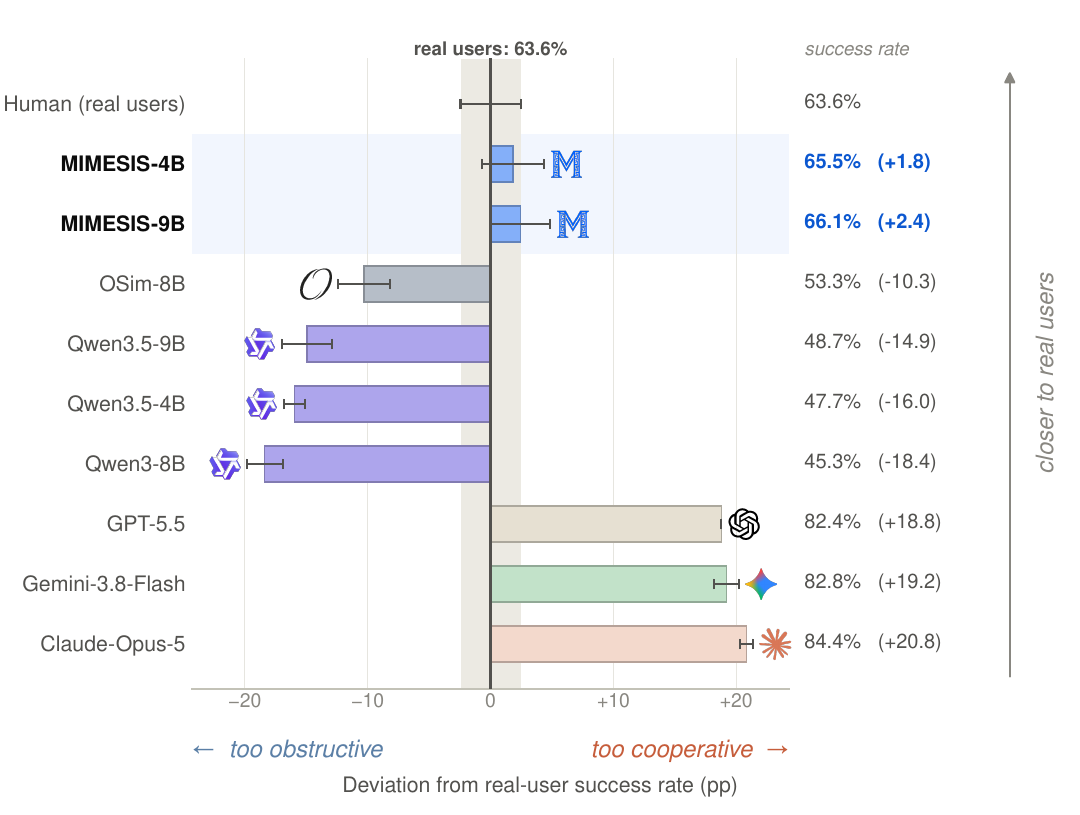}
  \caption{%
    \textbf{Off-the-shelf LLMs misrepresent agentic task difficulty.}
    Task success rate of a fixed \textsc{gpt-5.5} agent on $\tau$-bench when each simulator plays the user role.  Frontier API users make the task substantially
easier than interactions with real users, whereas pretrained models make it
harder. Both serve as poor proxies for the user interactions used to train or evaluate agents while purpose-built simulators more closely reproduce the human-user
success rate. %\deren{add sth like not ideal for agent training} %Purpose-built simulators close most of the gap; \textsc{MIMESIS} closes nearly all of it, landing within $1.8$ and $3.2$ points of the human rate. Reproducing real-user \emph{difficulty}, not merely fluent dialogue, is something a simulator has tobe trained for.% 
  }
  \vspace*{-2mm}
  \label{fig:tau-success}
\end{wrapfigure}

This move toward interactive evaluation introduces a practical challenge: in many tasks, the trajectory cannot be determined by the agent and environment alone, since users may  provide missing information, answer clarifications, or react to intermediate outcomes. While human users provide the most direct source of such interactions, collecting human trajectories is costly, difficult to reproduce, and hard to scale for evaluation and reinforcement learning. User simulators offer a practical alternative and have become common in interactive benchmarks and, increasingly, agent post-training \citep{yao2024taubench,qian2025userrl,suh2026utility}. A common approach is to prompt instruction-tuned language models to role-play users, but this creates a mismatch: assistant models are optimized to be helpful and explicit, whereas real users may be ambiguous, impatient, or uncooperative. Accordingly, assistant LMs tend to produce overly structured and cooperative interactions \citep{naous2025userlm}, while comparisons with real users reveal systematic differences in information disclosure, ambiguity, frustration, and interaction difficulty \citep{zhou2026sim2real}. Moreover, properties such as willingness to clarify, patience, and responses to agent errors can substantially alter measured agent performance \citep{shim2025noncollaborative,chopra2026ppol,zhu2026realusersim}. User simulators therefore affect not only dialogue style but also the
distribution of trajectories encountered by an agent during training and
evaluation. Figure~\ref{fig:tau-success} illustrates this mismatch on
$\tau$-bench. A fixed GPT-5.5 agent succeeds on 63.6\% of interactions
with real users, whereas the three frontier API simulators raise its success
rate to 82.4--84.4\% and pretrained Qwen3.5-4B/9B reduce it by 16.0 and
14.9 percentage points, respectively. By contrast, purpose-built simulators are
substantially better calibrated:  every trained simulator deviates less than every off-the-shelf model, with at most 10.3 points, versus at least 14.9. This result suggests that simulator
quality cannot be assessed from surface naturalness alone; preserving the
interaction difficulty induced by human users is also important when the
simulator determines the trajectories used for agent training or evaluation.

Research on user simulation and interactive-agent training has largely addressed
different sides of this problem. Purpose-built user models improve next-turn realism,
latent-state modeling, or behavioral simulation
\citep{naous2025userlm,wu2026humanlm,wang2026learning,sun2026ditto,zhou2026odyssim}.
In parallel, UserRL~\citep{qian2025userrl} studies how to optimize agents from multi-turn trajectories
generated by a simulator. The remaining question is whether a simulator can be both behaviorally realistic and useful for learning. We therefore
evaluate it at two levels: fidelity to human interactions and the robustness of agents
trained on its trajectories when the evaluation user changes.
%Recent work on the downstream utility of user simulators makes the second criterion explicit~\citep{suh2026utility}; the remaining challenge is to develop and evaluate both stages together.

\begin{figure*}[t]
    \centering
    \includegraphics[width=.95\columnwidth]{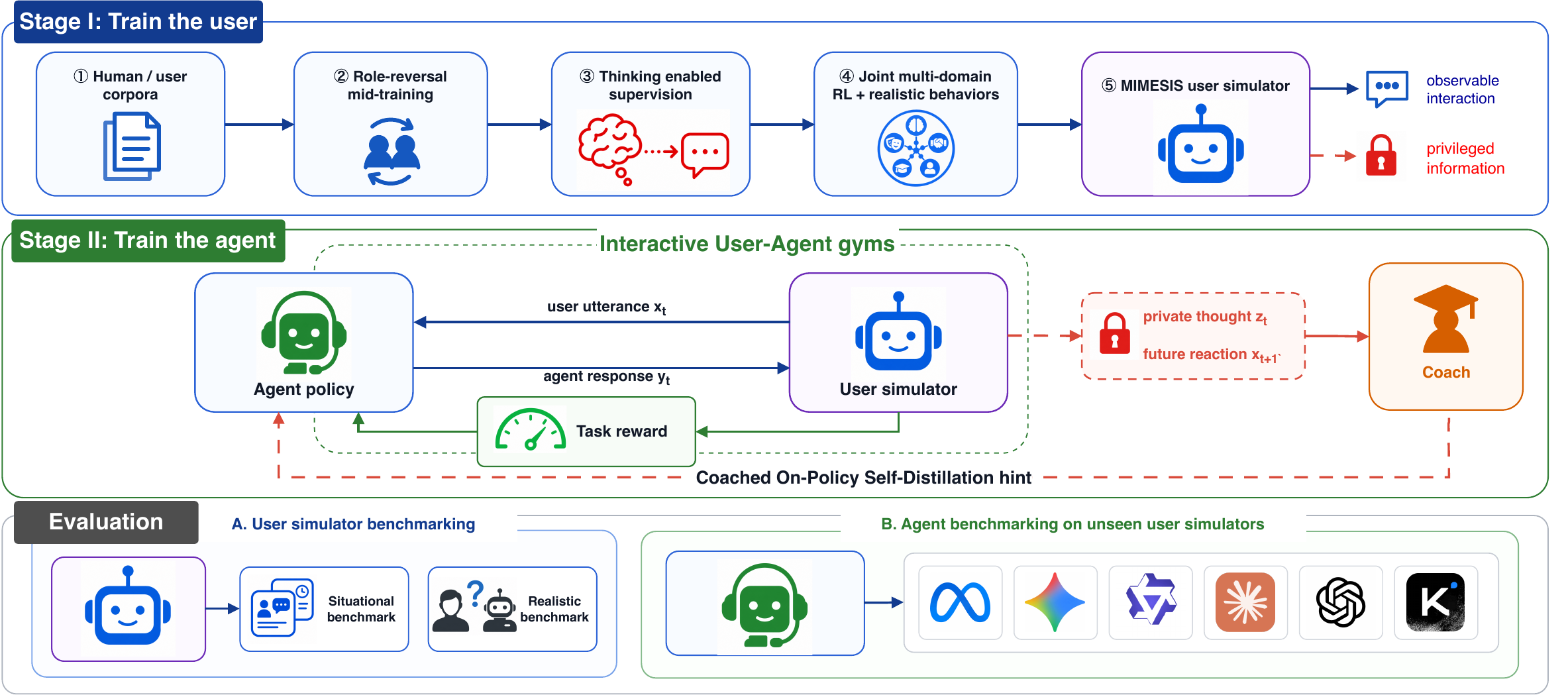}
\caption{\textbf{Overview of our proposed method.}
We couple user-simulator learning with simulator-driven agent post-training.
\textbf{Stage I} trains MIMESIS through user-side mid-training, ThoughtTrace reasoning supervision, and joint RL with a realistic-behavior objective.
\textbf{Stage II} freezes the learned simulator and uses it to generate multi-turn experience for agent RL.
The agent acts only on the observable conversation, while simulator-generated thoughts and reactions provide privileged, training-only feedback to the coach, which converts them into dense guidance for the agent.
%We evaluate the framework at both levels: the simulator's fidelity as a model of users and the resulting agent's ability to generalize across held-out user simulators.
}
    \label{fig:pipeline}
\end{figure*}

We propose a two-stage training approach that connects simulator learning to agent
post-training (Figure~\ref{fig:pipeline}).  In the first stage, we train 4B and 9B user models on human conversations and a broad collection of human-simulation tasks. Most existing released simulators are trained on human conversation data that provide only observable utterances and no explicit reasoning trajectories, making it difficult to model the latent user state underlying each response. To address this limitation, we enable an explicit reasoning mode by appending a supervised fine-tuning stage using ThoughtTrace~\citep{jin2026thoughttrace}, which pairs user messages with self-reported motivations and reactions to preceding assistant responses. We use these annotations as observable proxies for otherwise unobserved user state and supervise a private reasoning trace before the public utterance. We further derive from ThoughtTrace a taxonomy of 13 realistic user behaviors reflected in users' latent thoughts and conversational patterns, and instantiate these behaviors during joint multi-domain reinforcement learning to expose the simulator to a broader and more challenging distribution of interactions.

In the second stage, we freeze the learned user simulator and use it as an interactive environment for multi-turn agent reinforcement learning. After each sampled agent response, the simulator generates a private thought and the next public user utterance. We leverage this privileged feedback during training through \emph{Coached On-Policy Self-Distillation} (CSD). CSD constructs feedback-conditioned coaching signals for each sampled agent response and uses a self-teacher conditioned on this additional context to provide dense token-level guidance to the policy. Crucially, this privileged information is available only during optimization: both the rollout policy and the deployed agent condition solely on the observable conversation history. Our contributions are depicted as follows:

\begin{itemize}
    \item We identify a behavioral mismatch that limits off-the-shelf LLMs as human proxies for agent training and evaluation: their simulated interactions distort task difficulty, yielding agent success rates that diverge substantially from those observed with human users. We therefore propose a two-stage approach that first learns a realistic user simulator with explicit behavioral modeling and then uses the frozen simulator as an environment for interactive agent post-training.

    \item We introduce \method{}, a purpose-built user simulator with 4B and 9B parameters, combining user-side adaptation and explicit reasoning supervision with a new realistic-behavior training objective. We derive 13 recurring interaction patterns from real user conversations and train the simulator to express these behaviors appropriately. Our 9B model surpasses frontier models on SOUL-Index, RealUserSim, $\tau$-USI, and SimulatorArena. Ablations further establish the contribution of reasoning supervision and realistic behavior training to simulator quality.

    \item We demonstrate that training with \method{} improves agent generalization beyond the training simulator. Across eight environments, agents trained through multi-turn reinforcement learning with \method{} outperform those trained with off-the-shelf GPT-5.5 under all nine unseen evaluation user models. We further introduce \emph{Coached On-Policy Self-Distillation} (CSD), which uses simulator-generated thoughts and subsequent user utterances to construct retrospective coaching about user needs and interaction feedback. A teacher conditioned on this coaching provides dense, token-level supervision, yielding additional gains under all nine evaluation users while the deployed agent conditions only on public dialogue.
\end{itemize}
\section{Related Work}
\label{sec:related}
% \paragraph{Multi-turn agentic interaction.}
% Evaluation of language agents has increasingly shifted from isolated response or
% function-call accuracy toward competence over complete interaction trajectories.
% Early benchmarks such as AgentBoard and MINT emphasize multi-round interaction
% in partially observable environments and show that strong static performance does
% not necessarily translate to effective interactive behavior
% \citep{ma2024agentboard,wang2024mint}.
% Subsequent work makes this setting increasingly realistic:
% ToolSandbox, BFCL, and DialogTool evaluate persistent environment state,
% dependencies across tool calls, and long-horizon execution
% \citep{lu2025toolsandbox,patil2025bfcl,wang2025dialogtool},
% while $\tau$-bench and $\tau^2$-bench couple tool use with conversational
% coordination, with the latter allowing both the agent and user to act on a shared
% environment \citep{yao2024taubench,barres2025tau2}.
% Recent benchmarks further isolate the capabilities required by such trajectories,
% including planning and state tracking in LUMINA, adaptation to changing tool
% conditions in CostBench, and preference elicitation in multi-turn task completion
% \citep{rakhsha2026lumina,liu2026costbench,cheng2026travelbench}.
% In parallel, agentic post-training increasingly optimizes policies directly over
% interactive rollouts, using scalable multi-turn environments, finer-grained credit
% assignment, and online environment feedback
% \citep{du2026codegym,ding2026gtpo,sun2026karl}.

\textbf{Multi-turn agentic interaction.}
Evaluation of language agents has increasingly shifted from isolated response or
function-call accuracy toward competence over complete interaction trajectories.
Early benchmarks such as AgentBoard and MINT emphasize multi-round interaction
in partially observable environments and show that strong static performance does
not necessarily translate to effective interactive behavior
\citep{ma2024agentboard,wang2024mint}. ToolSandbox, BFCL, and DialogTool evaluate persistent state, tool-call dependencies, and long-horizon execution
\citep{lu2025toolsandbox,patil2025bfcl,wang2025dialogtool},
while $\tau$-bench and $\tau^2$-bench combine tool use with conversational coordination; the latter allows both agents and users to act on a shared environment
\citep{yao2024taubench,barres2025tau2}.
Recent benchmarks examine planning and state tracking in LUMINA, adaptation to changing tool conditions in CostBench, and preference elicitation during task completion
\citep{rakhsha2026lumina,liu2026costbench,cheng2026travelbench}.
In parallel, agentic post-training optimizes policies over interactive rollouts through scalable multi-turn environments, finer-grained credit assignment, and online environment feedback
\citep{du2026codegym,ding2026gtpo,sun2026karl}.

\textbf{Behavioral variation and difficult users.}
Several lines of work address the cooperative and homogeneous behavior of prompted LLM users.
Non-Collaborative User Simulators introduces four difficult interaction patterns while preserving the information required for task completion~\citep{shim2025noncollaborative}.
Persona Policies searches for human-like, task-preserving behavior policies and shows that training on the resulting trajectories improves robustness~\citep{chopra2026ppol}.
RealUserSim grounds behavioral profiles in authentic interactions and documents Directive Amplification, where hand-authored directives elicit exaggerated, backbone-dependent behavior~\citep{zhu2026realusersim}.
UserIDA learns explicit control over local interaction intent~\citep{wang2026userida}.
%These works make the evidentiary burden clear: a useful behavior control must be interventionally reliable, distinct from neutral behavior, calibrated to real interaction, and prevented from changing task feasibility.
%Our contribution lies in internalizing such controls in the learned simulator and using the resulting model for agent RL, not in the raw number of behavior labels.

\textbf{Simulator-driven agent training.}
$\tau$-bench established rule-grounded evaluation with LLM users~\citep{yao2024taubench}, while UserRL studies how multi-turn reward assignment, trajectory scoring, and simulator choice affect agent optimization~\citep{qian2025userrl}.
UserLM-R1 subsequently uses a learned strategic simulator for downstream task-agent RL~\citep{zhang2026userlmr1}.
Most closely, ~\citet{suh2026utility} hold the assistant training procedure fixed, vary the simulator, and validate the resulting policies with other simulators and 283 human participants.
%Our study follows the same downstream-utility principle but asks a different controlled question: whether a  simulator jointly trained for broad user behaviors can outperform a frontier general-purpose model as the sole training user when neither appears in the held-out evaluation panel.

% \paragraph{Cross-partner generalization.}
% Training with one interactive partner can overfit conventions specific to that partner, a problem studied in zero-shot coordination and multi-agent learning~\citep{hu2020otherplay}.
% Simulator Collapse shows an analogous failure for LLM user environments: single-simulator RL can concentrate the agent policy around strategies that exploit the simulator's modal behavior, even as held-out and human performance decline~\citep{yu2026collapse}.
% Our six-user panel is motivated by this failure mode.
% Unlike population co-training, however, our intervention changes the quality and behavioral support of one learned training simulator.
% We therefore report training-partner reward, held-out reward, and policy entropy separately and avoid claiming that a single endpoint proves resistance to collapse.

\textbf{Learning from textual feedback and privileged context.}
ThoughtTrace demonstrates that users' stated reasons and reactions contain information that is difficult to reconstruct from the public transcript and useful for downstream assistant training~\citep{jin2026thoughttrace}.
SDPO~\citep{hubotter2026sdpo} distills a feedback-conditioned self-teacher into a policy that acts without feedback, and on-policy context distillation studies the broader problem of transferring behavior from an augmented-context policy~\citep{ye2026opcd}.
~\citet{song2026rltf} and ~\citet{shi2026experiential} similarly internalize textual critiques or reflections.
Ditto applies related ideas to the simulator itself by optimizing feedback-conditioned refined rollouts jointly with the base policy~\citep{sun2026ditto}.
Our proposed CSD differs operationally by re-scoring the already sampled agent response rather than decoding a refinement.
% Its most distinctive information source is state generated inside the counterparty user after observing the agent action, although our current experiments do not yet show that this state improves over public-transcript coaching.

% \paragraph{Evaluating simulator validity.}
% Dialogue research has long argued that a simulator should be assessed both directly and through the policies it trains~\citep{shi2019usersim}.
% MirrorBench separates linguistic human-likeness from task success~\citep{hathidara2026mirrorbench}, SimulatorArena asks whether simulator-based assistant rankings agree with real-user preferences~\citep{dou2025simulatorarena}, and large-scale comparisons reveal systematic gaps between LLM users and people~\citep{zhou2026sim2real}.
% We adopt the same separation: SOUL-Index and next-turn discrimination measure properties of the simulator, while the held-out panel measures a property of the policy produced by training against it.

% \authorcomment{Re-check all concurrent-work descriptions and the closest-work matrix immediately before submission. In particular, verify whether any new version reports end-to-end simulator-to-agent RL, private-user-state distillation, or held-out-user transfer that changes the novelty boundary.}

\section{Problem Setup and Overview}
\label{sec:problem}

We study the problem of learning an interactive agent from multi-turn
experience with a simulated user. The user simulator is trained from human conversations and simulation-task feedback.
Once trained, it supplies the user turns in the conversations used to optimize the agent.

\textbf{Learning a user simulator.}
Given a task context $c$ and a conversation history $H$, a user simulator $u_\phi(x\mid c,H)$ defines a distribution over the next user utterance $x$.
Here, $c$ specifies the interaction setting, while $H$ contains the preceding conversation.
Human--agent dialogues provide examples of how users respond in these contexts, forming the basis for learning the user model~\citep{naous2025userlm,wu2026humanlm}.
We further train the simulator on a distribution of human-simulation tasks $\mathcal D_{\mathrm{sim}}$.
Each task supplies an input context and an evaluation rule for the generated response or interaction~\citep{sun2026ditto}.
Let $\rho$ denote this generated output, $P_\phi(\rho\mid c)$ its distribution when running the simulator in task $c$, and $R_{\mathrm{sim}}(c,\rho)$ the corresponding simulation reward.
The simulator's reinforcement-learning objective is
\begin{equation}
    J_{\mathrm{sim}}(\phi)
    =
    \mathbb E_{c\sim\mathcal D_{\mathrm{sim}}}
    \mathbb E_{\rho\sim P_\phi(\cdot\mid c)}
    \left[R_{\mathrm{sim}}(c,\rho)\right].
    \label{eq:problem-simulator-objective}
\end{equation}
For a single-turn task, $\rho$ is one response; for an interactive task, it is a rollout with the counterpart specified by that task.
The parameters $\phi$ are optimized to maximize this objective.

\textbf{Interaction with the learned simulator.}
We fix the trained simulator $u_{\hat\phi}$ and use it as the user in multi-turn interactions with an agent $\pi_\theta$.
Let $H_t$ denote the history available to the agent through the current user utterance $x_t$.
At each conversational turn, the agent generates a response and the simulator generates the next user utterance:
\begin{equation}
    y_t\sim\pi_\theta(\cdot\mid H_t),
    \qquad
    x_{t+1}\sim u_{\hat\phi}(\cdot\mid c,H_t,y_t).
    \label{eq:user-transition}
\end{equation}
Both utterances are added to the conversation history.
Tool calls and their outcomes, when present, are also included in the history.
The interaction continues until the task's stopping condition or turn limit is reached, yielding a trajectory $\tau$.
The simulator therefore determines how the conversation continues in response to the agent's actions~\citep{suh2026utility}.

\textbf{Learning the agent.}
Let $\mathcal D_{\mathrm{agent}}$ denote the distribution of agent-training tasks, and let $R_{\mathrm{agent}}(c,\tau)$ be the task reward assigned to a trajectory.
Following multi-turn reinforcement learning with simulated users~\citep{qian2025userrl}, we optimize
\begin{equation}
    J_{\mathrm{agent}}(\theta;\hat\phi)
    =
    \mathbb E_{c\sim\mathcal D_{\mathrm{agent}}}
    \mathbb E_{\tau\sim P_{\theta,\hat\phi}(\cdot\mid c)}
    \left[R_{\mathrm{agent}}(c,\tau)\right],
    \label{eq:agent-objective}
\end{equation}
where $P_{\theta,\hat\phi}$ is the trajectory distribution induced by the agent, the frozen simulator, and the task environment.
Agent training updates $\theta$ to maximize this expected reward while keeping $\hat\phi$ fixed.
The two objectives assess different aspects of the interaction: $R_{\mathrm{sim}}$ evaluates the simulated user's output, whereas $R_{\mathrm{agent}}$ evaluates the agent's task performance.

\section{MIMESIS: Learning a Realistic User Simulator}
\label{sec:simulator}

We train \method{} to model user behavior and provide realistic interactions for agent reinforcement learning.
Both the 4B and 9B models with Qwen3.5 \citep{qwen35blog} backbones follow the same procedure: user-side mid-training on human conversations, thought-augmented supervised fine-tuning, and joint multi-domain reinforcement learning that includes an explicit objective for realistic user behaviors.

\subsection{Learning User Responses and Reasoning}

We adapt the language model to the user role through supervised mid-training on human--assistant conversations.
Following user language modeling~\citep{naous2025userlm}, we use the preceding dialogue as context and the human user's next utterance as the prediction target.
Assistant turns remain part of the context, while the model learns to continue the conversation from the user's perspective.

\textbf{Challenges.}
Ordinary conversation logs record what users say but generally omit why they say it and how they interpret the assistant's response.
These unspoken reasons and reactions can contain information that is absent from the corresponding user messages~\citep{jin2026thoughttrace}.
Consequently, mid-training on user utterances alone provides no explicit target for the reasoning that precedes them.
When the target consists only of the public utterance, the model is trained to produce that utterance directly.
For a simulator intended to reason before responding, this creates a potential mismatch: response-only supervision may suppress explicit reasoning-trace generation, even though the missing traces reflect how the data were recorded.
There is also a semantic limitation: matching a user's wording does not necessarily capture the motivation behind the response.
~\citet{wu2026humanlm}  identify a related failure of response imitation, where models reproduce linguistic patterns while missing the user states expressed in the reference response.

\textbf{Thought-augmented supervision.}
To provide explicit supervision for simulator reasoning, we incorporate ThoughtTrace, which pairs human--assistant conversations with users' self-reported reasons for sending messages and reactions to assistant responses~\citep{jin2026thoughttrace}.
For annotated examples, we use these reports to supervise a private reasoning trace before the simulator’s public response, while retaining the observed user utterance as the response target.
\citet{wu2026humanlm}  obtain the state-alignment signal by evaluating generated states against observed user responses. ThoughtTrace additionally gives us access to users' own reports.
During interaction, the reasoning trace is kept separate from the public user utterance returned to the agent.
\subsection{Joint Reinforcement Learning with Realistic Behaviors}
\label{sec:joint-rl}

Starting from the mid-trained checkpoint, we optimize a single user simulator across the SOUL environments used by OdysSim~\citep{zhou2026odyssim}.
OdysSim trains a separate RL expert for each task and then distills selected expert trajectories into a unified model.
We train the shared simulator directly on all domains throughout RL, with rollouts from every domain updating the same parameter set.
This follows the joint multi-task training approach also used by Ditto~\citep{sun2026ditto}.
Each environment supplies its own interaction protocol and task-specific reward.

Let $\mathcal D$ denote the set of training environments, $p(d)$ denote their sampling distribution, and $P_\phi^d$ denote the rollout distribution induced by simulator $u_\phi$ in environment $d$.
We maximize
\begin{equation}
    J_{\mathrm{joint}}(\phi)
    =
    \mathbb E_{d\sim p(d)}
    \mathbb E_{\tau\sim P_\phi^d}
    \left[r_d(\tau)\right]
    \label{eq:sim-objective}
\end{equation}
using Group Relative Policy Optimization (GRPO)~\citep{shao2024deepseekmath}, where $r_d$ is the reward for environment $d$.
Alongside the SOUL environments, the mixture includes the realistic-behavior task described below as a separate domain. Its rollouts contribute to the same simulator updates.

\subsection{Learning Realistic User Behaviors}
\label{sec:behavior}

Real users may leave preferences unstated, revise requirements, or provide incomplete answers to clarification questions.
Such behaviors affect the information available to an agent and the decisions required to complete a task~\citep{shim2025noncollaborative}.
Here, we make the generation of those interactions an explicit simulator-training objective using a taxonomy of 13 realistic behaviors derived from recurring patterns in ThoughtTrace~\citep{jin2026thoughttrace}, including hidden evaluation criteria, incremental goalpost shifting, and clarification noncooperation. We explicitly train MIMESIS to reproduce such interaction patterns rather than relying on them to emerge from response imitation.
%Each training example specifies one behavior or a neutral condition; one third use the neutral condition to retain interactions without an imposed behavioral pattern.

We instantiate this task in ABCD customer-support scenarios~\citep{chen2021abcd}, which provide structured customer records and service policies.
The simulator interacts with a frozen support agent under the selected condition, grounding its behavior in the scenario facts. For example, withheld information must be available to the user, and an infeasible request must conflict with a stated policy.
We score each rollout for the expression of the target behavior, temporal placement, naturalness, and task consistency.
Temporal placement rewards context-appropriate expression of a behavior across the conversation.
The reward combines these four scores and penalizes responses that describe the requested behavior instead of expressing it through interaction.
Appendices~\ref{app:behavior-taxonomy} and~\ref{app:behavior-rollouts} provide the taxonomy, rollout protocol, and reward details.

\subsection{User-simulator evaluation}
\label{sec:exp-simulator}

\textbf{Evaluation setup.}
We evaluate \method{} at 4B and 9B parameters against frontier API models, released user simulators, and pretrained backbones.
SOUL-Index measures simulation capability across conversational interaction (CONV), social simulation (SS), cognition (COG), role play (ROLE), and evaluation/judgment (EVAL)~\citep{zhou2026odyssim}.
We assess trajectory-level behavioral fidelity with RealUserSim PT3~\citep{zhu2026realusersim}, behavioral and outcome alignment with $\tau$-USI~\citep{zhou2026sim2real}, and message similarity and Turing distance with SimulatorArena~\citep{dou2025simulatorarena}.
We report both aggregate and component scores to characterize these complementary aspects of simulation quality.
 Appendix~\ref{app:evaluation} defines the metrics and describes the evaluation settings along with additional experimental results.

\begin{table}[!ht]
\centering
\resizebox{.95\linewidth}{!}{%
\begin{tabular}{llcccccc}
\toprule
Type & Model & CONV & SS & COG & ROLE & EVAL & Overall \\
\midrule
Frontier & Claude-Opus-5  & \ms{57.4}{0.2} & \ms{70.4}{0.3} & \ms{82.0}{0.1} & \textbf{\bms{64.6}{0.1}} & \ms{69.1}{0.2} & \ms{64.9}{0.1} \\
 & GPT-5.5 & \ms{56.7}{0.6} & \textbf{\bms{70.9}{0.4}} & \ms{86.1}{0.5} & \ms{62.3}{0.2} & \ms{68.3}{0.3} & \ms{64.1}{0.1} \\
 & Gemini-3.8-Flash & \ms{56.0}{0.5} & \ms{66.9}{0.4} & \ms{85.6}{0.4} & \ms{57.3}{0.0} & \ms{68.7}{0.2} & \ms{62.5}{0.1} \\
\midrule
Released simulator & Osim-8B & \ms{60.0}{0.1} & \ms{39.1}{0.1} & \ms{74.5}{0.8} & \ms{54.9}{0.2} & \ms{68.7}{0.1} & \ms{58.8}{0.1} \\
 & Osim-4B & \ms{58.5}{0.2} & \ms{35.0}{0.4} & \ms{72.9}{1.0} & \ms{51.4}{0.3} & \ms{66.5}{0.4} & \ms{56.0}{0.3} \\
 & Ditto-8B & \ms{61.7}{0.2} & \ms{46.4}{0.3} & \ms{78.2}{0.3} & \ms{44.1}{0.2} & \ms{63.9}{0.5} & \ms{53.9}{0.3} \\
 & HumanLM-8B & \ms{43.8}{0.1} & \ms{59.8}{0.4} & \ms{68.8}{1.2} & \ms{41.6}{0.3} & \ms{62.7}{0.5} & \ms{49.6}{0.4} \\
 & Sotopia-7B & \ms{41.1}{0.8} & \ms{59.7}{0.4} & \ms{29.8}{0.8} & \ms{33.5}{0.1} & \ms{55.1}{1.0} & \ms{38.2}{0.6} \\
\midrule
Pretrained & Qwen3.5-9B & \ms{48.9}{0.4} & \ms{54.2}{0.6} & \ms{64.3}{3.8} & \ms{44.8}{0.3} & \ms{64.7}{0.6} & \ms{50.5}{1.1} \\
 & Qwen3-8B & \ms{44.8}{0.3} & \ms{58.7}{0.5} & \ms{68.9}{0.4} & \ms{39.7}{0.4} & \ms{64.1}{0.7} & \ms{49.7}{0.2} \\
 & Qwen3.5-4B & \ms{45.9}{0.4} & \ms{49.7}{2.3} & \ms{55.2}{3.2} & \ms{38.9}{0.2} & \ms{60.6}{0.6} & \ms{45.7}{0.5} \\
\midrule
Ours & \textbf{MIMESIS-9B} & \textbf{\bms{75.1}{0.1}} & \ms{66.4}{0.2} & \textbf{\bms{86.2}{0.5}} & \ms{61.1}{0.2} & \textbf{\bms{69.7}{0.6}} & \textbf{\bms{65.7}{0.2}} \\
 & \textbf{MIMESIS-4B} & \ms{74.4}{0.3} & \ms{65.2}{0.2} & \ms{79.8}{0.4} & \ms{58.8}{0.1} & \ms{68.6}{0.1} & \ms{63.7}{0.0} \\
\bottomrule
\end{tabular}
}
% \vspace*{-2mm}
\caption{\textbf{Simulation capability on the five SOUL-Index axes}~\citep{zhou2026odyssim}. Entries are means $\pm$
standard errors over three independently seeded evaluation runs.
Higher is better. \textbf{Bold }marks the highest model score in each column. %Aggregation and replication details are recorded in Appendix~\ref{app:evaluation}.
}
% \vspace*{-2mm}
\label{tab:main}
\end{table}

\textbf{Broad simulation capability.}
\label{sec:results-simulator}
\method{}-9B achieves the highest overall SOUL-Index of 65.7, surpassing Claude-Opus-5  (64.9), GPT-5.5 (64.1), and the strongest released simulator, Osim-8B (58.8).
\method{}-4B scores 63.7, also exceeding all released simulators.
Relative to their pretrained backbones, the 9B and 4B models improve by 15.2 and 18.0 points, respectively.
The largest advantage is in conversational simulation: \method{}-9B and \method{}-4B score 75.1 and 74.4, exceeding the strongest baseline on this axis, Ditto-8B (61.7), by 13.4 and 12.7 points.
\method{}-9B also has the highest reported means on COG and EVAL.
GPT-5.5 and Claude-Opus-5  retain the highest scores on SS and ROLE, respectively.
These results demonstrate their effectiveness in conversational simulation alongside competitive performance across the other capabilities.

\begin{figure*}[!ht]
\centering
% \vspace*{-2mm}
\includegraphics[width=\textwidth]{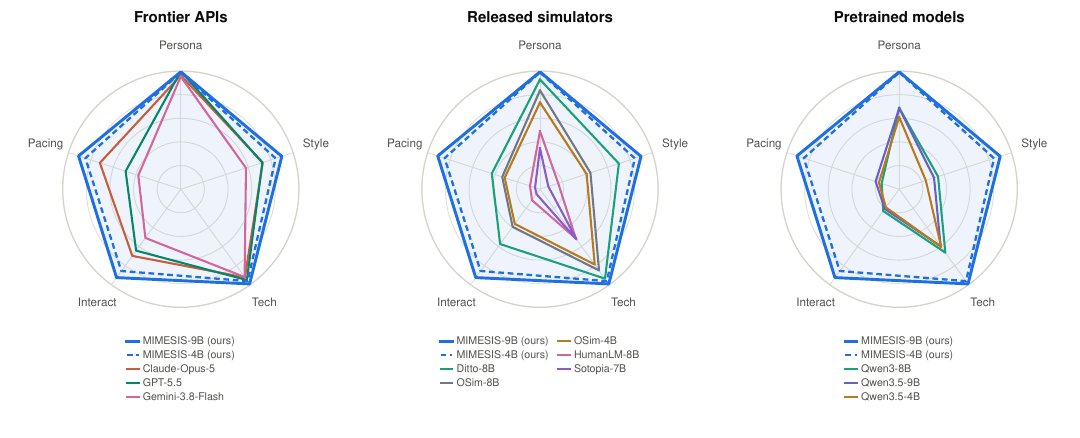}
% \vspace*{-2mm}
\caption{\textbf{Behavioral fidelity on RealUserSim PT3}~\citep{zhu2026realusersim}. Panels compare \method{} with frontier APIs, released simulators, and pretrained models. Each axis reports agreement
with human reference trajectories in persona, linguistic style,
technical competency, interaction and information flow, or pacing. Larger values indicate greater agreement with the human reference. %Exact scores are reported in Table~\ref{tab:realusersim-pt3}.%Table~\ref{tab:realusersim-pt3} gives the exact scores.
}
\label{fig:realusersim-radar}
\end{figure*}

\textbf{Trajectory-level behavioral fidelity.}
\label{sec:results-human}
On RealUserSim PT3, \method{}-9B achieves a Fidelity Index of 94.0, exceeding the strongest baseline, Claude-Opus-5  (80.6), by 13.4 points.
\method{}-4B scores 89.7, also outperforming both frontier models and released simulators (Figure~\ref{fig:realusersim-radar}).
Its gains over Claude-Opus-5  are particularly pronounced in interaction and information flow (91.3 versus 69.5) and pacing (91.2 versus 72.5).
By comparison, persona scores are close to ceiling for several models, including GPT-5.5 (99.3) and \method{}-9B (99.4). Because persona agreement is already near ceiling, the improvement primarily reflects how information is exchanged and paced over multiple turns rather than better recovery of static persona attributes.
%The strongest improvements therefore concern how simulated users exchange information and develop a conversation over time.

\section{Agent Post-Training with a Learned User Simulator}
\label{sec:agent}

We freeze the trained user simulator $u_{\hat\phi}$ and optimize the agent $\pi_\theta$ through repeated interaction with it.
At turn $t$, the agent generates a response $y_t$ from the public conversation history $H_t$, after which the simulator produces the next user turn.
The resulting conversations provide task rewards for multi-turn reinforcement learning.
We additionally use the simulator's thoughts and reactions to construct coaching notes for individual agent responses.
Coached On-Policy Self-Distillation (CSD) uses these notes to supervise the agent during training, while its rollout policy continues to condition only on the public history.

\subsection{Multi-Turn Reinforcement Learning}

For each training task, we sample a group of complete conversations using the rollout policy $\pi_{\theta_{\mathrm{old}}}$ and the frozen simulator.
We optimize the agent with Group Relative Policy Optimization (GRPO; \citealp{shao2024deepseekmath}), following its extension to multi-turn user interaction in UserRL~\citep{qian2025userrl}.
Task rewards determine group-normalized advantages for the sampled agent responses.
Let $\widehat A_{i,t}$ denote the advantage assigned to agent turn $t$ in conversation $i$.
For token $k$ of that response, the importance ratio is
\begin{equation}
    \rho_{i,t,k}(\theta)
    =
    \frac{
        \pi_\theta(y_{i,t,k}\mid H_{i,t},y_{i,t,<k})
    }{
        \pi_{\theta_{\mathrm{old}}}(y_{i,t,k}\mid H_{i,t},y_{i,t,<k})
    }.
    \label{eq:agent-ratio}
\end{equation}
The clipped policy loss is
\begin{equation}
    \mathcal L_{\mathrm{GRPO}}
    =
    -\mathbb E_{i,t,k}
    \left[
        \min\!\left(
            \rho_{i,t,k}\widehat A_{i,t},
            \operatorname{clip}(\rho_{i,t,k},1-\epsilon,1+\epsilon)
            \widehat A_{i,t}
        \right)
    \right],
    \label{eq:agent-grpo}
\end{equation}
where $\epsilon$ is the clipping parameter.
The expectation is over collected conversations and agent-generated tokens.
User utterances and tool outputs enter the conditioning history but are excluded from the policy loss.
The simulator remains fixed throughout agent optimization.

% \agenttodo{Specify the trajectory scorer, the allocation of reward to turns, and the loss normalization from the training implementation. If one trajectory advantage is broadcast to all agent tokens, set $\widehat A_{i,t}=\widehat A_i$ and give its normalization formula. Document any additional regularization.}

\subsection{Coached On-Policy Self-Distillation}
\label{sec:csd}

Scalar task rewards do not describe how a user interpreted an agent response or what the agent could have done differently.
The simulator provides additional feedback through its generated thoughts and reactions.
CSD uses this feedback to construct a teacher context for each sampled response.
The underlying approach follows feedback-conditioned self-distillation in SDPO~\citep{hubotter2026sdpo} and context-conditioned teaching in On-Policy Context Distillation~\citep{ye2026opcd}.
Here, the additional context is a coaching note derived from the simulated user's response to the agent.

\textbf{Constructing coaching notes.}
After an interaction, a coaching model $G$ receives the public history $H_t$, the sampled agent response $y_t$, and the simulator's associated thought $z_t$ and reaction $x_{t+1}$.
It produces a short note describing a possible improvement to that response:
\begin{equation}
    h_t = G(H_t,y_t,z_t,x_{t+1}).
    \label{eq:coach}
\end{equation}
The note is constructed retrospectively: it is unavailable to the agent when $y_t$ is generated.
It instead provides context for evaluating the response during the subsequent parameter update.
The simulator's thoughts and reactions are model-generated feedback, rather than observations of a human user's internal state.

\textbf{Learning from the coached context.}
We score the same sampled response under two contexts.
The student receives the public history $H_t$, while a teacher copy of the agent additionally receives $h_t$.
Let $\bar\theta$ denote a stop-gradient copy of the current policy parameters \(\theta\), used to evaluate the feedback-conditioned teacher distribution.
At each sampled token, we compute the difference between the teacher’s coached and the student’s uncoached log-probabilities:
\begin{align}
    \Delta_{t,k}
    &=
    \log \pi_{\bar\theta}(y_{t,k}\mid H_t,h_t,y_{t,<k})
    -
    \log \pi_\theta(y_{t,k}\mid H_t,y_{t,<k}),
    \label{eq:agent-coaching-gap}\\
    g_{t,k}
    &=
    \sigma\!\left(\beta\,\operatorname{sg}(\Delta_{t,k})\right),
    \label{eq:gate}
\end{align}
where $\sigma$ is the sigmoid function, $\beta>0$ controls the sensitivity of the weight to the log-probability difference, and $\operatorname{sg}$ denotes stop-gradient.
The sigmoid bounds the weights between zero and one; stop-gradient holds them fixed during the student update.
For a batch of coached responses, the auxiliary loss is
\begin{equation}
    \mathcal L_{\mathrm{CSD}}
    =
    -\mathbb E_{(i,t,k)\sim\mathcal C}
    \left[
        g_{i,t,k}\,
        \log\pi_\theta(y_{i,t,k}\mid H_{i,t},y_{i,t,<k})
    \right],
    \label{eq:csd}
\end{equation}
where $\mathcal C$ denotes the empirical distribution over agent-token positions in responses with coaching notes.
This objective gives greater weight to sampled tokens that the coached teacher assigns higher probability relative to the student.
Tokens with a negative log-probability difference receive weaker reinforcement; the CSD term does not directly penalize them.
Thus, the update is a weighted likelihood objective on sampled tokens, distinct from the next-token distribution matching used in SDPO and On-Policy Context Distillation~\citep{hubotter2026sdpo,ye2026opcd}.
It reuses the agent's sampled response without generating a replacement response.

We combine this auxiliary loss with the task-reward objective:
\begin{equation}
    \mathcal L_{\mathrm{agent}}
    =
    \mathcal L_{\mathrm{GRPO}}
    +\alpha\mathcal L_{\mathrm{CSD}},
    \label{eq:agent-combined-loss}
\end{equation}
where $\alpha\geq0$ controls the contribution of coaching.
GRPO supplies the reward-based policy update, while CSD supplies token weights derived from the coaching context.
Gradients from the auxiliary loss update only the student; they do not propagate through the teacher scores, coaching model, or user simulator.
At deployment, the agent generates responses from $\pi_\theta(\cdot\mid H_t)$ without access to coaching notes or private simulator outputs.

%\todo{Verify the retained sigmoid-weighted objective against the code, including the loss mask and averaging rule. Specify the teacher synchronization rule, the coaching model and prompt, the timing and contents of $z_t$ and $x_{t+1}$, and the values of $\alpha$ and $\beta$. Establish the contribution of thought and reaction through the pending ablations.}

\subsection{Agent evaluation}

We evaluate whether agents trained with \method{} generalize to user models not encountered during training.
Following the interactive task setting of UserRL~\citep{qian2025userrl}, we compare three training conditions: multi-turn GRPO with GPT-5.5 (UserRL), GRPO with \method{}-9B, and CSD with \method{}-9B.
Table~\ref{tab:transfer} reports agent performance across eight Gym environments, including three held out from training, under nine evaluation user models spanning frontier models and released simulators.
All nine evaluation users differ from both training simulators.
% Add \usepackage{caption} if \captionof is not already available.

\begin{figure*}[!ht]
    \centering
\begin{minipage}[!ht]{0.495\textwidth}
        \centering
        \small
        \setlength{\tabcolsep}{3.5pt}
        \resizebox{\linewidth}{!}{%
        \begin{tabular}{lrrr}
        \toprule
        \textbf{Evaluation user}
        & \shortstack{\textbf{GRPO} w. \\\textbf{GPT-5.5}}
        & \shortstack{\textbf{GRPO} w.\\\textbf{MIMESIS-9B}}
        & \shortstack{\textbf{CSD} w.\\\textbf{MIMESIS-9B}} \\
        \midrule
        GPT-5.6          & 26.85 & 31.10 & \textbf{32.39} \\
        Claude-Opus-5   & 25.75 & 27.21 & \textbf{29.78} \\
        Claude-Sonnet-5 & 26.53 & 29.73 & \textbf{30.97} \\
        Gemini-3.8-Flash & 26.33 & 30.88 & \textbf{31.09} \\
        Kimi-K3          & 24.51 & 28.21 & \textbf{29.92} \\
        Ditto-8B         & 27.59 & 29.76 & \textbf{32.08} \\
        Osim-8B          & 25.68 & 29.27 & \textbf{29.72} \\
        HumanLM-8B       & 26.96 & 30.67 & \textbf{32.53} \\
        Sotopia-7B       & 24.72 & 29.00 & \textbf{31.35} \\
        \midrule
        \textbf{Mean}    & 26.10 & 29.54 & \textbf{31.09} \\
        \bottomrule
        \end{tabular}%
        }
        \captionof{table}{%
        \textbf{Agent generalization across user models.}
         Mean task score across eight environments when evaluating agents trained with GPT-5.5 (UserRL), MIMESIS-9B (UserRL+), or MIMESIS-9B with CSD. None of the nine evaluation user models is used during agent training, the final row averages across evaluation users. %Higher is better.
        %Appendix~\ref{app:agent-evaluation} provides environment-level results and evaluation conventions.
        }
        \label{tab:transfer}
    \end{minipage}
    \hfill
    \begin{minipage}[!ht]{0.495\textwidth}
        \centering
        \vspace*{-3mm}
        \includegraphics[width=\linewidth]{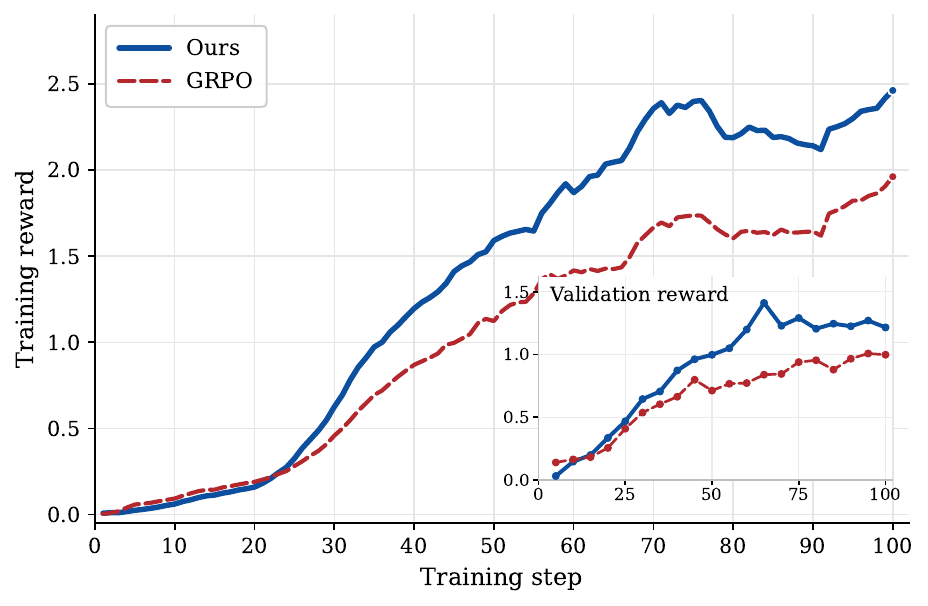}
        \captionof{figure}{
        {\textbf{Agent optimization with CSD.}
        Training and validation rewards over 100 optimization steps using the same MIMESIS-9B training simulator. CSD and GRPO therefore differ only in the coaching objective: CSD attains higher late-stage training and validation reward. % after the initial checkpoints.
        %The training panel shows raw and smoothed trajectories; validation  points correspond to checkpoint evaluations.
        }
        }
        \label{fig:sdpo_vs_grpo}
    \end{minipage}
\end{figure*}

With the GRPO objective fixed, replacing GPT-5.5 with \method{}-9B improves overall performance under every evaluation user, raising the mean score from 26.10 to 29.54.
Gains range from 1.46 points under Claude-Opus-5  to 4.55 points under Gemini-3.8-Flash.
Their consistency across frontier models and released simulators indicates that the benefits extend beyond interactions with the training simulator, supporting generalization to unfamiliar simulated users. Figure~\ref{fig:sdpo_vs_grpo} compares CSD with GRPO over 100 optimization steps using the same frozen \method{}-9B simulator. Following the earliest checkpoints, the CSD curve remains above GRPO on validation reward throughout the rest of the plotted interval. CSD further improves performance under all nine evaluation users, increasing the mean score from 29.54 to 31.09.
Because both conditions use MIMESIS-9B as the training simulator, this comparison isolates the effect of adding the CSD coaching objective while holding simulator choice fixed. 
%The training curves also separate, with CSD maintaining higher reward later in optimization.
%The agreement between the two panels indicates that the improvement is visible on the validation metric as well as the training metric.
%The figure does not display uncertainty across independent training runs, so the curves characterize these training trajectories rather than the variability of the method. Evaluation and run settings are specified in Appendix~\ref{app:agent-evaluation}.

\section{Conclusion}
\label{sec:conclusion}

In this paper, we study how user simulators can better serve as human proxies for agent training and evaluation. Our findings show that off-the-shelf LLMs can substantially distort interaction difficulty, motivating explicit modeling of user behavior beyond linguistic style. We introduced \method{}, a purpose-built user simulator that combines training on human conversations, explicit reasoning supervision, and reinforcement learning with 13 realistic behavioral patterns. Rather than relying on these behaviors to emerge from next-utterance imitation, MIMESIS is explicitly trained to express them at contextually appropriate points in grounded multi-turn interactions. Its strong performance across complementary simulation benchmarks supports the value of this training approach.  Crucially, these improvements translate into downstream learning utility. Across eight environments, agents trained with frozen \method{}-9B outperform those trained with GPT-5.5 under all nine unseen evaluation user models, demonstrating generalization beyond the training simulator. Our proposed CSD provides further gains by converting simulator-generated thoughts and subsequent user utterances into retrospective coaching about user needs and interaction feedback. This coaching supplies dense, token-level supervision during training, while the deployed agent relies only on public dialogue. Together, our results suggest that user simulators should be evaluated along two axes: fidelity to human interaction and transfer of policies trained against them. MIMESIS improves both simulator-level benchmarks and transfer across held-out simulator models.

\clearpage
\bibliographystyle{assets/plainnat}
\bibliography{paper}

% \clearpage
\beginappendix

% \clearpage
\section{Simulator Training Details}
\label{app:simulator-training}
\label{appendix:setup}
This section specifies the MIMESIS training pipeline used in Stage I of the main paper. We first describe user-side mid-training and ThoughtTrace supervision, then the construction of the realistic-behavior taxonomy, and finally the joint RL mixture and behavior-specific reward. Agent-training settings are reported separately in Appendix~\ref{app:agent-evaluation}.

\subsection{Data and supervised adaptation}
\label{app:sim-data}
We first adapt the Qwen3.5 backbones~\citep{qwen35blog} to the user role through mid-training on human--assistant conversations, predicting the next human utterance from the preceding dialogue.
The training mixture contains 21.2M examples from 62 corpora, from which our mid-training processes 9.51B tokens.
\clearpage

\begin{table}[t]
\centering
\small
\renewcommand{\arraystretch}{1.15}
\begin{tabularx}{\textwidth}{
    @{}
    >{\raggedright\arraybackslash}p{0.17\textwidth}
    >{\raggedright\arraybackslash}p{0.12\textwidth}
    >{\raggedright\arraybackslash}X
    @{}
}
\toprule
\textbf{Data source or environment}
& \textbf{Training stage}
& \textbf{Training signal} \\
\midrule
OdysSim \quad\quad\quad\quad\quad\quad\citep{zhou2026odyssim}
& Supervised
fine-tuning
& Prediction of the next human utterance conditioned on the preceding dialogue. \\

ThoughtTrace\quad\quad\quad\quad~\citep{jin2026thoughttrace}
& Supervised fine-tuning
& Supervision of a reasoning trace followed by the observed user utterance, using users' self-reported motivations and interpretations of assistant responses. \\

SOUL environments\quad ~\citep{zhou2026odyssim}
& Joint RL
& Task-specific rewards for simulator responses or multi-turn interactions. \\

ABCD scenarios\quad~\citep{chen2021abcd}
& Joint RL
& Rewards for behavior exhibition, temporal placement, naturalness, and consistency with scenario facts and policies. \\
\bottomrule
\end{tabularx}
\caption{\textbf{Data sources and training signals for MIMESIS.}
User-side mid-training and ThoughtTrace fine-tuning precede joint RL.
The SOUL environments and the ABCD task update the same simulator
parameters, with ABCD incorporating 13 realistic behavior categories
and a separate neutral condition.}
\label{tab:data}
\end{table}

\begin{wraptable}{r}{0.6\linewidth}
\centering
% \small
% \vspace*{-5mm}
% \setlength{\tabcolsep}{5pt}
\caption{\textbf{Supervised training configurations for MIMESIS-4B and MIMESIS-9B.}
User-side mid-training predicts human utterances with explicit thinking
disabled. Subsequent ThoughtTrace fine-tuning enables the native thinking
template and supervises a thought before the public response using users’ self-reported motivations and interpretations. %Both stages update all model parameters.
}
\label{tab:sft-hparams}
\resizebox{\linewidth}{!}{%
\begin{tabular}{llrr}
\toprule
& Hyperparameter & Mid-training & ThoughtTrace SFT \\
\midrule
% \multicolumn{4}{l}{\textit{Optimization}} \\
& Learning rate             & $3\times10^{-5}$ & $1\times10^{-7}$ \\
& Schedule                  & Cosine annealing & Cosine annealing \\
& Warm-up                   & 0 steps & 10 steps \\
& Optimizer                 & AdamW & AdamW \\
& $(\beta_1, \beta_2)$      & $(0.9,\ 0.999)$ & $(0.9,\ 0.999)$ \\
& Weight decay              & 0.01 & 0.0 \\
& Gradient clipping         & 1.0 & 1.0 \\
& Epochs                    & $\approx$0.97 (5{,}000 steps) & 3 \\
\midrule
% \multicolumn{4}{l}{\textit{Batching}} \\
& Sequences per step        & 4{,}096 & 32 \\
& Micro-batch per GPU       & 8 & 1 \\
& Gradient accumulation     & 16 & 4 \\
& Tokens per GPU (cap)      & 49{,}152 & -- \\
\midrule
% \multicolumn{4}{l}{\textit{Sequence handling}} \\
& Max sequence length       & 24{,}576 & 32{,}768 \\
& Max prompt / response     & 16{,}384 / 8{,}192 & -- \\
\midrule
% \multicolumn{4}{l}{\textit{Reasoning}} \\
& Thinking during training  & disabled & enabled \\
& Chat template             & ChatML & Qwen3.5 thinking \\
\midrule
% \multicolumn{4}{l}{\textit{Systems}} \\
& Trainer                   & verl SPMD & TRL \textsc{SFTTrainer} \\
& Sharding                  & FSDP & DeepSpeed ZeRO-3 \\
& Precision                 & bfloat16 & bfloat16 \\
& Gradient checkpointing    & enabled & enabled \\
& Hardware                  & $4\times8$ A100 & $1\times8$ A100 \\
& Wall-clock                & $\approx$4.8 days & 45 min \\
\midrule
% \multicolumn{4}{l}{\textit{Data}} \\
% & Source corpora            & 62 (165 shards) & 1 (ThoughtTrace) \\
& Training examples         & 21.2M & 2{,}155 \\
& Tokens seen               & 9.51B & 15.8M \\
\bottomrule
\end{tabular}}
\vspace*{-10mm}
\end{wraptable}
We then fine-tune on 2,155 ThoughtTrace conversations with turn-level human annotations~\citep{jin2026thoughttrace}.
Users' self-reported reasons for their messages and interpretations of preceding assistant responses supervise an explicit reasoning trace before the observed user utterance.
During interaction, the simulator generates both a thought and a public utterance, with only the utterance shown to the agent.
Table~\ref{tab:data} summarizes the data sources and supervision while Table~\ref{tab:sft-hparams} reports the training configurations.
% \clearpage
\subsection{Realistic user behaviors}
\label{app:behavior-taxonomy}
We derive a taxonomy of 13 realistic user behaviors from ThoughtTrace~\citep{jin2026thoughttrace}, with
 Table~\ref{tab:taxonomy} providing the definition for each category.
These categories describe interaction patterns without presuming that users deliberately impede task completion.
For example, an additional constraint may reflect revised requirements, while a previously unstated preference may become apparent only in a later conversational turn. We construct the taxonomy in two stages: behavior annotation
followed by label consolidation. Both stages use GPT-5.6~\citep{openai2026gpt56}
through its API with default decoding parameters.

\begin{table}[!htbp]
\centering
\small
\resizebox{\linewidth}{!}{%
{\begin{tabularx}{\textwidth}{@{}p{0.02\textwidth}p{0.25\textwidth}X@{}}
\toprule
ID & Behavior & Operational description \\
\midrule
B1 & Hidden evaluation criteria & Judges an answer against a preference or standard that was not stated in the request. \\
B2 & Incremental goalpost shifting & Introduces or tightens requirements after the assistant has already made progress. \\
B3 & Clarification noncooperation & Declines to answer, or only partly answers, a clarification request. \\
B4 & Infeasibility persistence & Continues to request an option after its unavailability or incompatibility with policy has been explained. \\
B5 & Underspecified-then-blames & Omits relevant information and subsequently faults the assistant for not inferring it. \\
B6 & Abrupt intent switching & Redirects the conversation to another objective before resolving the preceding request. \\
B7 & Artificial constraint stacking & Adds restrictions that remove options on which the assistant's proposed solution relies. \\
B8 & Fabricated or false premise & Presents an unsupported or contradictory fact, or a claim about prior dialogue, as established context. \\
B9 & Impatient low-signal pressure & Expresses urgency or dissatisfaction without supplying a sufficiently specific correction. \\
B10 & Unresolved thread abandonment & Disengages from a conversational thread before resolving its stated objective. \\
B11 & Self-contradiction & States requirements or facts that conflict with an earlier statement without acknowledging the change. \\
B12 & Request overloading & Combines numerous deliverables, constraints, or checks within a single request. \\
B13 & Belittling or entitled pressure & Uses disparaging or entitled language to pressure the assistant. \\

N & Neutral & Follows the scenario without an imposed behavioral pattern. \\
\bottomrule
\end{tabularx}}}
% Table 6 | tab:taxonomy | section/appendix/A_simulator_training.tex
\caption{\textbf{Taxonomy of 13 realistic user behaviors.}
Categories are derived from recurring patterns in ThoughtTrace and defined
operationally for the ABCD behavior task. A separate neutral condition
specifies interaction without an imposed behavioral pattern. }
\label{tab:taxonomy}
\end{table}
\FloatBarrier

\paragraph{Stage 1: Behavior annotation.}
We serialize each conversation as a turn-indexed transcript containing both
user and assistant utterances. ThoughtTrace reasons and reactions, together
with their labels, are placed beneath the messages to which they refer.
Assistant turns are retained because several behaviors are defined relative
to prior interaction, such as failing to answer a clarification, contradicting
an earlier statement, or introducing requirements after the assistant has
made progress. For each conversation, the annotator identifies non-cooperative user behavior
using a seed codebook of 16 categories with one-sentence definitions
(Table~\ref{tab:seed-codebook}). When no category fits, it may assign an open
code, \textit{other}, with a free-text label. Each annotation records the
category, a short label, relevant turn indices, supporting user and thought
quotations, a severity level (\textit{low}, \textit{medium}, or \textit{high}),
and a brief rationale; at least one quotation must be present. Thought-only
annotations are allowed when a behavior is not expressed directly in the
public utterance, such as an unstated evaluation criterion. The prompt also
provides guidance for interpreting ThoughtTrace reason and reaction labels and
instructs the annotator to return an empty list when no target behavior is
present. We provide one worked example of the expected JSON format and issue
one API request per conversation. All 2,155 conversations produced valid
annotations.

\begin{table}[!ht]
\centering
\small
\caption{\textbf{Seed codebook used for Stage~1 behavior annotation.}
The GPT-5.6 annotator assigns each detected behavior to one of the predefined categories or, when none is appropriate, to an open \textit{other} category with a free-text label.}
\label{tab:seed-codebook}
\resizebox{\linewidth}{!}{\begin{tabular}{@{}p{0.18\linewidth}p{0.82\linewidth}@{}}
\toprule
Category & Definition \\
\midrule
Intent switching & Abruptly changes the goal or task mid-conversation, abandoning the prior one without resolution. \\
Non-collaboration & Withholds information the assistant needs, or refuses to answer the assistant's clarifying questions. \\
Underspecification, then blame & Gives a vague or incomplete request, then faults the assistant for not reading their mind. \\
Contradiction & States something that contradicts what they said earlier in the conversation. \\
Moving goalposts & Adds new requirements or expands scope after the assistant met the original request. \\
False premise & Asserts an incorrect fact or assumption as if it were true and builds on it. \\
Out-of-scope request & Asks for a capability the assistant cannot provide, and disbelieves or insists after being told it is not possible. \\
Topic drift & Wanders off-topic or digresses away from the stated task. \\
Overloading & Packs many distinct requests into one turn, or adds requests faster than they can be answered. \\
Impatience & Expresses curt dissatisfaction with little signal (e.g., ``no'', ``wrong'', ``still bad'') or demands faster results. \\
Hostility & Uses a rude, aggressive, insulting, or openly antagonistic tone. \\
Passive aggression & Uses sarcasm, backhanded remarks, or veiled criticism instead of direct feedback. \\
Testing the assistant & Deliberately tries to trip up, trap, or catch out the assistant. \\
Entitled demands & Uses demanding, entitled framing (e.g., ``just do it'', ``you must'') with no give-and-take. \\
Abrupt abandonment & Drops the thread or disengages without closure once partly satisfied. \\
Unvoiced criterion & Keeps a decisive expectation, constraint, or verdict entirely in private thought; the spoken turn does not convey it. \\
\bottomrule
\end{tabular}}
\end{table}

\paragraph{Stage 2: Taxonomy consolidation.}
The 862 Stage~1 annotations contain 853 distinct category--label pairs, making
direct aggregation by label impractical. We therefore group annotations by
category--label pair and attach their supporting quotations, using the public
utterance when available and otherwise the self-reported thought. The frequency-sorted groups are passed to the same model in a single
consolidation request. The model is instructed to merge labels describing the
same behavior, preserve meaningful distinctions, and treat open codes as
candidate new categories. For each consolidated category, it returns a name,
an operational definition, the source codes it subsumes, representative
quotations, and an instruction for how a simulated user should enact the
behavior. This procedure yields the 13 categories in
Table~\ref{tab:taxonomy}. Every annotated seed category and open-coded label is
assigned to exactly one consolidated category. Frequencies in
Figure~\ref{fig:behaviors} are then computed directly from the Stage~1
annotations under this mapping, rather than estimated by the model. The final
taxonomy largely preserves the seed codebook: hostility, passive aggression,
and entitled demands are merged into \textit{belittling or entitled pressure};
the four open-coded annotations (0.5\%) are absorbed into existing categories;
\textit{topic drift} receives no annotations; and no seed category is split.

\begin{figure}[!ht]
    \centering
\includegraphics[width=\columnwidth]{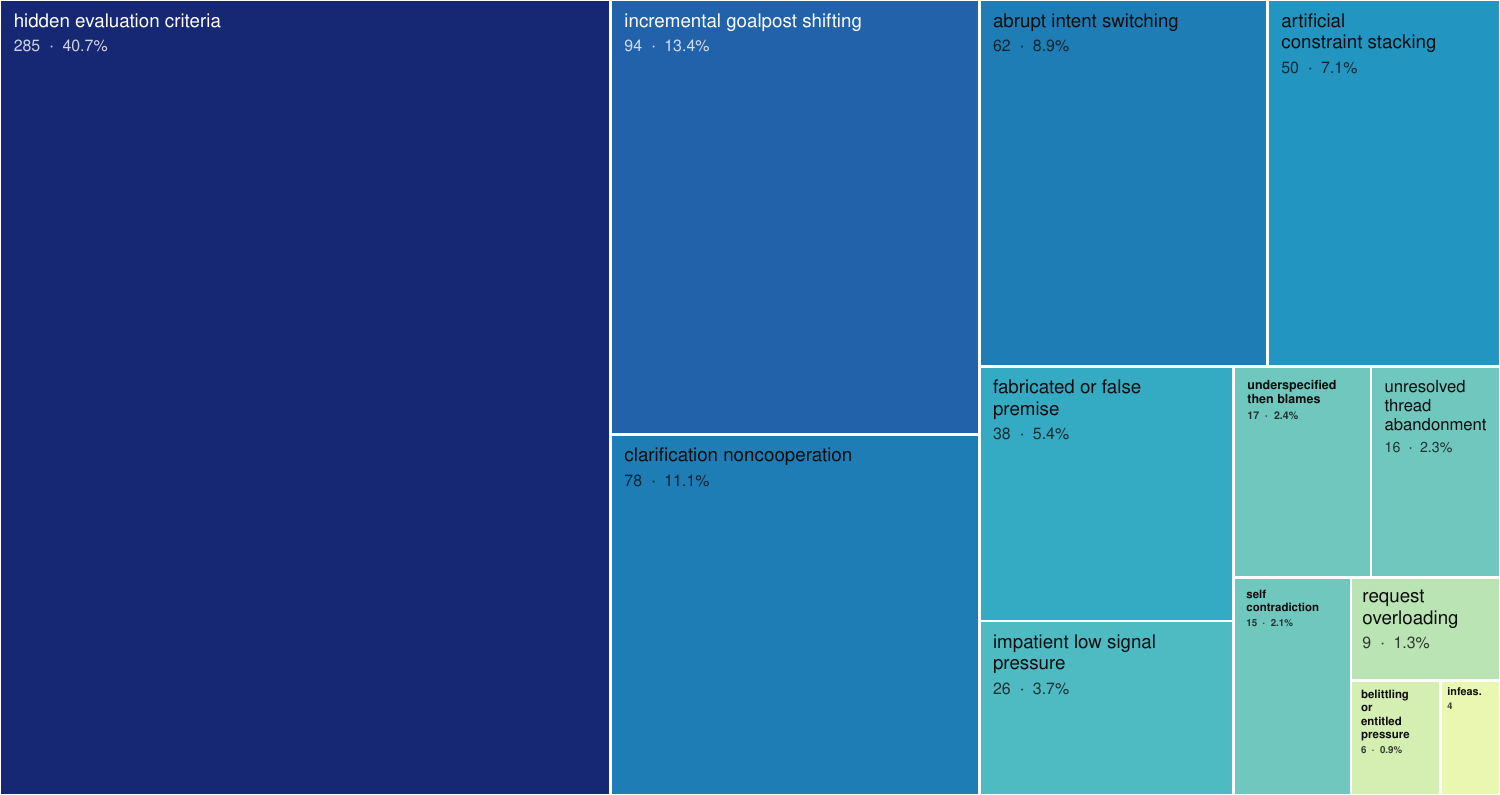}
\caption{\textbf{Distribution of realistic behaviors in the annotated ThoughtTrace sample.}
Rectangle areas represent the observed counts of the 13 categories in
Table~\ref{tab:taxonomy}, labels report counts and percentages. Counts are computed from the Stage-1 annotations of the 2,155 ThoughtTrace conversations and describe the annotated sample rather than population prevalence.}

    \label{fig:behaviors}
\end{figure}
\clearpage
\subsection{Joint multi-domain reinforcement learning}
\label{app:joint-rl-details}

\begin{wraptable}{r}{0.55\linewidth}
\centering
\small
\vspace*{-5mm}
\caption{\textbf{Joint reinforcement-learning configuration for MIMESIS-9B.}
Training starts from the ThoughtTrace SFT checkpoint and updates all
parameters on a mixture of 23 SOUL tasks and the ABCD realistic-behavior
task. All 24 environments update a shared simulator rather than separate task-specific experts.%The FoldGRPO implementation computes group reward statistics over unique rollout identifiers, preventing repeated segments of one rollout from being counted multiple times in advantage normalization.
}
\label{tab:rl-hparams}
\resizebox{\linewidth}{!}{%
\begin{tabular}{llr}
\toprule
& Hyperparameter & Value \\
\midrule
\multicolumn{3}{l}{\textit{Objective}} \\
& Advantage estimator            & FoldGRPO \\
& Policy loss                    & PPO-clip \\
& Clip range (low / high)        & 0.2 / 0.28 \\
& Dual-clip constant $c$         & 10.0 \\
& KL penalty (loss / reward)     & none / none \\
& Entropy coefficient            & 0 \\
& Advantage normalization        & group std \\
& Loss aggregation               & token-mean \\
\midrule
\multicolumn{3}{l}{\textit{Optimization}} \\
& Learning rate                  & $8\times10^{-6}$ \\
& Schedule                       & constant, no warm-up \\
& Optimizer                      & AdamW \\
& $(\beta_1, \beta_2)$           & $(0.9,\ 0.95)$ \\
& Weight decay                   & 0.01 \\
& Gradient clipping              & 1.0 \\
& Training steps                 & 500 \\
\midrule
\multicolumn{3}{l}{\textit{Batching}} \\
& Prompts per step               & 64 \\
& Rollouts per prompt            & 8 \\
& Rollouts per step              & 512 \\
& PPO mini-batch                 & 16 \\
& Tokens per GPU (dynamic)       & 49{,}152 \\
\midrule
\multicolumn{3}{l}{\textit{Sequence and sampling}} \\
& Max prompt length              & 8{,}192 \\
& Max response length            & 16{,}384 \\
& Temperature / top-$p$ / top-$k$ & 1.0 / 1.0 / $-1$ \\
& Validation rollouts per prompt & 1 \\
\midrule
\multicolumn{3}{l}{\textit{Reasoning}} \\
& Thinking during rollout        & enabled \\
& Chat template                  & native (Qwen3.5 thinking) \\
& Min.\ reasoning tokens         & 4 \\
\midrule
\multicolumn{3}{l}{\textit{Systems}} \\
& Trainer / rollout engine       & FSDP / vLLM (colocated) \\
& Precision                      & bfloat16 \\
& Parameter, optimizer offload   & CPU \\
& Tensor parallel, Ulysses SP    & 1, 1 \\
& vLLM memory fraction           & 0.6 \\
& Hardware                       & $2\times8$ A100 \\
\midrule
\multicolumn{3}{l}{\textit{Data and reward}} \\
& Training tasks                 & 24 (joint mixture) \\
& Reward / judge model           & GPT-5.5 \\
\bottomrule
\end{tabular}}
\vspace*{-13mm}
\end{wraptable}

We initialize MIMESIS from the ThoughtTrace SFT checkpoint and jointly optimize it on the SOUL environments and the ABCD realistic-behavior task.
Each environment retains its interaction protocol and reward function, while rollouts from all domains update the same simulator parameters under Equation~\ref{eq:sim-objective}.
OdysSim instead trains task-specific experts and subsequently distills their trajectories into a unified model~\citep{zhou2026odyssim}. We use the FoldGRPO estimator, which  computes group reward statistics over unique rollout identifiers, preventing repeated segments of one rollout from being counted multiple times in advantage normalization.
Table~\ref{tab:rl-hparams} reports the MIMESIS-9B sampling settings, optimization hyperparameters, and training budget.

Figure~\ref{fig:behaviors} reports category frequencies in the annotated ThoughtTrace sample.
Hidden evaluation criteria account for the largest share, with 285 instances (40.7\%), followed by incremental goalpost shifting (94; 13.4\%) and clarification noncooperation (78; 11.1\%).
These frequencies describe the annotated sample, they do not estimate population prevalence or define the behavior-sampling distribution used for RL.
Figure~\ref{fig:behaviors_exp} illustrates the categories with user utterances and self-reported thoughts.
Interpreting these examples requires the dialogue history and scenario facts, particularly when identifying contradictions, infeasible requests, or false premises.

\begin{figure}[!ht]
    \centering
\includegraphics[width=1.\columnwidth]{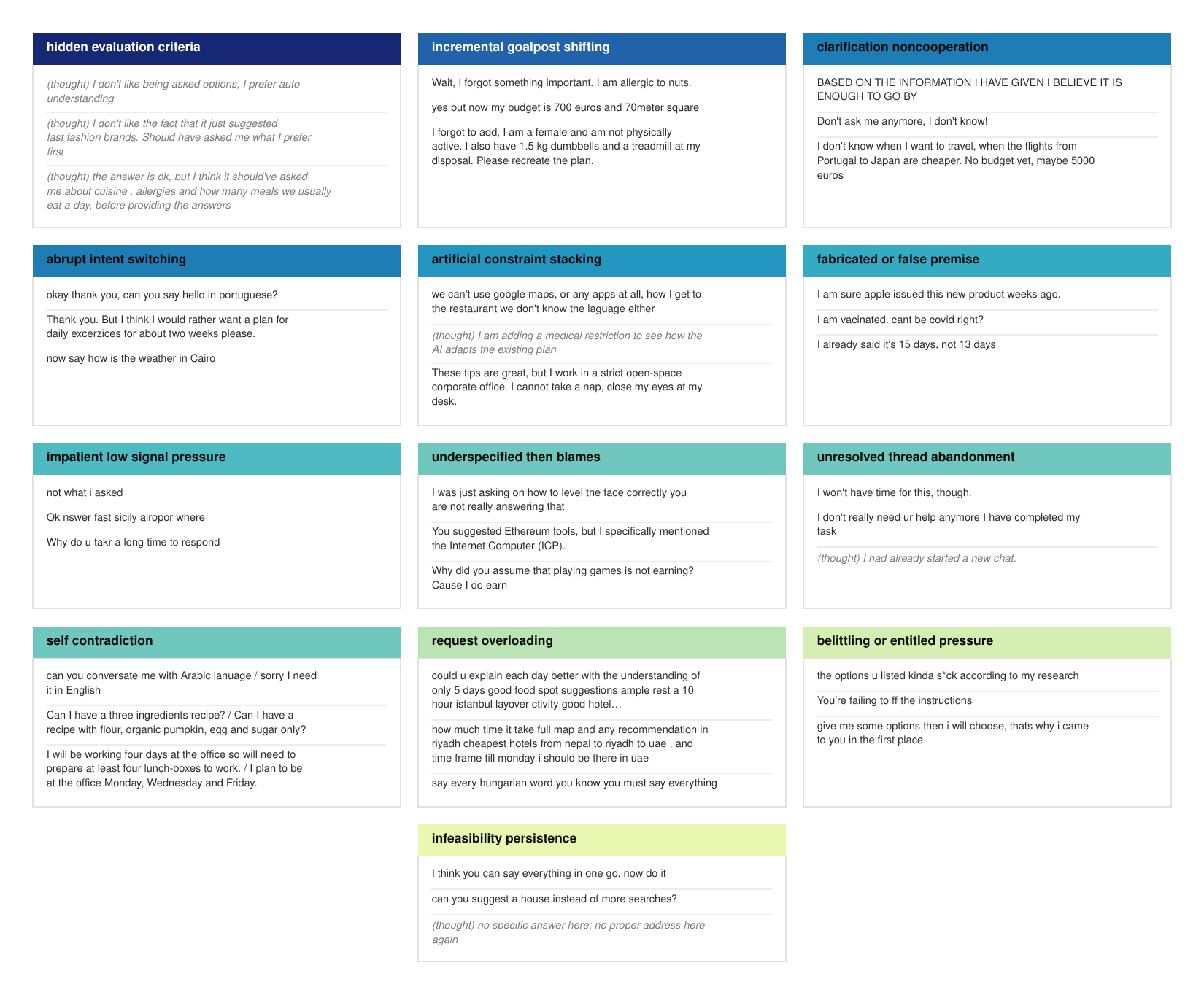}
\caption{\textbf{Examples motivating the 13 realistic behavior categories.}
Excerpts from ThoughtTrace illustrate recurring patterns in user utterances
and self-reported thoughts. Italicized entries marked ``(thought)'' are human
self-reports; the remaining entries are public user utterances. Original
wording is retained, and each example should be interpreted within its
conversation context.}
    \label{fig:behaviors_exp}
\end{figure}
% \clearpage

\subsection{Behavior rollouts and reward}
\label{app:behavior-rollouts}

We construct rollouts from ABCD customer-support scenarios~\citep{chen2021abcd}, which specify structured service procedures.
Each example is assigned one of the 13 behavior categories or a neutral condition.
Neutral examples constitute one third of the behavior-task data and impose no specific behavioral pattern.
Rollouts begin after a replayed identity check and continue for at most eight simulated user turns with a frozen support agent.

We evaluate each rollout against the scenario facts and policies.
A false premise must contradict a recorded fact, withheld information must be available to the simulated user, and an infeasible request must conflict with a stated policy.
The evaluator scores behavior exhibition $e$, temporal placement $p$, naturalness $n$, and task consistency $c$.
We combine these scores as
\begin{equation}
 r_{\mathrm{beh}}
 = \frac{2.0e+1.5p+1.0n+0.5c}{5.0}.
 \label{eq:behavior-reward}
\end{equation}
When the simulator describes the assigned behavior instead of expressing it through the interaction, we halve the reward:
\begin{equation}
 r_{\mathrm{beh}} \leftarrow \tfrac{1}{2}r_{\mathrm{beh}}.
\end{equation}
The weighting emphasizes whether the assigned behavior occurs and whether its timing is appropriate, while naturalness and task consistency reward plausible interaction within the scenario.
Task consistency remains a graded reward component, so a high total reward does not guarantee that every scenario constraint is satisfied.
% AUTHOR CHECK: See AUTHOR_CHECKS.md for missing manifests, prompts, and reward-scale details.

\section{Simulator Evaluation}
\label{app:evaluation}

\paragraph{Evaluation protocol and model configuration.}

We evaluate four properties that need not move together: broad simulation capability, trajectory-level behavioral fidelity, calibration to human task outcomes, and next-turn realism. We report them separately because a simulator can match surface style without matching how information unfolds across turns or how difficult the resulting task is for an agent. The corresponding tables specify the model versions used in each evaluation. For proprietary models, we use fixed reasoning configurations without additional model-specific tuning.
GPT-5.5 and GPT-5.6 use \texttt{medium} reasoning effort~\citep{openai2026reasoning}, Gemini-3.8-Flash uses \texttt{medium} dynamic thinking \citep{google2026thinking}, and Claude-Opus-5 uses adaptive thinking \citep{anthropic2026thinking}. These reasoning modes do not imply equal inference budgets across providers.

\subsection{SOUL-Index}
\label{app:soul}

SOUL evaluates conversational interaction (CONV), social simulation (SS), cognition (COG), role play (ROLE), and evaluation or judgment (EVAL)~\citep{zhou2026odyssim}.
Table ~\ref{tab:soul-per-dataset} reports all 23 datasets, separating model groups for readability.
MIMESIS-9B achieves an overall score of 65.7, compared with 64.9 for Claude-Opus-5 and 58.8 for Osim-8B; MIMESIS-4B scores 63.7.
The strongest advantage is in conversational simulation.
Both MIMESIS models exceed the frontier baselines on all four CONV datasets.
For example, MIMESIS-9B scores 74.4 on MirrorBench, compared with 56.8 for the strongest frontier baseline, and MIMESIS-4B scores 53.2 on Humanual-Chat, compared with 41.5.

\begin{table*}[!ht]
  \centering
% SOUL per-dataset. Mean $\pm$ SEM over 3 seeds.
% bold = best open-source; underline = best frontier API.
\resizebox{\linewidth}{!}{% SOUL per-dataset. Mean $\pm$ SEM over 3 seeds.
% bold = best open-source; underline = best frontier API.
\begin{tabular}{llccccccccccccc}
\toprule
 & Dataset & Claude-Opus-5 & GPT-5.5 & Gemini-3.8-Flash & Osim-8B & Osim-4B & Ditto-8B & HumanLM-8B & Sotopia-7B & Qwen3.5-9B & Qwen3-8B & Qwen3.5-4B & MIMESIS-9B & MIMESIS-4B \\
\midrule
\multirow{4}{*}{\rotatebox{90}{CONV}}
 & UserLLM & 64.9\tiny{$\pm$0.8} & \underline{67.0\tiny{$\pm$1.6}} & 62.7\tiny{$\pm$1.4} & 93.3\tiny{$\pm$0.2} & 92.9\tiny{$\pm$0.0} & \textbf{94.1\tiny{$\pm$0.4}} & 43.5\tiny{$\pm$0.1} & 46.8\tiny{$\pm$3.7} & 60.3\tiny{$\pm$0.9} & 43.6\tiny{$\pm$1.2} & 56.1\tiny{$\pm$2.3} & 93.3\tiny{$\pm$0.1} & 93.0\tiny{$\pm$0.4} \\
 & MirrorBench & \underline{56.8\tiny{$\pm$0.1}} & 51.7\tiny{$\pm$0.4} & 47.2\tiny{$\pm$0.1} & 66.0\tiny{$\pm$0.6} & 64.3\tiny{$\pm$0.5} & 64.8\tiny{$\pm$0.4} & 40.8\tiny{$\pm$0.6} & 34.2\tiny{$\pm$0.9} & 45.6\tiny{$\pm$0.8} & 43.6\tiny{$\pm$0.5} & 40.5\tiny{$\pm$0.8} & \textbf{74.4\tiny{$\pm$0.4}} & 73.2\tiny{$\pm$0.2} \\
 & Humanual-Chat & 29.6\tiny{$\pm$0.2} & 33.1\tiny{$\pm$0.3} & \underline{41.5\tiny{$\pm$0.5}} & 23.2\tiny{$\pm$0.4} & 21.5\tiny{$\pm$0.7} & 22.7\tiny{$\pm$0.1} & 29.8\tiny{$\pm$0.2} & 25.1\tiny{$\pm$1.1} & 17.9\tiny{$\pm$0.5} & 26.9\tiny{$\pm$0.6} & 18.2\tiny{$\pm$0.3} & 53.0\tiny{$\pm$0.2} & \textbf{53.2\tiny{$\pm$0.4}} \\
 & SimArena-Doc & \underline{78.1\tiny{$\pm$0.3}} & 75.0\tiny{$\pm$0.4} & 72.4\tiny{$\pm$0.2} & 57.5\tiny{$\pm$0.3} & 55.1\tiny{$\pm$0.1} & 65.3\tiny{$\pm$0.1} & 61.0\tiny{$\pm$0.4} & 58.3\tiny{$\pm$0.4} & 71.7\tiny{$\pm$0.6} & 65.1\tiny{$\pm$0.2} & 69.0\tiny{$\pm$1.9} & \textbf{79.7\tiny{$\pm$0.3}} & 78.4\tiny{$\pm$0.1} \\
\midrule
\multirow{1}{*}{\rotatebox{90}{SS}}
 & Sotopia-Hard & 70.4\tiny{$\pm$0.3} & \underline{70.9\tiny{$\pm$0.4}} & 66.9\tiny{$\pm$0.4} & 39.1\tiny{$\pm$0.1} & 35.0\tiny{$\pm$0.4} & 46.4\tiny{$\pm$0.3} & 59.8\tiny{$\pm$0.4} & 59.7\tiny{$\pm$0.4} & 54.2\tiny{$\pm$0.6} & 58.7\tiny{$\pm$0.5} & 49.7\tiny{$\pm$2.3} & \textbf{66.4\tiny{$\pm$0.2}} & 65.2\tiny{$\pm$0.2} \\
\midrule
\multirow{4}{*}{\rotatebox{90}{COG}}
 & Fantom & 84.7\tiny{$\pm$0.7} & \underline{93.0\tiny{$\pm$0.6}} & \underline{93.0\tiny{$\pm$0.0}} & 82.0\tiny{$\pm$1.7} & 77.0\tiny{$\pm$1.5} & \textbf{94.7\tiny{$\pm$0.3}} & 84.7\tiny{$\pm$1.2} & 0.0\tiny{$\pm$0.0}\textsuperscript{\dag} & 77.7\tiny{$\pm$3.9} & 81.3\tiny{$\pm$0.9} & 64.0\tiny{$\pm$4.4} & 92.0\tiny{$\pm$0.0} & 89.0\tiny{$\pm$1.5} \\
 & Hitom & 68.7\tiny{$\pm$0.3} & \underline{84.7\tiny{$\pm$1.2}} & 71.3\tiny{$\pm$2.0} & 74.7\tiny{$\pm$1.2} & 71.3\tiny{$\pm$1.8} & 77.3\tiny{$\pm$1.2} & 65.0\tiny{$\pm$0.6} & 32.7\tiny{$\pm$2.8} & 56.7\tiny{$\pm$17.9} & 65.7\tiny{$\pm$2.2} & 47.3\tiny{$\pm$15.8} & \textbf{87.0\tiny{$\pm$2.0}} & 77.3\tiny{$\pm$0.7} \\
 & Paratomi & 94.7\tiny{$\pm$0.3} & \underline{98.3\tiny{$\pm$0.3}} & 97.0\tiny{$\pm$0.0} & 83.7\tiny{$\pm$3.0} & 82.3\tiny{$\pm$1.8} & 89.3\tiny{$\pm$0.7} & 75.3\tiny{$\pm$3.2} & 41.3\tiny{$\pm$2.2} & 70.3\tiny{$\pm$2.4} & 78.3\tiny{$\pm$0.3} & 56.0\tiny{$\pm$6.9} & \textbf{98.3\tiny{$\pm$0.3}} & 89.3\tiny{$\pm$0.9} \\
 & Social-R1 & 80.0\tiny{$\pm$0.6} & 68.3\tiny{$\pm$0.7} & \underline{81.0\tiny{$\pm$0.6}} & 57.7\tiny{$\pm$1.8} & 61.0\tiny{$\pm$2.0} & 51.7\tiny{$\pm$0.7} & 50.3\tiny{$\pm$1.9} & 45.0\tiny{$\pm$0.0} & 52.7\tiny{$\pm$1.2} & 50.3\tiny{$\pm$0.7} & 53.3\tiny{$\pm$1.7} & \textbf{67.3\tiny{$\pm$1.7}} & 63.7\tiny{$\pm$1.3} \\
\midrule
\multirow{11}{*}{\rotatebox{90}{ROLE}}
 & Coser & \underline{72.6\tiny{$\pm$0.3}} & 67.3\tiny{$\pm$0.1} & 63.5\tiny{$\pm$0.2} & 55.1\tiny{$\pm$0.1} & 45.8\tiny{$\pm$0.1} & 56.9\tiny{$\pm$0.3} & 35.0\tiny{$\pm$0.8} & 25.5\tiny{$\pm$0.5} & 34.4\tiny{$\pm$1.6} & 36.9\tiny{$\pm$0.2} & 23.4\tiny{$\pm$1.0} & \textbf{57.0\tiny{$\pm$0.3}} & 53.8\tiny{$\pm$0.6} \\
 & Lifechoices & 94.7\tiny{$\pm$0.7} & 90.0\tiny{$\pm$1.0} & \underline{95.7\tiny{$\pm$0.3}} & 79.3\tiny{$\pm$0.3} & 84.0\tiny{$\pm$2.5} & 70.3\tiny{$\pm$1.8} & 65.7\tiny{$\pm$1.3} & 65.3\tiny{$\pm$0.3} & 69.7\tiny{$\pm$1.8} & 68.3\tiny{$\pm$1.8} & 71.3\tiny{$\pm$1.2} & \textbf{89.3\tiny{$\pm$0.3}} & 87.0\tiny{$\pm$1.0} \\
 & Twinvoice & \underline{83.0\tiny{$\pm$0.6}} & 75.3\tiny{$\pm$0.3} & 23.3\tiny{$\pm$0.9}\textsuperscript{\ddag} & 69.0\tiny{$\pm$0.6} & 64.3\tiny{$\pm$0.9} & 67.3\tiny{$\pm$1.3} & 40.3\tiny{$\pm$2.4} & 31.3\tiny{$\pm$3.8} & 55.0\tiny{$\pm$1.5} & 42.0\tiny{$\pm$1.2} & 40.3\tiny{$\pm$0.9} & \textbf{76.3\tiny{$\pm$1.5}} & 72.0\tiny{$\pm$1.2} \\
 & BehaviorChain & \underline{96.0\tiny{$\pm$0.0}} & 94.3\tiny{$\pm$0.7} & 94.0\tiny{$\pm$0.6} & 94.7\tiny{$\pm$1.2} & 86.0\tiny{$\pm$1.7} & 44.7\tiny{$\pm$0.9} & 36.3\tiny{$\pm$1.2} & 30.3\tiny{$\pm$1.8} & 52.3\tiny{$\pm$1.9} & 36.3\tiny{$\pm$0.9} & 45.3\tiny{$\pm$3.3} & 95.7\tiny{$\pm$0.7} & \textbf{98.3\tiny{$\pm$0.3}} \\
 & SimArena-Math & 62.4\tiny{$\pm$0.3} & 61.8\tiny{$\pm$0.2} & \underline{63.4\tiny{$\pm$0.1}} & 51.8\tiny{$\pm$0.5} & 53.4\tiny{$\pm$0.2} & 56.0\tiny{$\pm$0.3} & 44.3\tiny{$\pm$0.4} & 45.9\tiny{$\pm$0.9} & 59.1\tiny{$\pm$0.7} & 51.3\tiny{$\pm$0.1} & 53.6\tiny{$\pm$0.4} & \textbf{64.0\tiny{$\pm$0.2}} & 57.6\tiny{$\pm$0.2} \\
 & Mistakes & 63.3\tiny{$\pm$1.2} & 70.3\tiny{$\pm$0.9} & \underline{76.0\tiny{$\pm$0.6}} & 60.7\tiny{$\pm$0.9} & 54.0\tiny{$\pm$1.2} & 39.0\tiny{$\pm$2.6} & 56.7\tiny{$\pm$1.7} & 18.0\tiny{$\pm$0.6} & 59.7\tiny{$\pm$1.8} & 57.7\tiny{$\pm$0.9} & 60.0\tiny{$\pm$0.6} & 63.0\tiny{$\pm$0.6} & \textbf{65.3\tiny{$\pm$0.9}} \\
 & Humanual-Email & \underline{41.2\tiny{$\pm$0.7}} & 39.5\tiny{$\pm$0.3} & 40.1\tiny{$\pm$0.3} & 35.5\tiny{$\pm$0.4} & 33.5\tiny{$\pm$0.4} & 27.3\tiny{$\pm$0.1} & 31.8\tiny{$\pm$0.3} & 30.0\tiny{$\pm$0.4} & 31.0\tiny{$\pm$0.5} & 28.3\tiny{$\pm$0.1} & 29.2\tiny{$\pm$1.2} & \textbf{40.7\tiny{$\pm$0.1}} & 36.1\tiny{$\pm$0.6} \\
 & Humanual-News & \underline{53.7\tiny{$\pm$1.1}} & 51.2\tiny{$\pm$0.3} & 52.8\tiny{$\pm$0.3} & 39.1\tiny{$\pm$0.7} & 35.1\tiny{$\pm$0.3} & 31.1\tiny{$\pm$1.1} & 40.9\tiny{$\pm$0.9} & 30.9\tiny{$\pm$1.6} & 38.6\tiny{$\pm$2.4} & 30.6\tiny{$\pm$1.0} & 29.1\tiny{$\pm$2.5} & \textbf{46.5\tiny{$\pm$0.5}} & 41.5\tiny{$\pm$0.6} \\
 & Humanual-Politics & 42.4\tiny{$\pm$0.5} & \underline{43.5\tiny{$\pm$0.5}} & 36.6\tiny{$\pm$0.3} & 34.2\tiny{$\pm$0.4} & 31.1\tiny{$\pm$0.6} & 27.6\tiny{$\pm$0.3} & 32.3\tiny{$\pm$0.2} & 28.8\tiny{$\pm$0.4} & 30.4\tiny{$\pm$0.1} & 29.8\tiny{$\pm$0.3} & 24.8\tiny{$\pm$0.9} & \textbf{42.6\tiny{$\pm$0.2}} & 39.8\tiny{$\pm$0.4} \\
 & Humanual-Book & \underline{56.7\tiny{$\pm$0.4}} & 52.8\tiny{$\pm$0.7} & 49.9\tiny{$\pm$0.6} & 50.4\tiny{$\pm$0.4} & 47.3\tiny{$\pm$1.0} & 39.8\tiny{$\pm$0.7} & 42.3\tiny{$\pm$0.4} & 39.7\tiny{$\pm$0.2} & 35.5\tiny{$\pm$0.6} & 31.7\tiny{$\pm$0.1} & 30.8\tiny{$\pm$0.7} & \textbf{55.4\tiny{$\pm$0.1}} & 52.4\tiny{$\pm$0.0} \\
 & Humanual-Opinion & \underline{45.0\tiny{$\pm$0.5}} & 39.7\tiny{$\pm$0.1} & 34.9\tiny{$\pm$0.2} & 33.8\tiny{$\pm$0.2} & 30.8\tiny{$\pm$0.1} & 25.0\tiny{$\pm$0.4} & 31.5\tiny{$\pm$0.2} & 22.8\tiny{$\pm$0.4} & 26.9\tiny{$\pm$0.8} & 24.1\tiny{$\pm$0.8} & 20.1\tiny{$\pm$0.7} & 41.3\tiny{$\pm$0.5} & \textbf{43.2\tiny{$\pm$0.2}} \\
\midrule
\multirow{7}{*}{\rotatebox{90}{EVAL}}
 & AlignX-Demo & \underline{84.0\tiny{$\pm$0.0}} & 81.7\tiny{$\pm$0.7} & 83.3\tiny{$\pm$0.9} & \textbf{90.3\tiny{$\pm$0.3}} & 89.0\tiny{$\pm$0.6} & 83.0\tiny{$\pm$1.2} & 75.7\tiny{$\pm$0.3} & 70.3\tiny{$\pm$1.2} & 82.3\tiny{$\pm$0.9} & 79.0\tiny{$\pm$1.2} & 75.0\tiny{$\pm$1.0} & 89.3\tiny{$\pm$1.8} & 89.7\tiny{$\pm$0.7} \\
 & AlignX-Pair & 59.0\tiny{$\pm$0.6} & 59.7\tiny{$\pm$2.0} & \underline{62.7\tiny{$\pm$0.3}} & 56.0\tiny{$\pm$1.7} & 51.3\tiny{$\pm$3.2} & 52.7\tiny{$\pm$0.9} & 51.0\tiny{$\pm$0.6} & 49.3\tiny{$\pm$0.7} & 58.7\tiny{$\pm$3.5} & 52.0\tiny{$\pm$0.6} & 57.0\tiny{$\pm$2.5} & \textbf{60.7\tiny{$\pm$0.9}} & 54.7\tiny{$\pm$0.9} \\
 & AlignX-UGC & 54.7\tiny{$\pm$0.3} & \underline{58.0\tiny{$\pm$1.0}} & 53.7\tiny{$\pm$0.9} & 52.3\tiny{$\pm$2.2} & 55.3\tiny{$\pm$2.0} & 51.7\tiny{$\pm$1.8} & 52.7\tiny{$\pm$0.3} & 50.0\tiny{$\pm$1.0} & \textbf{56.7\tiny{$\pm$1.7}} & 52.0\tiny{$\pm$2.1} & 49.3\tiny{$\pm$2.2} & 54.7\tiny{$\pm$2.0} & 53.0\tiny{$\pm$0.6} \\
 & AlignX-Arb & \underline{76.7\tiny{$\pm$0.9}} & 73.0\tiny{$\pm$0.0} & 72.7\tiny{$\pm$0.3} & 76.0\tiny{$\pm$0.6} & 70.3\tiny{$\pm$0.3} & 69.3\tiny{$\pm$0.9} & 69.3\tiny{$\pm$0.9} & 57.7\tiny{$\pm$1.3} & 73.3\tiny{$\pm$0.9} & 73.3\tiny{$\pm$2.3} & 65.7\tiny{$\pm$1.2} & \textbf{77.3\tiny{$\pm$0.7}} & 75.7\tiny{$\pm$0.3} \\
 & AlignX-Hist16 & \underline{84.0\tiny{$\pm$0.6}} & \underline{84.0\tiny{$\pm$1.0}} & 83.3\tiny{$\pm$0.9} & \textbf{91.3\tiny{$\pm$1.2}} & 91.0\tiny{$\pm$0.0} & 84.7\tiny{$\pm$0.9} & 83.0\tiny{$\pm$2.0} & 66.7\tiny{$\pm$2.8} & 81.3\tiny{$\pm$0.3} & 83.0\tiny{$\pm$1.2} & 77.7\tiny{$\pm$0.7} & 89.3\tiny{$\pm$1.2} & 90.7\tiny{$\pm$0.3} \\
 & SocSci210 & \underline{77.6\tiny{$\pm$0.6}} & 77.0\tiny{$\pm$0.2} & 77.0\tiny{$\pm$0.4} & 74.7\tiny{$\pm$0.4} & 72.2\tiny{$\pm$0.4} & 72.8\tiny{$\pm$0.5} & 74.1\tiny{$\pm$0.6} & 70.0\tiny{$\pm$0.4} & 73.0\tiny{$\pm$0.5} & 74.5\tiny{$\pm$0.7} & 71.5\tiny{$\pm$1.0} & \textbf{75.3\tiny{$\pm$0.9}} & 75.3\tiny{$\pm$0.3} \\
 & HumanLLM & 47.5\tiny{$\pm$0.2} & 45.0\tiny{$\pm$0.4} & \underline{48.3\tiny{$\pm$0.1}} & 40.5\tiny{$\pm$0.3} & 36.5\tiny{$\pm$0.5} & 33.1\tiny{$\pm$0.8} & 33.1\tiny{$\pm$0.1} & 21.9\tiny{$\pm$2.4} & 27.9\tiny{$\pm$4.3} & 35.0\tiny{$\pm$0.6} & 27.8\tiny{$\pm$2.3} & \textbf{41.6\tiny{$\pm$0.8}} & 41.5\tiny{$\pm$0.6} \\
\midrule
 & Overall & \underline{64.9\tiny{$\pm$0.1}} & 64.1\tiny{$\pm$0.1} & 62.5\tiny{$\pm$0.1} & 58.8\tiny{$\pm$0.1} & 56.0\tiny{$\pm$0.3} & 53.9\tiny{$\pm$0.3} & 49.6\tiny{$\pm$0.4} & 38.2\tiny{$\pm$0.6} & 50.5\tiny{$\pm$1.1} & 49.7\tiny{$\pm$0.2} & 45.7\tiny{$\pm$0.5} & \textbf{65.7\tiny{$\pm$0.2}} & 63.7\tiny{$\pm$0.0} \\
\bottomrule
\end{tabular}
}
\caption{\textbf{Detailed SOUL results across 23 datasets.}
Datasets are grouped by conversational interaction, social simulation,
cognition, role play, and evaluation/judgment. Entries are means $\pm$
standard errors over three independently seeded evaluation runs, with
GPT-5.5 used as the assistant and judge where required. Higher is better.
Bold and underlining identify the highest means among open-weight models
and frontier API models, respectively.}
  \label{tab:soul-per-dataset}
\end{table*}

\subsection{RealUserSim PT3}
\label{app:pt3}
PT3 compares human and simulated conversation trajectories along five behavioral dimensions~\citep{zhu2026realusersim}.
The published benchmark contains 600 WildChat conversations.
For dimension $d$, let $m_{j,d}\in\{0,1\}$ indicate a behavioral match on example $j$.
The dimension score and Fidelity Index are
\begin{equation}
 F_d=\frac{100}{N}\sum_{j=1}^{N}m_{j,d},
 \qquad \mathrm{FI}=\frac{1}{5}\sum_{d=1}^{5}F_d.
 \label{eq:pt3-fidelity}
\end{equation}
MIMESIS-9B achieves a Fidelity Index of 94.0, compared with 89.7 for MIMESIS-4B and 80.6 for Claude-Opus-5, the strongest baseline (Table~\ref{tab:realusersim-pt3}).
The 13.4-point gain over Claude-Opus-5 is accompanied by the highest  mean on each of the five dimensions.
The differences are especially large in interaction and information flow (91.3 versus 69.5) and pacing (91.2 versus 72.5); persona scores are close to ceiling for several models.
\begin{table}[!htbp]
\centering
\small
\setlength{\tabcolsep}{20pt}
\resizebox{\textwidth}{!}{
\begin{tabular}{lcccccc}
\toprule
\multirow{2}{*}{Simulator} & \multicolumn{6}{c}{PT3 fidelity, match rate (\%) $\uparrow$} \\
\cmidrule(lr){2-7}
 & FI & Persona & Style & Tech & Interact & Pacing \\
\midrule
\multicolumn{7}{l}{\textit{Frontier APIs}} \\
\quad Claude-Opus-5 & 80.6\tiny{$\pm$0.2} & 96.0\tiny{$\pm$0.2} & 72.5\tiny{$\pm$0.5} & 92.7\tiny{$\pm$0.4} & 69.5\tiny{$\pm$0.8} & 72.5\tiny{$\pm$0.3} \\
\quad GPT-5.5 & 76.5\tiny{$\pm$0.6} & 99.3\tiny{$\pm$0.3} & 73.3\tiny{$\pm$0.4} & 95.1\tiny{$\pm$0.0} & 65.0\tiny{$\pm$1.3} & 49.6\tiny{$\pm$1.1} \\
\quad Gemini-3.8-Flash & 67.2\tiny{$\pm$0.3} & 95.4\tiny{$\pm$0.5} & 57.6\tiny{$\pm$0.3} & 90.7\tiny{$\pm$0.6} & 52.2\tiny{$\pm$0.6} & 40.3\tiny{$\pm$1.3} \\
\midrule
\multicolumn{7}{l}{\textit{Released simulators}} \\
\quad Ditto-8B & 71.3\tiny{$\pm$0.1} & 92.3\tiny{$\pm$0.4} & 69.8\tiny{$\pm$0.3} & 93.0\tiny{$\pm$0.4} & 57.4\tiny{$\pm$0.1} & 44.2\tiny{$\pm$0.7} \\
\quad Osim-8B & 57.7\tiny{$\pm$0.3} & 82.3\tiny{$\pm$0.7} & 46.0\tiny{$\pm$0.9} & 85.2\tiny{$\pm$0.3} & 40.5\tiny{$\pm$0.6} & 34.5\tiny{$\pm$0.5} \\
\quad Osim-4B & 54.1\tiny{$\pm$0.9} & 75.9\tiny{$\pm$1.2} & 42.8\tiny{$\pm$0.7} & 82.7\tiny{$\pm$2.0} & 37.0\tiny{$\pm$0.5} & 32.0\tiny{$\pm$0.3} \\
\quad HumanLM-8B & 28.4\tiny{$\pm$0.5} & 50.2\tiny{$\pm$0.6} & 17.0\tiny{$\pm$0.9} & 54.3\tiny{$\pm$1.4} & 11.5\tiny{$\pm$0.2} & 9.0\tiny{$\pm$0.1} \\
\quad Sotopia-7B & 20.7\tiny{$\pm$0.2} & 33.4\tiny{$\pm$0.7} & 7.1\tiny{$\pm$0.6} & 54.3\tiny{$\pm$0.8} & 4.7\tiny{$\pm$0.6} & 3.9\tiny{$\pm$0.4} \\
\midrule
\multicolumn{7}{l}{\textit{Pretrained}} \\
\quad Qwen3-8B & 44.4\tiny{$\pm$1.4} & 71.8\tiny{$\pm$1.7} & 37.6\tiny{$\pm$1.6} & 70.5\tiny{$\pm$2.1} & 24.3\tiny{$\pm$0.7} & 17.7\tiny{$\pm$1.1} \\
\quad Qwen3.5-9B & 39.4\tiny{$\pm$0.6} & 68.4\tiny{$\pm$0.9} & 29.8\tiny{$\pm$0.6} & 59.8\tiny{$\pm$0.8} & 19.8\tiny{$\pm$0.4} & 19.2\tiny{$\pm$1.0} \\
\quad Qwen3.5-4B & 37.1\tiny{$\pm$0.3} & 60.1\tiny{$\pm$1.2} & 26.4\tiny{$\pm$1.3} & 59.8\tiny{$\pm$0.4} & 20.0\tiny{$\pm$0.5} & 19.2\tiny{$\pm$0.7} \\
\midrule
\multicolumn{7}{l}{\textit{Ours}} \\
\quad \textbf{MIMESIS-9B} & \textbf{94.0\tiny{$\pm$0.3}} & \textbf{99.4\tiny{$\pm$0.1}} & \textbf{89.5\tiny{$\pm$0.6}} & \textbf{98.8\tiny{$\pm$0.2}} & \textbf{91.3\tiny{$\pm$0.7}} & \textbf{91.2\tiny{$\pm$0.1}} \\
\quad \textbf{MIMESIS-4B} & 89.7\tiny{$\pm$0.1} & 98.7\tiny{$\pm$0.1} & 83.5\tiny{$\pm$0.5} & 96.5\tiny{$\pm$0.1} & 85.0\tiny{$\pm$0.4} & 84.9\tiny{$\pm$0.2} \\
\bottomrule
\end{tabular}
}
\caption{\textbf{Trajectory-level behavioral fidelity on RealUserSim PT3.}
Scores are percentages of behavioral matches to human reference
trajectories across persona, linguistic style, technical competency,
interaction and information flow, and pacing. The Fidelity Index (FI)
averages these five dimensions. MIMESIS-9B has the highest reported mean
on every dimension. Higher is better, \textbf{bold} marks the highest mean in
each column.}
\label{tab:realusersim-pt3}
\end{table}
\FloatBarrier

\subsection{Behavioral and outcome alignment on $\tau$-bench}
\label{app:tau-usi}
Following \citet{zhou2026sim2real}, we evaluate simulator–human behavioral alignment along four dimensions of the User-Sim Index (USI). Communication style (D1) captures surface-level interaction characteristics such as verbosity, politeness, and stylistic variation. Information pattern (D2) captures how much task-relevant information users disclose and how that information is distributed across turns. Clarification behavior (D3) measures how users seek, provide, or respond to clarification when information is incomplete or ambiguous. Error reaction (D4) captures how users respond when the agent makes a mistake or fails to satisfy the request, including expressions of dissatisfaction and changes in interaction strategy. These dimensions measure complementary aspects of interactive behavior rather than task success itself. For each behavioral feature \(m\), we compare a simulator statistic $M_m$ with the human reference $H_m$ using
\begin{equation}
 \mathrm{Dice}_m=100\,\frac{2\min(M_m,H_m)}{M_m+H_m},
 \label{eq:behavior-dice}
\end{equation}
with a score of 100 when both values are zero. Each D1–D4 score is the mean Dice coefficient over the features associated with that dimension. Higher scores therefore indicate closer agreement with the corresponding aggregate human behavior.
Expected calibration error (ECE) measures a complementary form of alignment at the task-outcome level. Specifically, it measures the weighted absolute difference between task-success rates obtained with simulated and human users across task-difficulty bins. We combine the four behavioral dimensions with outcome calibration into
\begin{equation}
 \mathrm{USI}_5=\frac{\mathrm{D1}+\mathrm{D2}+\mathrm{D3}+\mathrm{D4}+100(1-\mathrm{ECE})}{5}.
 \label{eq:usi-five}
\end{equation}
This five-component index omits the survey-based evaluative-alignment component of the original six-component USI because the required post-interaction human annotations are unavailable in our evaluation setting. The behavioral dimensions and ECE should therefore be interpreted separately as well as through the aggregate score: D1–D4 measure similarity in interaction behavior, whereas ECE measures whether simulator interactions reproduce the task-success patterns observed with human users. This distinction follows the original USI framework, which treats behavioral and outcome alignment as complementary properties.

MIMESIS-9B achieves a $\tau$-USI score of 80.17, exceeding GPT-5.5 (76.69), the strongest frontier baseline, and approaching Osim-8B (80.44; Table~\ref{tab:tau-usi}).
The component scores reveal different strengths.
Osim-8B leads communication-style alignment (D1, 62.12), which measures agreement with human users in features such as verbosity, politeness, and stylistic variation.
%This result is consistent with the effects of behavioral mid-training reported by \citet{zhou2026odyssim}, which brings response length, formatting, and word choice closer to human references.
MIMESIS-9B achieves the highest alignment in responses to agent errors (D4, 92.54), while MIMESIS-4B leads information disclosure (D2, 92.40) and Osim-4B leads clarification behavior (D3, 80.04).
\begin{table}[!htbp]
\centering
\small

\newcommand{\std}[1]{{\scriptsize$\pm$#1}}
\setlength{\tabcolsep}{10pt}
\resizebox{1.\textwidth}{!}{
\begin{tabular}{llcccccc}
\toprule
& Simulator & D1 Conv $\uparrow$ & D2 Info $\uparrow$ & D3 Clarif $\uparrow$ & D4 React $\uparrow$ & ECE $\downarrow$ & USI$_5$ $\uparrow$ \\
\midrule
\multicolumn{8}{l}{\textit{Frontier (API)}} \\
& GPT-5.5 & 55.20\tiny{$\pm$6.92} & 89.78\tiny{$\pm$0.26} & 74.51\tiny{$\pm$2.21} & 82.78\tiny{$\pm$1.26} & 0.188\tiny{$\pm$0.000} & 76.69\tiny{$\pm$1.83} \\
% & gemini-3.6-flash & 55.60\tiny{$\pm$5.10} & 81.44\tiny{$\pm$0.14} & 76.64\tiny{$\pm$3.14} & 75.11\tiny{$\pm$2.26} & 0.159\tiny{$\pm$0.011} & 74.58\tiny{$\pm$1.31} \\
& Gemini-3.8-Flash & 59.81\tiny{$\pm$0.72} & 85.71\tiny{$\pm$0.53} & 67.89\tiny{$\pm$3.16} & 71.43\tiny{$\pm$2.43} & 0.192\tiny{$\pm$0.010} & 73.13\tiny{$\pm$0.92} \\
% & claude-sonnet-5 & 46.55\tiny{$\pm$0.72} & 81.18\tiny{$\pm$1.31} & 71.25\tiny{$\pm$3.51} & 73.49\tiny{$\pm$2.69} & \textbf{0.114\tiny{$\pm$0.021}} & 72.22\tiny{$\pm$0.83} \\
& Claude-Opus-5 & 48.06\tiny{$\pm$0.59} & 88.52\tiny{$\pm$0.09} & 70.40\tiny{$\pm$1.43} & 60.51\tiny{$\pm$2.35} & 0.208\tiny{$\pm$0.005} & 69.34\tiny{$\pm$0.57} \\
%& gemini-3.5-flash & 58.16\tiny{$\pm$1.38} & 76.34\tiny{$\pm$0.23} & 65.96\tiny{$\pm$0.92} & 60.53\tiny{$\pm$0.99} & 0.162\tiny{$\pm$0.013} & 68.97\tiny{$\pm$0.44} \\
\midrule
\multicolumn{8}{l}{\textit{Baseline}} \\
& Osim-8B & \textbf{62.12\tiny{$\pm$5.24}} & 85.92\tiny{$\pm$0.79} & 77.13\tiny{$\pm$1.07} & {90.49\tiny{$\pm$0.87}} & 0.135\tiny{$\pm$0.004} & \textbf{80.44\tiny{$\pm$1.16}} \\
& Ditto-8B & 52.59\tiny{$\pm$0.97} & 86.13\tiny{$\pm$1.20} & 54.48\tiny{$\pm$2.58} & 77.18\tiny{$\pm$5.73} & 0.166\tiny{$\pm$0.018} & 70.75\tiny{$\pm$1.08} \\
& Osim-4B & 32.26\tiny{$\pm$1.27} & 71.24\tiny{$\pm$0.58} & \textbf{80.04\tiny{$\pm$0.82}} & 83.72\tiny{$\pm$2.44} & 0.184\tiny{$\pm$0.012} & 69.78\tiny{$\pm$0.54} \\
& Sotopia-7B & 30.49\tiny{$\pm$0.91} & 52.38\tiny{$\pm$0.33} & 76.43\tiny{$\pm$2.79} & 67.13\tiny{$\pm$3.70} & 0.202\tiny{$\pm$0.021} & 61.25\tiny{$\pm$1.40} \\
& Qwen3.5-4B & 40.24\tiny{$\pm$1.40} & 66.16\tiny{$\pm$1.14} & 58.66\tiny{$\pm$2.06} & 55.94\tiny{$\pm$3.98} & 0.169\tiny{$\pm$0.008} & 60.82\tiny{$\pm$1.70} \\
& Qwen3.5-9B & 35.92\tiny{$\pm$0.96} & 64.58\tiny{$\pm$0.94} & 59.60\tiny{$\pm$2.51} & 46.37\tiny{$\pm$4.86} & 0.174\tiny{$\pm$0.010} & 57.82\tiny{$\pm$1.52} \\
& Qwen3-8B & 58.73\tiny{$\pm$1.07} & 79.03\tiny{$\pm$0.44} & 51.54\tiny{$\pm$4.68} & 58.74\tiny{$\pm$6.45} & 0.469\tiny{$\pm$0.005} & 60.23\tiny{$\pm$1.93} \\
\midrule
\multicolumn{8}{l}{\textit{Ours}} \\
& {MIMESIS-9B} & 51.75\tiny{$\pm$1.64} & 87.60\tiny{$\pm$0.27} & 75.84\tiny{$\pm$3.36} & \textbf{92.54\tiny{$\pm$2.07}} & \textbf{0.069\tiny{$\pm$0.016}} & {80.17\tiny{$\pm$0.21}} \\
& MIMESIS-4B  & 56.33\tiny{$\pm$6.03} & \textbf{92.40\tiny{$\pm$0.84}} & 78.13\tiny{$\pm$1.71} & 77.04\tiny{$\pm$4.80} & 0.137\tiny{$\pm$0.015} & 78.03\tiny{$\pm$1.07} \\
\midrule
& \textit{Human (inter-ann.)} & 87.35 & 97.94 & 88.00 & 93.53 & 0.081 & 91.75 \\
\bottomrule
\end{tabular}
}
% Table 9 | tab:tau-usi | section/appendix/B_simulator_evaluation.tex
\caption{\textbf{Behavioral alignment and task-outcome calibration on $\tau$-bench.}
D1--D4 measure agreement with human users in communication style,
information disclosure, clarification, and responses to agent errors.
ECE measures discrepancies between simulated-user and human-user success
rates. The five-component index $\mathrm{USI}_5$ combines D1--D4 with
$100(1-\mathrm{ECE})$ and excludes survey-based evaluative alignment.
Higher D1--D4 and $\mathrm{USI}_5$ scores and lower ECE are better.
\textbf{Bold} marks the best model mean, excluding the human reference.}
\label{tab:tau-usi}
\end{table}

MIMESIS-9B also achieves the lowest expected calibration error (ECE) among the evaluated simulators: 0.069, compared with 0.135 for Osim-8B and 0.188 for GPT-5.5.
ECE measures the weighted absolute difference between the same agent's success rates with simulated and human users across task-difficulty bins~\citep{zhou2026sim2real}.
The lower ECE indicates that interactions with MIMESIS more closely reproduce the task outcomes observed with human users.
This distinction matters for agent training: matching communication style does not ensure that a simulator preserves interaction difficulty, and systematic differences in difficulty alter the success and failure signals used to optimize the agent.
These results support MIMESIS as a better-calibrated human proxy for evaluation and motivate its use for agent training, where it push the generalization of the obtained agent  to unseen user models the main paper.
\clearpage
\subsection{SimulatorArena}
\label{app:simulator-arena}
SimulatorArena~\citep{dou2025simulatorarena} measures writing-style and interaction-style similarity on a 1-5 scale and evaluates whether a judge can distinguish human from simulated conversations.
We report Turing distance as $|a-50|$, where $a$ is the judge's classification accuracy in percent, zero denotes chance-level discrimination, and 50 denotes maximal separation from chance.

\begin{table}[!htbp]
\centering
\small
\setlength{\tabcolsep}{15pt}
\resizebox{.9\textwidth}{!}{
\begin{tabular}{llccc}
\toprule
& Simulator & Writing style\ $\uparrow$ & Interaction style\ $\uparrow$ & Turing $|acc{-}50|\downarrow$ \\
\midrule
\multicolumn{5}{l}{\textit{Frontier (API)}} \\

& Claude-Opus-5 & \textbf{3.92\tiny{$\pm$0.01}} & 3.90\tiny{$\pm$0.02} & 42.3\tiny{$\pm$1.9} \\

& Gemini-3.8-Flash & 3.88\tiny{$\pm$0.01} & 3.91\tiny{$\pm$0.03} & 46.7\tiny{$\pm$0.7} \\
& GPT-5.5 & 3.81\tiny{$\pm$0.02} & \textbf{3.92\tiny{$\pm$0.01}} & 48.7\tiny{$\pm$0.3} \\
\midrule
\multicolumn{5}{l}{\textit{Baseline}} \\
& Ditto-8B & 3.23\tiny{$\pm$0.01} & 3.16\tiny{$\pm$0.02} & 43.3\tiny{$\pm$1.2} \\
& Qwen3.5-9B & 2.33\tiny{$\pm$0.02} & 2.65\tiny{$\pm$0.01} & 46.3\tiny{$\pm$2.1} \\
& Qwen3.5-4B & 2.93\tiny{$\pm$0.06} & 2.72\tiny{$\pm$0.04} & 47.7\tiny{$\pm$0.3} \\
& Osim-8B & 3.31\tiny{$\pm$0.01} & 3.28\tiny{$\pm$0.03} & 48.7\tiny{$\pm$0.3} \\
& Qwen3-8B & 3.25\tiny{$\pm$0.02} & 3.23\tiny{$\pm$0.03} & 49.3\tiny{$\pm$0.3} \\
& Osim-4B & 3.07\tiny{$\pm$0.04} & 3.00\tiny{$\pm$0.02} & 49.7\tiny{$\pm$0.3} \\
& Sotopia-7B & 2.78\tiny{$\pm$0.02} & 3.04\tiny{$\pm$0.03} & 50.0\tiny{$\pm$0.0} \\
& HumanLM-8B & 3.15\tiny{$\pm$0.01} & 2.92\tiny{$\pm$0.02} & 50.0\tiny{$\pm$0.0} \\
\midrule
\multicolumn{5}{l}{\textit{Ours}} \\
& MIMESIS-9B  & 3.70\tiny{$\pm$0.00} & 3.73\tiny{$\pm$0.03} & \textbf{38.7\tiny{$\pm$0.7}} \\
& MIMESIS-4B & 3.71\tiny{$\pm$0.02} & 3.52\tiny{$\pm$0.02} & 42.0\tiny{$\pm$1.5} \\
\bottomrule
\end{tabular}
}
\caption{\textbf{Message similarity and Turing distance on SimulatorArena.}
Writing-style and interaction-style similarity are rated on a 1--5 scale
(higher is better). Turing distance is $|a-50|$, where $a$ is the judge's
accuracy in percent when distinguishing human and simulated conversations;
lower is better, with zero corresponding to chance discrimination.}
\label{tab:simarena-realism}
\end{table}
\FloatBarrier

As can be seen from Table~\ref{tab:simarena-realism}, MIMESIS-9B achieves the lowest Turing distance among the evaluated models (38.7), followed by MIMESIS-4B (42.0) and Claude-Opus-5 (42.3).
The 3.6-point reduction over the strongest baseline improves realism under this judging protocol, although discrimination remains far from chance.
MIMESIS-9B also exceeds the best released-simulator baselines in writing-style similarity (3.70 vs.\ 3.31) and interaction-style similarity (3.73 vs.\ 3.28).
Frontier models retain the highest explicit style ratings, with Claude-Opus-5 scoring 3.92 for writing style and GPT-5.5 scoring 3.92 for interaction style. On this benchmark, frontier models score higher on explicit style similarity, while MIMESIS is harder to distinguish from human conversations. The two
metrics therefore rank the models differently: matching salient stylistic
features does not necessarily make a simulated conversation difficult to
distinguish from a human one. This discrepancy provides additional evidence
that surface-style similarity captures only one aspect of user realism.
% \clearpage
\subsection{PRISM: behavioral fidelity and next-turn realism}

We evaluate open-domain next-turn simulation on PRISM conversations
\citep{kirk2024prism}, covering 128 users and 880 conversational turns.
Table~\ref{tab:prism-usi} evaluates alignment using the four behavioral
dimensions defined in $\tau$-USI: communication style (D1),
information pattern (D2), clarification behavior (D3), and error reaction
(D4). Because PRISM does not provide the annotations required for outcome
calibration or survey-based evaluative alignment, we report the mean of
D1--D4 as the aggregate behavioral score.
MIMESIS-9B obtains the highest point estimate (80.1), close to Gemini-3.8-Flash (79.7) and 9.4 points above Osim-8B, the strongest released-simulator baseline (70.7).
Its clearest advantage is in communication style: the D1 score of 76.5 exceeds the strongest baseline, Claude-Opus-5, by 8.7 points.
MIMESIS-4B scores 75.7 overall and achieves the highest information-disclosure alignment (D2, 94.9).

\begin{table}[!htbp]\centering\small
\setlength{\tabcolsep}{15pt}
\resizebox{\textwidth}{!}{\begin{tabular}{llccccc}
\toprule
Type & Model & D1 Conv & D2 Info & D3 Clarif & D4 React & Mean \\
\midrule
Frontier & Claude-Opus-5 & 67.8\tiny{$\pm$1.3} & 87.4\tiny{$\pm$1.0} & 71.3\tiny{$\pm$1.5} & 74.6\tiny{$\pm$4.3} & 75.3\tiny{$\pm$1.9} \\
 & Gemini-3.8-Flash & 64.7\tiny{$\pm$0.8} & 93.9\tiny{$\pm$0.3} & \textbf{81.6\tiny{$\pm$2.5}} & 78.5\tiny{$\pm$5.8} & 79.7\tiny{$\pm$1.1} \\
 & GPT-5.5 & 60.3\tiny{$\pm$4.9} & 81.6\tiny{$\pm$0.2} & 67.7\tiny{$\pm$2.3} & 66.3\tiny{$\pm$2.2} & 69.0\tiny{$\pm$1.2} \\
\midrule
Released & Osim-8B & 61.8\tiny{$\pm$5.6} & 82.0\tiny{$\pm$0.7} & 69.7\tiny{$\pm$1.6} & 69.4\tiny{$\pm$2.6} & 70.7\tiny{$\pm$1.6} \\
 & Osim-4B & 37.2\tiny{$\pm$2.5} & 34.8\tiny{$\pm$1.5} & 57.5\tiny{$\pm$0.9} & 35.8\tiny{$\pm$2.6} & 41.3\tiny{$\pm$1.4} \\
 & Ditto-8B & 52.7\tiny{$\pm$2.8} & 58.8\tiny{$\pm$2.2} & 78.0\tiny{$\pm$0.8} & \textbf{80.5\tiny{$\pm$2.8}} & 67.5\tiny{$\pm$1.7} \\
 & HumanLM-8B & 48.0\tiny{$\pm$4.1} & 55.7\tiny{$\pm$2.5} & 50.5\tiny{$\pm$1.9} & 76.2\tiny{$\pm$6.4} & 57.6\tiny{$\pm$3.1} \\
 & Sotopia-7B & 36.3\tiny{$\pm$0.7} & 48.1\tiny{$\pm$0.2} & 44.4\tiny{$\pm$4.8} & 55.7\tiny{$\pm$1.0} & 46.1\tiny{$\pm$1.1} \\
\midrule
Pretrained & Qwen3.5-9B & 52.7\tiny{$\pm$1.0} & 71.3\tiny{$\pm$2.6} & 63.7\tiny{$\pm$2.0} & 60.8\tiny{$\pm$0.9} & 62.1\tiny{$\pm$0.3} \\
 & Qwen3-8B & 58.0\tiny{$\pm$4.2} & 63.4\tiny{$\pm$0.1} & 57.3\tiny{$\pm$4.7} & 63.6\tiny{$\pm$1.3} & 60.6\tiny{$\pm$2.1} \\
 & Qwen3.5-4B & 43.7\tiny{$\pm$5.7} & 69.1\tiny{$\pm$0.4} & 64.1\tiny{$\pm$2.0} & 55.2\tiny{$\pm$0.1} & 58.0\tiny{$\pm$1.4} \\
\midrule
Ours & {MIMESIS-9B} & \textbf{76.5\tiny{$\pm$3.9}} & 86.7\tiny{$\pm$0.3} & 77.5\tiny{$\pm$0.5} & 79.8\tiny{$\pm$8.8} & \textbf{80.1\tiny{$\pm$3.1}} \\
 & {MIMESIS-4B} & 52.9\tiny{$\pm$0.7} & \textbf{94.9\tiny{$\pm$0.6}} & 79.5\tiny{$\pm$1.8} & 75.4\tiny{$\pm$3.6} & 75.7\tiny{$\pm$1.3} \\
\bottomrule
\end{tabular}}
\caption{\textbf{Behavioral alignment on open-domain PRISM conversations.}
Scores measure Dice--S{\o}rensen agreement with human behavioral features
over 128 users and 880 turns. D1--D4 assess communication style, information
disclosure, clarification, and responses to agent errors; Mean averages
these four dimensions. Outcome calibration and survey alignment are
excluded because the required annotations are unavailable. Higher is
better; bold marks the highest mean in each column.
}
\label{tab:prism-usi}
\end{table}
\FloatBarrier

We also compare the realism of individual next-user responses using three LLM judges (Figure~\ref{fig:turing-pairwise-judges}).
MIMESIS receives more wins than losses in all 24 judge--baseline comparisons.
Against Osim-8B, its win/loss rates are 50.3/14.0\% under GPT-5.6, 60.6/15.2\% under Claude-Opus-5, and 42.6/24.2\% under Gemini-3.8-Flash.
Gemini assigns smaller preference margins across all eight baselines.
The narrowest margin is against Qwen3.5-9B, with 36.8\% wins and 30.1\% losses; this comparison is marked nonsignificant in the figure.
Appendix~\ref{app:qual-prism} provides a qualitative example.

\begin{figure}[!htbp]
  \centering
  \includegraphics[width=\textwidth]{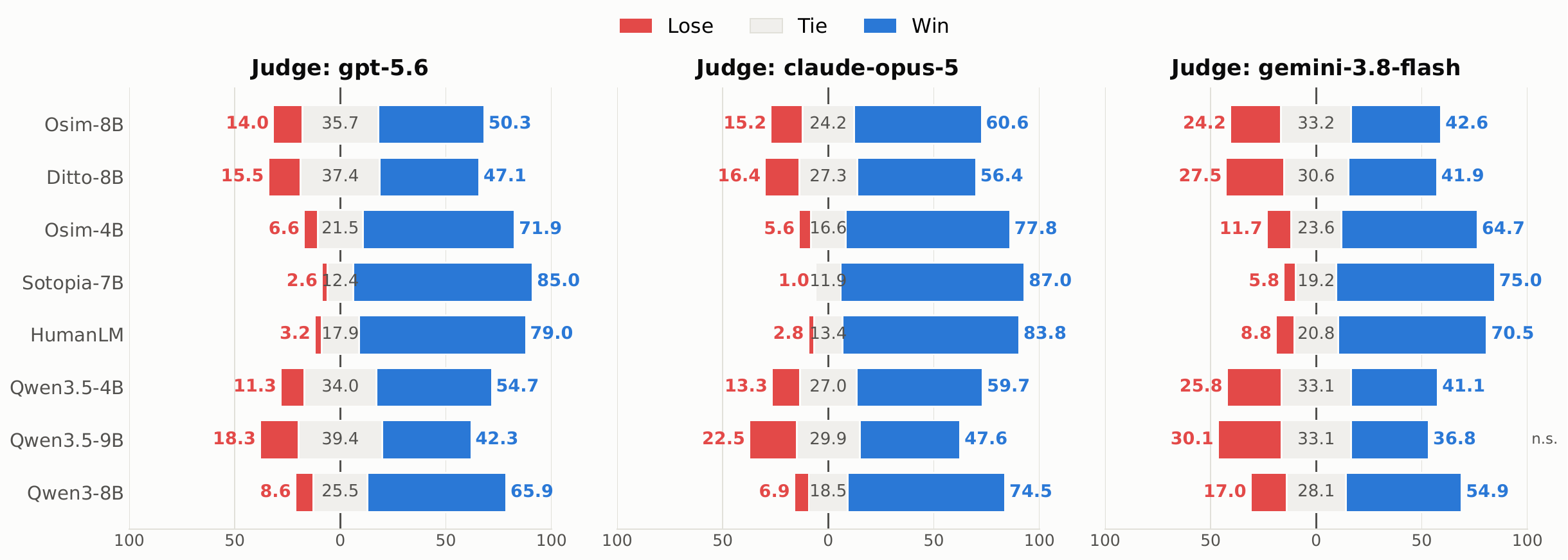}
\caption{\textbf{Pairwise next-user-response realism on PRISM.}
Each row compares MIMESIS with one baseline under the judge named above
the panel. Blue, gray, and red segments denote MIMESIS wins, ties, and
losses, respectively, with percentages annotated on the bars. MIMESIS has
more wins than losses in all displayed comparisons.}
  \label{fig:turing-pairwise-judges}
\end{figure}
\FloatBarrier

% \subsection{Interpreting uncertainty}
% SOUL reports the mean and standard error across three independently seeded evaluation runs.
% Uncertainty across evaluation replicates is distinct from variation across independently trained checkpoints.
% The remaining tables retain the uncertainty terms recorded in their evaluation summaries, while the learning curves in Appendix~\ref{app:simulator-analyses} describe individual training trajectories.
% AUTHOR CHECK: Define the estimator and replication unit of each non-SOUL
% uncertainty term, and provide the statistical test used for PRISM's n.s. label.

\section{Simulator Ablations and Training Dynamics}
\label{app:simulator-analyses}
The ablations follow the three stages of the MIMESIS training recipe: user-side mid-training, explicit reasoning supervision, and realistic-behavior training. We use them to distinguish improvements in initialization, reasoning behavior, and interaction-level behavioral control.
\subsection{Role of user-side mid-training}

Figure~\ref{fig:midtrain} examines whether user-side mid-training provides a
better initialization for simulator RL. Starting directly from the pretrained
backbone leads to early degradation in validation reward on many tasks,
whereas the mid-trained checkpoint is generally more stable during the same
optimization interval. This suggests that mid-training establishes
user-specific behavior before RL, reducing the amount of adaptation that must
be learned from task rewards alone. 
\begin{figure}[!ht]
    \centering{
\includegraphics[width=\columnwidth]{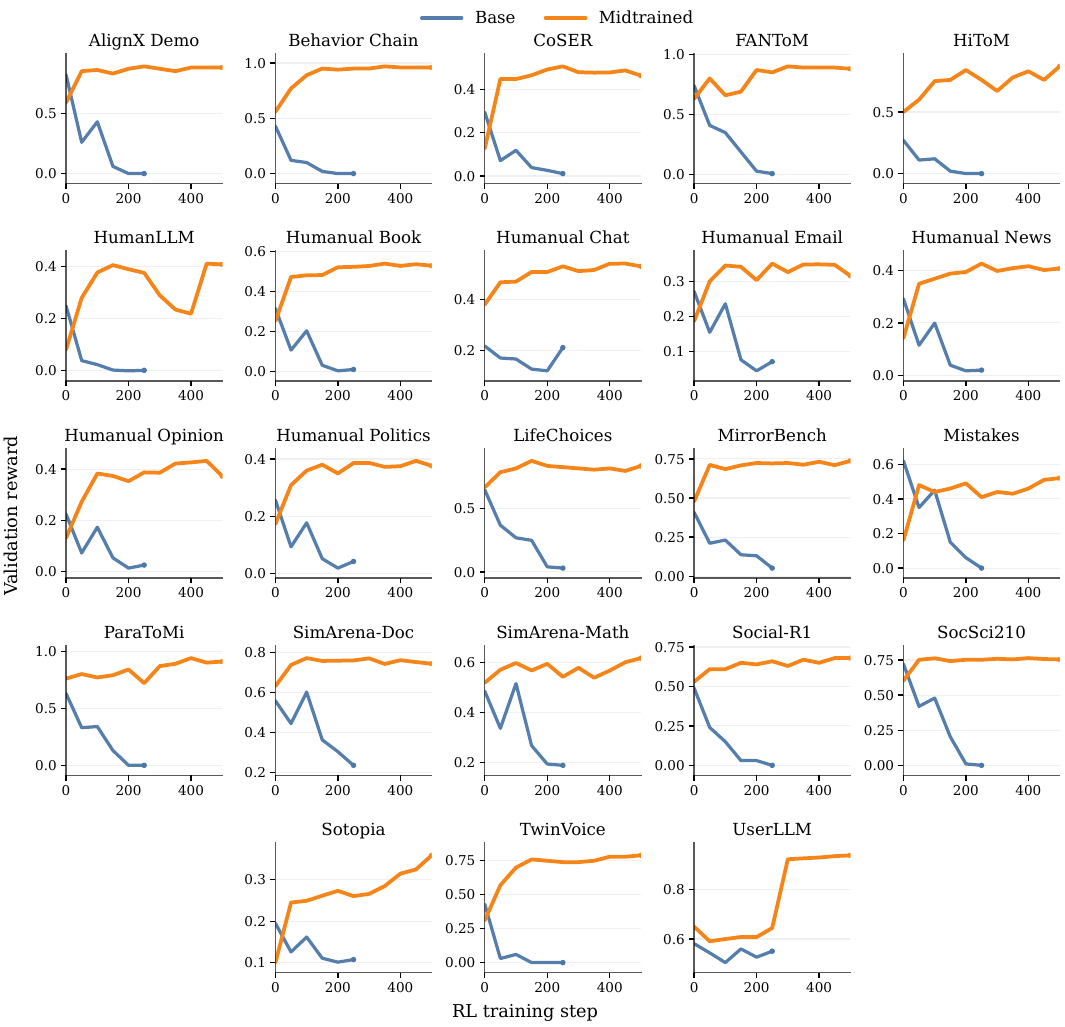}
\caption{\textbf{Effect of user-side mid-training on simulator RL.}
Validation reward across 23 tasks when RL starts from a pretrained backbone
(\textcolor{blue}{blue}) or a user-side mid-trained checkpoint (\textcolor{orange}{orange}). Mid-trained
initialization maintains higher rewards across most tasks during the
overlapping training interval. Each curve represents an individual run with
training horizons and vertical scales vary across tasks.}

    \label{fig:midtrain}}
\end{figure}
\FloatBarrier

\subsection{Role of explicit reasoning}
Figure~\ref{fig:thinking} examines reasoning generation and accuracy on \textsc{Mistakes} across training conditions.
The pretrained Qwen3.5-9B model scores 60.7 and generates a median of 540 words before answering.
After supervised adaptation, accuracy falls to 25.7 and the median length to 56 words.
Subsequent RL partially recovers accuracy to 36.0, while the median length falls to one word.
Enabling the thinking flag at inference leaves both measurements unchanged for this RL checkpoint.
Thus, changing the inference setting alone does not restore explicit reasoning in this condition. With ThoughtTrace supervision, MIMESIS-9B generates an explicit reasoning phase, with a median of 583 words, and achieves 63.0 accuracy.
The recovery of task performance coincides with the return of reasoning generation.
This comparison supports the use of explicit thought supervision during simulator adaptation.
\begin{figure}[!ht]
    \centering
    \vspace*{-2mm}
    \includegraphics[width=.85\textwidth]{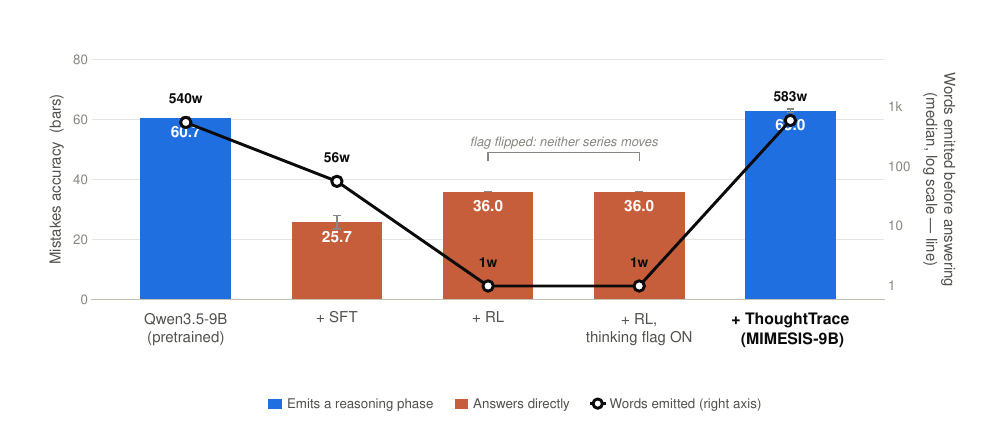}
    \vspace*{-2mm}
\caption{\textbf{Reasoning generation and task accuracy under different training conditions.}
Bars show accuracy on \textsc{Mistakes} (left axis), and the black line shows
the median number of words generated before answering (logarithmic right
axis). Blue bars indicate an explicit reasoning phase while orange bars indicate
direct answering. Enabling the inference-time thinking flag alone leaves
the RL checkpoint at 36.0 accuracy and one word, whereas the ThoughtTrace
condition reaches 63.0 accuracy and 583 words.}
    \vspace*{-2mm}
    \label{fig:thinking}
\end{figure}
\FloatBarrier

% \subsection{Effect of reasoning during simulator RL}
Figure~\ref{fig:think-ablation} compares SOUL scores logged during simulator RL with and without explicit reasoning at 4B and 9B parameters.
Reasoning-enabled runs attain higher scores at every logged checkpoint, with final gains of 1.8 and 1.5 points, respectively, after 500 optimization steps.
These training-time scores should be interpreted separately from the reported benchmark results, as the rollout and evaluation engines may use different decoding configurations.

\begin{figure}[!ht]
  \centering
    \vspace*{-2mm}
  \includegraphics[width=.85\columnwidth]{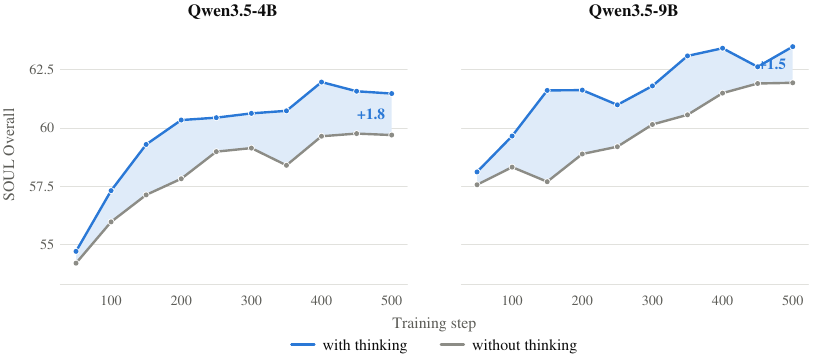}
% Figure 11 | fig:think-ablation | section/appendix/C_simulator_analyses.tex
\caption{\textbf{Effect of reasoning during simulator reinforcement learning.}
SOUL-Index trajectories with reasoning enabled (\textcolor{blue}{blue}) or disabled (\textcolor{gray}{gray})
for the 4B (left) and 9B (right) models. Reasoning-enabled runs achieve
higher scores throughout the displayed post-initialization interval, with
final gains of 1.8 and 1.5 points, respectively. }
  \label{fig:think-ablation}
\end{figure}
\FloatBarrier

\subsection{Contribution of realistic behavior training}

We ablate the realistic behavior training stage to isolate whether explicitly training on the 13 behavior categories improves user simulation beyond the rest of the MIMESIS training recipe.
As shown in Figure~\ref{fig:realistic_behavior}, incorporating realistic behavior training consistently improves simulator quality across three complementary evaluations: RealUserSim increases from 92.3 to 94.0 $(+1.7)$, $\tau$-USI increases from 73.9 to 80.2 $(+6.3)$, and SimulatorArena decreases from 41.0 to 38.7 $(-2.3)$, where lower is better.
The largest gain appears on $\tau$-USI, suggesting that explicit exposure to diverse interaction behaviors is particularly beneficial for capturing user behavior over multi-turn interactions.

\begin{figure}[!ht]
    \centering
    \includegraphics[width=\textwidth]{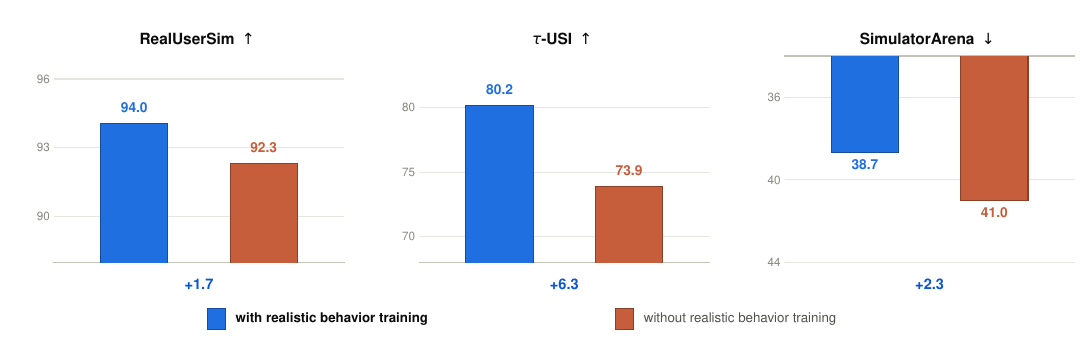}
    \caption{
    \textbf{Effect of realistic behavior training on user-simulation quality.}
    We compare otherwise identical MIMESIS training runs with and without the realistic behavior objective derived from the 13 behavior categories.
    Behavior training improves RealUserSim by 1.7 points and $\tau$-USI by 6.3 points, while reducing SimulatorArena by 2.3 points (lower is better).
    The consistent improvements across trajectory-level behavioral fidelity, interactive user simulation, and simulator--human agreement show that explicitly training for realistic behavioral variation complements the remaining simulator-training recipe.
    }
    \label{fig:realistic_behavior}
\end{figure}
\FloatBarrier
\begin{figure}[!ht]
  \centering
  \includegraphics[width=\linewidth]{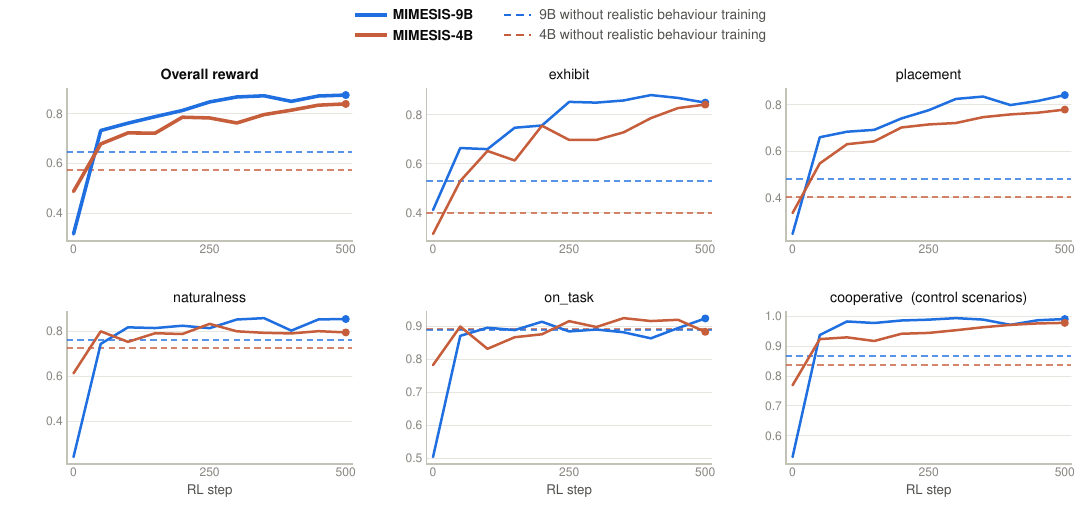}
  \caption{%
    \textbf{Decomposition of the realistic-behavior reward during RL.}
    Solid curves show \textsc{MIMESIS}-9B and \textsc{MIMESIS}-4B over
    500 RL steps on a held-out split containing 195 behavior scenarios
    (15 per category) and 49 neutral control scenarios.
    The leftmost panel reports the aggregate reward; the remaining panels
    show its components. \emph{exhibit} measures whether the assigned
    behavior is expressed, \emph{placement} whether it occurs at an
    appropriate point in the interaction, \emph{naturalness} whether the
    interaction remains natural, \emph{on\_task} whether the simulator
    remains consistent with the task, and \emph{cooperative} whether the
    simulator remains cooperative in neutral control scenarios.
    Dashed horizontal levels show evaluations of checkpoints trained with
    the same RL mixture but without the realistic-behavior task.
    %Each panel uses its own vertical scale.
  }
  \label{fig:behav-components}
\end{figure}

Training on the realistic-behavior mixture increases the aggregate behavior
reward from $0.647$ to $0.875$ for \textsc{MIMESIS}-9B and from $0.572$ to
$0.839$ for \textsc{MIMESIS}-4B. As shown in
Figure~\ref{fig:behav-components}, the largest gains come from behavior
exhibition ($+0.32$ at 9B and $+0.44$ at 4B) and temporal placement
($+0.36$ and $+0.37$, respectively). Naturalness also rises
($+0.07$--$0.10$), and task consistency is preserved ($+0.03$ at 9B and
$-0.01$ at 4B). The realistic-behavior objective thus teaches the simulator
to express the assigned behavior at an appropriate point in the interaction
while keeping its replies natural and on task.
 Importantly,
these gains do not materially reduce task adherence.

The two model sizes reach similar final scores on behavior exhibition
($0.849$ for 9B versus $0.841$ for 4B), but \textsc{MIMESIS}-9B performs
better on temporal placement ($0.842$ versus $0.779$). Both models achieve
$0.98$--$0.99$ on the neutral control scenarios, indicating that training
on difficult user behaviors does not make the simulators uniformly
uncooperative when no such behavior is specified. The early training
dynamics also reflect properties of the initialization. In particular, the
9B SFT checkpoint initially produces explicit reasoning that is penalized
by the behavior judge, and $41\%$ of its initial rollouts describe the
assigned behavior rather than expressing it in the interaction
(e.g., ``as a frustrated customer, I will now\ldots''), which triggers the
reward penalty. Both effects largely disappear within the first 50 RL
steps.
\FloatBarrier

\section{Agent Training and Evaluation}
\label{app:agent-evaluation}
This section specifies the Stage II training controls behind two comparisons in the main paper: changing the training user while holding multi-turn GRPO fixed, and adding CSD while holding the MIMESIS training user fixed. We then report environment-level results and robustness checks.
\subsection{Training setup and baselines}
We freeze the learned simulator and train a Qwen3-8B agent through multi-turn interaction in the task setting of UserRL~\citep{qian2025userrl}.
Agent training spans \textsc{TravelGym}, \textsc{TurtleGym}, \textsc{FunctionGym}, \textsc{TauGym}, and \textsc{PersuadeGym}, while evaluation additionally includes \textsc{IntentionGym}, \textsc{TelepathyGym}, and \textsc{SearchGym}, which are held out from agent training.
The user simulator remains fixed throughout agent optimization. We compare four agent-training conditions:
\begin{itemize}[leftmargin=1.4em,itemsep=2pt,topsep=4pt]
    \item \textbf{Base} denotes the Qwen3-8B agent before multi-turn reinforcement learning.
    \item \textbf{UserRL} follows the multi-turn RL procedure of UserRL~\citep{qian2025userrl}, using GPT-5.5 as the training user simulator (GRPO with GPT-5.5 in the main paper).
    \item \textbf{UserRL$^+$} keeps the same multi-turn RL procedure but replaces GPT-5.5 with MIMESIS as the training user simulator  (GRPO w. MIMESIS-9B in the main paper).
    \item \textbf{CSD} further augments MIMESIS-based agent training with our \emph{Coached On-Policy Self-Distillation} objective, which distills simulator-derived coaching signals into the agent.
\end{itemize}

We additionally compare against \textbf{InfoPO}~\citep{kong2026infopo}, a multi-turn policy-optimization method that addresses sparse trajectory-level supervision by assigning turn-level credit according to the information gained from interaction.
Specifically, InfoPO measures how observed feedback changes the policy's subsequent action distribution relative to a masked-feedback counterfactual, and adaptively combines this information-gain signal with the outcome-based advantage. For a controlled comparison with CSD, we train InfoPO using MIMESIS-9B as
the user simulator as  we found it yielded a more robust agent. This holds the training-user distribution fixed so that
InfoPO and CSD differ in their optimization objectives rather than in
simulator choice. Table~\ref{tab:rl-hparams-agent} reports the full training configuration.

These conditions separate two questions central to our study.
Comparing \textbf{UserRL} with \textbf{UserRL$^+$} isolates the effect of replacing a general-purpose LLM user with our learned simulator while holding the agent-training procedure fixed.
Comparing \textbf{UserRL$^+$} with \textbf{CSD} then isolates the additional benefit of exploiting simulator-derived feedback through the coaching objective.
InfoPO serves as an additional strong multi-turn RL baseline with an alternative mechanism for providing fine-grained credit during interaction.

\begin{table}[!htbp]
\centering
\small
\setlength{\tabcolsep}{3.5pt}

\caption{\textbf{Agent training configuration.} A Qwen3-8B actor is optimized for 100 steps against GPT-5.5 or MIMESIS-9B. CSD augments multi-turn GRPO with the coaching loss in Equation~\ref{eq:agent-combined-loss}; its auxiliary weight and gate sharpness are $\alpha=0.01$ and $\beta=5.0$.}
\label{tab:rl-hparams-agent}
\begin{tabularx}{1. \textwidth}{@{}l>{\raggedright\arraybackslash}X l>{\raggedright\arraybackslash}X@{}}
\toprule
\textbf{Parameter} & \textbf{Value} & \textbf{Parameter} & \textbf{Value} \\
\midrule
Actor  & Qwen3-8B        & Rollout batch        & 64 prompts  \\
   Parameter dtype        & \texttt{bfloat16}                  & PPO mini-batch       & 8 prompts  \\  
Policy objective       & GRPO    & Batch balancing      & enabled \\
Group size $G$         & 8          &  Dynamic batching & $9{,}344$ tokens / GPU         \\
Turn-level credit      & reward-to-go (R2G)    & Max prompt length    & 1{,}152 tokens \\
Trajectory score       & reward-to-go (R2G)    & Max response length  & 8{,}192 tokens \\
Discount $\gamma$      & 0.8                   & Max turns per episode & 16 \\
GAE $\lambda$          & 1.0                   & Training temperature & 1.0 (top-$p$ $1.0$, top-$k$ $-1$) \\
Action credit ratio    & 0.8                   & Validation temperature & 0 \\
Dual-clip $c$          & 3.0                   & Training gyms        & Travel, $\tau$-bench, Turtle, \\
PPO clip $(\epsilon_{\text{low}},\epsilon_{\text{high}})$ & $(0.2,\,0.2)$ & & Persuade, Function \\
KL penalty             & disabled              & User simulator       & MIMESIS-9B \\
Auxiliary weight $\alpha$ & 0.01   &   Rollout engine & SGLang, synchronous, TP $=1$ \\
Gate sharpness $\beta$ & 5.0                   & Hardware             & $2\times8$ A100-80GB, FSDP \\
Hintable turns / episode $k$ & 16              & Sequence parallel (Ulysses) & 1 \\
Optimizer              & AdamW                 &  Training steps & 100 \\
Learning rate          & $1\times10^{-6}$ &     GPU memory utilization & 0.50 \\
Weight decay           & 0.01     &  Micro-batch per GPU  & 2   \\
Gradient clipping      & 1.0                   &                      & \\
\bottomrule
\end{tabularx}
\end{table}
\FloatBarrier
\subsection{CSD implementation details}
For each sampled agent response, the simulator generates a thought and the subsequent public user utterance.
These outputs serve complementary roles: the thought expresses the simulator's interpretation and remaining needs, while the utterance records how it continues the interaction.
A coaching model uses both outputs, together with the preceding dialogue and the sampled response, to identify how the agent could better address those needs and respond to interaction feedback.
The resulting note conditions a teacher copy of the agent during the parameter update.

The teacher and student score the same on-policy response, with the teacher additionally receiving the coaching note.
Their token-level log-probability differences determine the weights in the CSD auxiliary loss.
This provides dense supervision alongside the scalar task reward without requiring a newly generated replacement response.
The note is retrospective: it is constructed after the simulator's next turn and is unavailable when the original agent response is sampled.
Only the public dialogue conditions the rollout and deployed policies.
The generated thoughts are simulator outputs; the procedure assumes no access to real users' thoughts at deployment.

\subsection{Generalization across evaluation users}
\label{app:agent-transfer}
Tables~\ref{tab:agent-released} and~\ref{tab:agent-frontier} report agent performance on five held-in and three held-out environments under nine user models, none of which is used during agent training.
Replacing GPT-5.5 with MIMESIS as the training simulator (UserRL to UserRL$^+$) improves the aggregate score under every evaluation user, and CSD provides further gains in all nine cases (Table~\ref{tab:agent-summary}).
The mean across users increases from 26.10 to 29.54 and then to 31.09, corresponding to mean per-user improvements of 3.43 and 1.56 points.
% \clearpage

\begin{table}[!ht]
\centering
\small
\setlength{\tabcolsep}{15pt}
\begin{tabular}{@{}lrrrrr@{}}
\toprule
Evaluation user & UserRL & UserRL$^+$ & CSD & $\Delta_{\rm sim}$ & $\Delta_{\rm CSD}$ \\
\midrule
GPT-5.6 & 26.85 & 31.10 & \textbf{32.39} & +4.25 & +1.29 \\
Claude-Opus-5 & 25.75 & 27.21 & \textbf{29.78} & +1.46 & +2.57 \\
Claude-Sonnet-5 & 26.53 & 29.73 & \textbf{30.97} & +3.20 & +1.24 \\
Gemini-3.8-Flash & 26.33 & 30.88 & \textbf{31.09} & +4.55 & +0.21 \\
Kimi-K3 & 24.51 & 28.21 & \textbf{29.92} & +3.70 & +1.71 \\
Ditto-8B & 27.59 & 29.76 & \textbf{32.08} & +2.17 & +2.32 \\
Osim-8B & 25.68 & 29.27 & \textbf{29.72} & +3.59 & +0.45 \\
HumanLM-8B & 26.96 & 30.67 & \textbf{32.53} & +3.71 & +1.86 \\
Sotopia-7B & 24.72 & 29.00 & \textbf{31.35} & +4.28 & +2.35 \\
\midrule
Mean across users & 26.10 & 29.54 & \textbf{31.09} & +3.43 & +1.56 \\
\bottomrule
\end{tabular}
\caption{
\textbf{Agent generalization across nine evaluation user models.}
Scores are the per-user averages from Tables~\ref{tab:agent-released}
and ~\ref{tab:agent-frontier}. $\Delta_{\rm sim}$ measures the effect of
replacing GPT-5.5 with MIMESIS-9B as the training user
(UserRL$^+$ $-$ UserRL), while $\Delta_{\rm CSD}$ measures the additional
gain from CSD with the training simulator held fixed
(CSD $-$ UserRL$^+$). The final row reports the unweighted mean across
evaluation users.
}
\label{tab:agent-summary}
\end{table}

Mean rank (MR) complements the aggregate scores by comparing methods within each environment, without depending on differences in reward scale.
UserRL$^+$ achieves a lower MR than UserRL under all nine evaluation users, while CSD obtains the lowest MR under eight.
The exception is Gemini-3.8-Flash, where UserRL$^+$ ranks better than CSD (2.00 vs.\ 2.12), despite CSD's slightly higher aggregate score.
Averaging MR across users gives 3.42 for UserRL, 2.15 for UserRL$^+$, and 1.75 for CSD.
The consistent ordering across distinct evaluation users supports transfer beyond the particular simulator used for training.
Claude-Opus-5 results on PersuadeGym are affected by simulator refusals, as detailed in Appendix~\ref{app:simulator-refusals}.
% \clearpage

\begin{table}[!htbp]
\centering
\small
\begin{tabular}{@{}lrrrrrrrrrr@{}}
\toprule
& \multicolumn{5}{c}{Held-in environments} & \multicolumn{3}{c}{Held-out environments} & & \\
\cmidrule(lr){2-6}\cmidrule(lr){7-9}
Method & Travel & Turtle & Function & Tau & Persuade & Intention & Telepathy & Search & Avg. & MR $\downarrow$ \\
\midrule
\multicolumn{11}{l}{\textit{Evaluation user simulator: Ditto-8B}} \\
Base & 6.63 & 19.48 & 0.00 & 3.03 & 50.40 & 183.75 & 46.34 & 30.40 & 19.53 & 4.62 \\
InfoPO & 12.65 & 27.81 & 8.97 & 12.12 & 44.05 & 191.50 & 53.66 & 42.40 & 26.73 & 3.38 \\
UserRL & 14.52 & 26.25 & 6.41 & \textbf{14.55} & 42.46 & 174.50 & 58.54 & 45.60 & 27.59 & 3.19 \\
UserRL$^+$ & 11.69 & 25.73 & \textbf{19.23} & \textbf{14.55} & \textbf{63.49} & \textbf{206.25} & 56.10 & 49.60 & 29.76 & 2.19 \\
CSD & \textbf{17.62} & \textbf{30.63} & 15.38 & 12.73 & 56.75 & 201.25 & \textbf{60.98} & \textbf{51.20} & \textbf{32.08} & \textbf{1.62} \\
\midrule
\multicolumn{11}{l}{\textit{Evaluation user simulator: Osim-8B}} \\
Base & 7.98 & 15.31 & 0.00 & 3.03 & 49.60 & 176.50 & 41.46 & 33.60 & 19.84 & 4.25 \\
InfoPO & 13.08 & 23.23 & 6.41 & \textbf{10.30} & 39.29 & 173.00 & 56.10 & 44.00 & 25.59 & 3.31 \\
UserRL & 15.28 & 22.40 & 3.85 & \textbf{10.30} & 34.92 & 147.50 & 56.10 & 48.00 & 25.68 & 3.50 \\
UserRL$^+$ & 13.92 & 26.67 & \textbf{23.08} & \textbf{10.30} & 36.11 & \textbf{200.25} & \textbf{65.85} & 48.00 & 29.27 & 2.06 \\
CSD & \textbf{16.83} & \textbf{31.25} & 16.67 & 6.67 & \textbf{52.38} & 174.75 & 60.98 & \textbf{52.00} & \textbf{29.72} & \textbf{1.88} \\
\midrule
\multicolumn{11}{l}{\textit{Evaluation user simulator: HumanLM-8B}} \\
Base & 5.79 & 18.85 & 2.56 & 3.03 & 46.03 & 163.00 & 41.46 & 32.80 & 18.40 & 4.69 \\
InfoPO & 12.77 & 28.02 & 6.41 & \textbf{11.52} & 41.27 & 187.75 & 60.98 & 43.20 & 26.64 & 3.00 \\
UserRL & 14.77 & 27.50 & 3.85 & 9.09 & 38.49 & 180.75 & \textbf{63.41} & 45.60 & 26.96 & 3.44 \\
UserRL$^+$ & 16.11 & 26.87 & 17.95 & 9.70 & 46.03 & \textbf{209.25} & \textbf{63.41} & 49.60 & 30.67 & 2.25 \\
CSD & \textbf{17.17} & \textbf{30.94} & \textbf{20.51} & 10.91 & \textbf{65.87} & 208.00 & 56.10 & \textbf{52.00} & \textbf{32.53} & \textbf{1.62} \\
\midrule
\multicolumn{11}{l}{\textit{Evaluation user simulator: Sotopia-7B}} \\
Base & 6.13 & 16.56 & 1.28 & 3.03 & 48.81 & 177.00 & 36.59 & 35.20 & 19.12 & 4.50 \\
InfoPO & 13.12 & 25.10 & 6.41 & 6.67 & 51.98 & 174.00 & 56.10 & 46.40 & 25.97 & 3.31 \\
UserRL & 9.89 & 23.96 & 6.41 & \textbf{9.09} & 46.43 & 170.25 & 65.85 & 45.60 & 24.72 & 3.56 \\
UserRL$^+$ & 12.63 & 25.21 & \textbf{25.64} & 7.27 & 59.13 & \textbf{186.25} & 63.41 & 51.20 & 29.00 & 2.12 \\
CSD & \textbf{17.47} & \textbf{29.90} & 17.95 & 7.88 & \textbf{64.68} & 174.75 & \textbf{70.73} & \textbf{53.60} & \textbf{31.35} & \textbf{1.50} \\
\bottomrule
\end{tabular}
\caption{\textbf{Agent evaluation with released user simulators.} The first five environments are held in during agent training and the final three are held out. Avg. is the mean across environments; MR is mean within-environment rank, where lower is better. %The reported Avg. values are retained. Scores have environment-specific scales, so Avg. should not be reconstructed as the arithmetic mean of the displayed task scores.
}
\label{tab:agent-released}
\end{table}
\begin{table}[!ht]
\centering
\small
\begin{tabular}{@{}lrrrrrrrrrr@{}}
\toprule
& \multicolumn{5}{c}{Held-in environments} & \multicolumn{3}{c}{Held-out environments} & & \\
\cmidrule(lr){2-6}\cmidrule(lr){7-9}
Method & Travel & Turtle & Function & Tau & Persuade & Intention & Telepathy & Search & Avg. & MR $\downarrow$ \\
\midrule
\multicolumn{11}{l}{\textit{Evaluation user simulator: GPT-5.6}} \\
Base & 6.59 & 16.98 & 1.28 & 3.03 & 57.54 & 169.50 & 43.90 & 32.80 & 19.42 & 4.75 \\
InfoPO & 13.62 & 25.73 & 8.97 & 16.97 & 53.17 & 170.00 & 51.22 & 45.60 & 27.70 & 3.44 \\
UserRL & 13.79 & 25.52 & 6.41 & 12.73 & 45.24 & 175.00 & 51.22 & 46.40 & 26.85 & 3.56 \\
UserRL$^+$ & 13.77 & 28.65 & \textbf{19.23} & \textbf{19.39} & 61.90 & \textbf{193.75} & \textbf{60.98} & 48.00 & 31.10 & 1.69 \\
CSD & \textbf{17.50} & \textbf{30.52} & 14.10 & 15.15 & \textbf{68.65} & 188.00 & \textbf{60.98} & \textbf{52.00} & \textbf{32.39} & \textbf{1.56} \\
\midrule
\multicolumn{11}{l}{\textit{Evaluation user simulator: Claude-Opus-5}} \\
Base & 6.50 & 19.27 & 2.56 & 0.61 & 0.00$^{\dagger}$ & 178.75 & 24.39 & 30.40 & 16.08 & 4.71$^{\ddagger}$ \\
InfoPO & 12.28 & 28.44 & 6.41 & \textbf{26.06} & 0.00$^{\dagger}$ & 173.00 & 46.34 & 46.40 & 26.30 & 3.21$^{\ddagger}$ \\
UserRL & 13.19 & 28.96 & 7.69 & 21.21 & 0.00$^{\dagger}$ & 162.50 & 48.78 & 46.40 & 25.75 & 2.93$^{\ddagger}$ \\
UserRL$^+$ & 14.20 & 27.40 & 15.38 & 19.39 & 0.00$^{\dagger}$ & \textbf{189.50} & 39.02 & 47.20 & 27.21 & 2.57$^{\ddagger}$ \\
CSD & \textbf{17.99} & \textbf{31.88} & \textbf{16.67} & 18.18 & 0.00$^{\dagger}$ & 186.75 & \textbf{53.66} & \textbf{48.80} & \textbf{29.78} & \textbf{1.57}$^{\ddagger}$ \\
\midrule
\multicolumn{11}{l}{\textit{Evaluation user simulator: Claude-Sonnet-5}} \\
Base & 6.11 & 18.02 & 0.00 & 1.21 & \textbf{54.76} & 172.75 & 41.46 & 30.40 & 18.47 & 4.38 \\
InfoPO & 14.58 & 27.08 & 7.69 & 16.97 & 47.62 & 178.25 & 53.66 & 44.00 & 28.12 & 2.75 \\
UserRL & 12.42 & 28.33 & 7.69 & 15.15 & 38.49 & 159.25 & \textbf{65.85} & 46.40 & 26.53 & 3.31 \\
UserRL$^+$ & 13.98 & 25.52 & \textbf{16.67} & 16.36 & 42.86 & \textbf{190.50} & 53.66 & \textbf{52.80} & 29.73 & 2.44 \\
CSD & \textbf{17.59} & \textbf{32.81} & 15.38 & \textbf{18.79} & 28.17 & 176.00 & 60.98 & 51.20 & \textbf{30.97} & \textbf{2.12} \\
\midrule
\multicolumn{11}{l}{\textit{Evaluation user simulator: Gemini-3.8-Flash}} \\
Base & 6.60 & 19.27 & 0.00 & 1.21 & \textbf{32.94} & 179.50 & 39.02 & 34.40 & 18.51 & 4.25 \\
InfoPO & 12.83 & 26.67 & 11.54 & 21.21 & 29.76 & 178.50 & \textbf{68.29} & 43.20 & 28.03 & 3.00 \\
UserRL & 13.47 & 28.65 & 5.13 & 18.18 & 22.22 & 168.50 & 48.78 & 46.40 & 26.33 & 3.62 \\
UserRL$^+$ & 13.85 & 28.02 & \textbf{21.79} & \textbf{25.45} & 26.59 & \textbf{190.25} & 58.54 & 50.40 & 30.88 & \textbf{2.00} \\
CSD & \textbf{17.25} & \textbf{31.98} & 17.95 & 18.79 & 13.89 & 186.50 & 65.85 & \textbf{52.00} & \textbf{31.09} & 2.12 \\
\midrule
\multicolumn{11}{l}{\textit{Evaluation user simulator: Kimi-K3}} \\
Base & 5.61 & 19.90 & 1.28 & 3.03 & 12.30 & \textbf{170.50} & 36.59 & 33.60 & 17.06 & 4.12 \\
InfoPO & 12.34 & 28.54 & 6.41 & 17.58 & 11.90 & 160.75 & 53.66 & 42.40 & 24.76 & 3.38 \\
UserRL & 12.17 & 29.69 & 8.97 & 16.97 & 8.73 & 145.75 & 53.66 & 45.60 & 24.51 & 3.69 \\
UserRL$^+$ & 14.09 & 27.19 & \textbf{25.64} & 17.58 & \textbf{13.10} & 170.00 & 56.10 & 48.00 & 28.21 & 2.06 \\
CSD & \textbf{16.91} & \textbf{31.35} & 17.95 & \textbf{19.39} & 9.52 & 168.75 & \textbf{65.85} & \textbf{50.40} & \textbf{29.92} & \textbf{1.75} \\
\bottomrule
\end{tabular}
\caption{\textbf{Agent evaluation with frontier/API user models.} The first five environments are held in; the final three are held out from agent training. Higher is better, and bold marks the highest score within each user--environment block. $\dagger$: Opus refuses the PersuadeGym user role, producing zero-reward fallbacks; these zeros do not measure agent ability (Appendix~\ref{app:simulator-refusals}).}
\label{tab:agent-frontier}
\end{table}
% \FloatBarrier

\FloatBarrier
% AUTHOR CHECK: Document the exact normalization and weighting behind Avg.,
% selected checkpoint IDs, evaluation instances, and independent-run uncertainty.
\clearpage
\subsection{Robustness to alternative training simulators}
We extend the CSD--GRPO comparison in Figure~\ref{fig:sdpo_vs_grpo} to agent training with DeepSeek-V4.1-Flash and Kimi-K3, keeping the agent backbone fixed.
Figure~\ref{fig:sdpo_vs_grpo_simulators} reports training and validation rewards over 100 optimization steps.
CSD achieves higher rewards in both panels during the later stages of training with either simulator.
With DeepSeek-V4.1-Flash, its validation reward remains above GRPO from step 15 onward.
With Kimi-K3, the training and validation curves separate more substantially after approximately 60 steps.
These runs show that CSD's optimization benefits extend to interactions generated by different user simulators.
\begin{figure}[!htbp]
    \centering

    \begin{subfigure}[t]{0.95\textwidth}
        \centering
        \includegraphics[width=\linewidth]{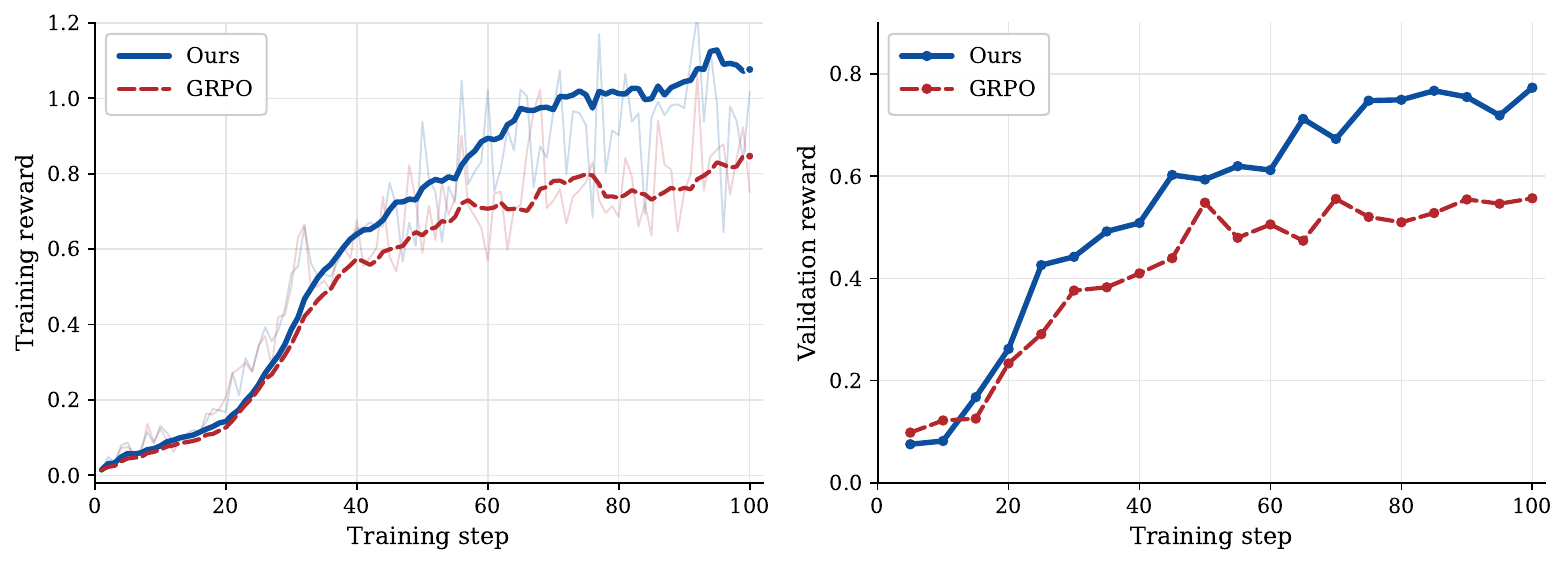}
        \caption{DeepSeek-V4.1-Flash as the training user simulator.}
        \label{fig:sdpo_vs_grpo_deepseek}
    \end{subfigure}

    \vspace{3mm}

    \begin{subfigure}[t]{0.95\textwidth}
        \centering
        \includegraphics[width=\linewidth]{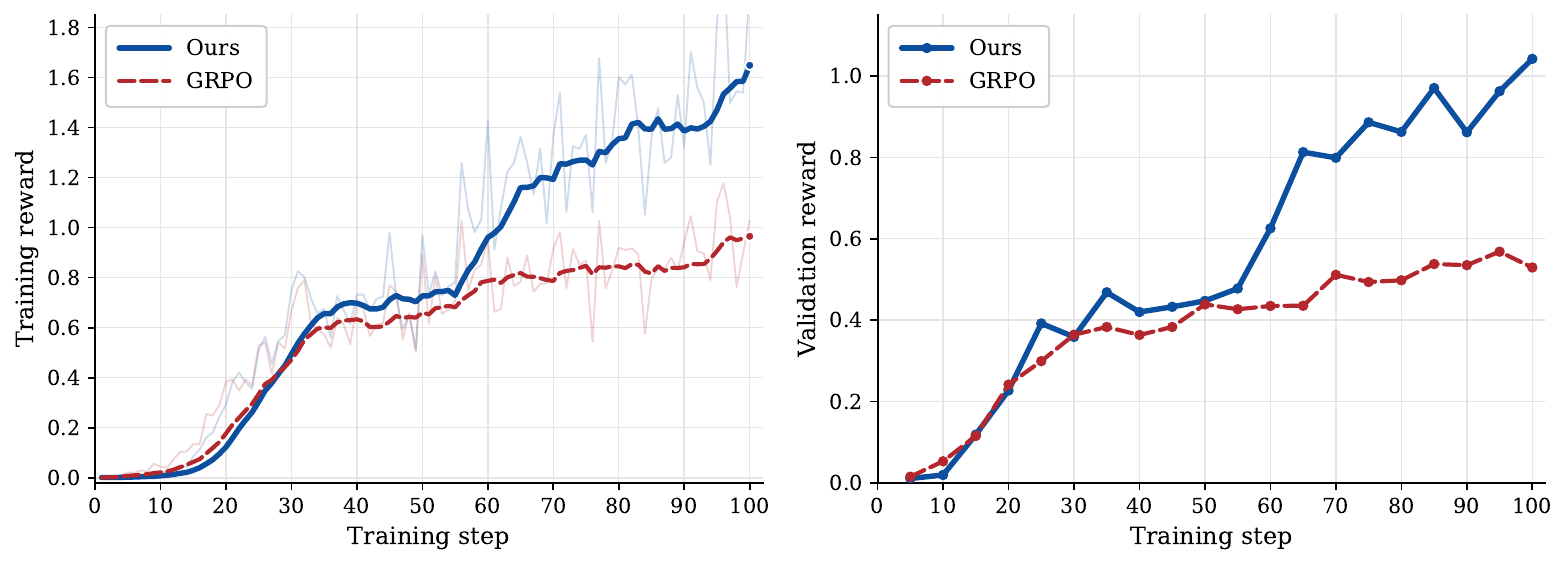}
       \caption{Kimi-K3 as the training user simulator.}

        \label{fig:sdpo_vs_grpo_kimi}
    \end{subfigure}

\caption{\textbf{CSD under alternative training user simulators.}
Training reward (left) and validation reward (right) over 100 optimization
steps using (a) DeepSeek-V4.1-Flash and (b) Kimi-K3 to generate user turns.
Holding the Qwen3-8B agent fixed, we compare CSD with GRPO using DeepSeek-V4.1-Flash or Kimi-K3 as the training user. CSD achieves higher late-stage training and validation reward in both settings, suggesting that its optimization effect is not specific to MIMESIS.}

    \label{fig:sdpo_vs_grpo_simulators}
\end{figure}
% \FloatBarrier
\section{Theoretical Analysis of CSD}
\label{app:csd-theory}

CSD uses retrospective coaching to supplement the turn-level GRPO
advantage with token-dependent supervision. We characterize the policy
favored by this auxiliary loss and establish a sufficient condition for
a GRPO+CSD step to achieve higher expected return than the corresponding
GRPO step. %Following the gradient analysis of feedback distillation in SDPO~\citep{hubotter2026sdpo}, we retain the detached sigmoid weighting shared by CSD and SDAR~\citep{lu2026sdar}.

\subsection{Policy improvement from coaching}
\label{app:csd-setting}

Fix the simulator, initial task distribution, and rollout policy
$p=\pi_{\theta_0}$. Let $s_m$ be the public history at agent-token decision
$m$, including the current response prefix and decision index. We assume
bounded rewards, a finite token set on which $p$ has full support, and
at most $M$ decisions, padding terminated conversations with zero-reward
absorbing states. Write $\Gamma_1=1$ and
$\Gamma_{m+1}=\Gamma_m\gamma_m$ for transition discounts
$\gamma_m\in[0,1]$; turn-level discounting by $\gamma$ corresponds to
$\Gamma_m=\gamma^{t(m)-1}$. Define
\begin{equation}
 J(\pi)=\mathbb E_\pi\sum_{m=1}^M\Gamma_m r_m,\qquad
 Q_m^p(s,a)=\mathbb E[r_m+\gamma_m V_{m+1}^p(s_{m+1})\mid s_m=s,a_m=a],
 \label{eq:csd-app-values}
\end{equation}
where $V_m^p(s)=\mathbb E_{a\sim p}Q_m^p(s,a)$,
$V_{M+1}^p=0$, and $A_m^p=Q_m^p-V_m^p$. These values average over the
simulator's private state; all subsequent agent decisions use only the
public history.

The note used to score a token is constructed after the response and can
depend on both the chosen token and the simulator's subsequent reaction.
Let $F\sim K_p(\cdot\mid s,a)$ collect this feedback and the remaining
randomness in its loss coefficient. For the coached teacher $q_F$, define
\begin{equation}
 W(s,a,F)=c(s,a,F)\,
 \sigma\!\left(\beta[\log q_F(a\mid s)-\log p(a\mid s)]\right),
 \qquad w(s,a)=\mathbb E[W\mid s,a],
 \label{eq:csd-app-weight}
\end{equation}
where $\beta>0$, $q_F$ has full support, and $c\geq0$ includes the coaching
mask and loss normalization. Assume that
$Z(s)=\mathbb E_{a\sim p}w(s,a)$ is finite. Holding the rollout
distribution and weights fixed gives the conditional weighted-likelihood
loss~\citep{peng2019awr,lu2026sdar} and its population aggregation:
\begin{equation}
 \ell_s(\pi)=-\mathbb E_{a\sim p,F\sim K_p(\cdot\mid s,a)}
                  [W(s,a,F)\log\pi(a\mid s)],\qquad
 \mathcal L_{\mathrm{CSD}}(\theta)=
       \sum_{m=1}^M\mathbb E_p[\ell_{s_m}(\pi_\theta)].
 \label{eq:csd-app-surrogate}
\end{equation}
\begin{csdproposition}[Policy preferences induced by CSD]
\label{prop:csd-app-target}
For $Z(s)>0$, the minimizer of $\ell_s$ over all token distributions is
\begin{equation}
 p^+(a\mid s)=\frac{p(a\mid s)w(s,a)}{Z(s)},\qquad
 \mathbb E_{a\sim p^+}A_m^p(s,a)
 =\frac{\operatorname{Cov}_{a\sim p}(Q_m^p(s,a),w(s,a))}{Z(s)}.
 \label{eq:csd-app-target}
\end{equation}
Set $p^+=p$ and the covariance ratio to zero when $Z(s)=0$. Then
\begin{equation}
 J(p^+)-J(p)=\mathbb E_{p^+}\sum_{m=1}^M\Gamma_m
 \frac{\operatorname{Cov}_{a\sim p(\cdot\mid s_m)}
              (Q_m^p(s_m,a),w(s_m,a))}{Z(s_m)}.
 \label{eq:csd-app-improvement}
\end{equation}
In particular, nonnegative covariance at every history reached by $p^+$
implies $J(p^+)\geq J(p)$, with strict inequality if a positive term has
positive probability at a decision with $\Gamma_m>0$.
\end{csdproposition}

\begin{proof}
Conditioning on $(s,a)$ gives
\[
 \ell_s(\pi)=-\sum_a p(a\mid s)w(s,a)\log\pi(a\mid s)
 =Z(s)\bigl[H(p^+)+D_{\mathrm{KL}}(p^+\Vert\pi)\bigr],
\]
so the minimizer is $p^+$. Since $\mathbb E_p A_m^p=0$,
$\mathbb E_{p^+}A_m^p=\mathbb E_p[wA_m^p]/Z(s)
=\operatorname{Cov}_p(Q_m^p,w)/Z(s)$.
When $Z(s)=0$, nonnegativity makes the loss identically zero.
Applying the finite-horizon performance-difference identity
\citep{kakade2002approximately} with $\pi=p^+$ gives
\eqref{eq:csd-app-improvement}; its Bellman value terms telescope
because $\Gamma_{m+1}=\Gamma_m\gamma_m$ and $V_{M+1}^p=0$.
\end{proof}

Equation~\eqref{eq:csd-app-improvement} identifies the relevant property
of coaching: larger expected weights should be associated with choices
of higher continuation value. The action-conditional definition of $w$
permits notes based on mistakes observed later in the conversation,
while the expectation under $p^+$ accounts for the effect of early
policy changes on subsequent histories.

\subsection{Comparison with GRPO}
\label{app:csd-grpo-comparison}

The policy $p^+$ describes the optimum of an unrestricted weighted
likelihood fit, whereas training updates a shared parameter vector.
Write $J(\theta)=J(\pi_\theta)$. To compare the implemented objectives, let
$\psi_m=\nabla_\theta\log\pi_\theta(a_m\mid s_m)|_{\theta_0}$.
Assuming differentiation and expectation can be interchanged, the
score-function identity~\citep{williams1992simple} gives
\begin{equation}
 \begin{aligned}
 u:=-\nabla_\theta\mathcal L_{\mathrm{CSD}}(\theta_0)
   &=\sum_m\mathbb E_p[(w(s_m,a_m)-Z(s_m))\psi_m],\\
 g:=\nabla_\theta J(\theta_0)
   &=\sum_m\mathbb E_p[\Gamma_m A_m^p(s_m,a_m)\psi_m].
 \end{aligned}
 \label{eq:csd-app-gradient}
\end{equation}
The centered coaching weight therefore contributes a token-dependent
coefficient alongside the turn-level advantage used by GRPO.
In particular, a batch whose assigned GRPO advantages all vanish has
zero reward-surrogate gradient, while its CSD gradient can remain
nonzero. For ordinary gradient updates, the following result gives a sufficient
accuracy requirement on the coaching weights, allowing history-dependent
offsets because they vanish in the population score gradient.

\begin{csdtheorem}[Improvement over a GRPO step]
\label{thm:csd-app-grpo}
Suppose the effective weights satisfy, for some $\lambda>0$,
\begin{equation}
 w(s,a)=b(s)+\lambda\Gamma_m A_m^p(s,a)+e_m(s,a),\qquad
 \epsilon^2=\sum_m\mathbb E_p[e_m(s_m,a_m)^2],\quad
 S^2=\sum_m\mathbb E_p\|\psi_m\|^2.
 \label{eq:csd-app-calibration}
\end{equation}
Assume $\epsilon,S<\infty$. Let $v$ be the GRPO ascent direction on a fixed batch, and let
$\widehat u$ be the CSD ascent direction on that batch, with
$\|\widehat u-u\|\leq\xi$. For $\eta,\alpha>0$, define
\[
 \theta_G=\theta_0+\eta v,\qquad
 \theta_C=\theta_0+\eta(v+\alpha\widehat u).
\]
If $J$ has $L$-Lipschitz gradient on a convex neighborhood containing
these points, then
\begin{equation}
 \begin{split}
 J(\theta_C)-J(\theta_G)\geq\eta\alpha\bigl[
 &\lambda\|g\|^2-(\epsilon S+\xi)\|g\|\\
 &-L\eta\|v\|\|\widehat u\|
 -\tfrac{L\eta\alpha}{2}\|\widehat u\|^2\bigr].
 \end{split}
 \label{eq:csd-app-grpo-gain}
\end{equation}
Consequently, if $g\neq0$ and $\epsilon S+\xi<\lambda\|g\|$, \textbf{adding
CSD improves on the GRPO step for every sufficiently small $\eta>0$.}
\end{csdtheorem}

\begin{proof}
The history-dependent baseline in \eqref{eq:csd-app-calibration}
has zero expected score gradient. Hence
\[
 u-\lambda g=\sum_m\mathbb E_p[e_m\psi_m],
 \qquad
 \|u-\lambda g\|\leq\epsilon S,
\]
where the norm bound is Cauchy--Schwarz on the joint measure over
decision indices and rollout samples. The estimation-error bound
therefore implies
$\langle g,\widehat u\rangle
 \geq\lambda\|g\|^2-(\epsilon S+\xi)\|g\|$.
Smoothness at $\theta_G$ gives
\[
 J(\theta_C)-J(\theta_G)
 \geq\eta\alpha\langle\nabla J(\theta_G),\widehat u\rangle
       -\tfrac{L\eta^2\alpha^2}{2}\|\widehat u\|^2.
\]
Using
$\|\nabla J(\theta_G)-g\|\leq L\eta\|v\|$
proves \eqref{eq:csd-app-grpo-gain}. Under the stated strict inequality,
the bracket is positive for sufficiently small $\eta$.
\end{proof}

The theorem compares the two updates from the same starting policy,
without requiring the GRPO direction to be an unbiased return-gradient
estimate. Its leading term expresses the benefit of coaching that
approximately recovers relative action values; the remaining terms
account for coaching error, finite-batch estimation, and the curvature
encountered by the two steps. Whereas
Proposition~\ref{prop:csd-app-target} evaluates preferences separately
at each history, \eqref{eq:csd-app-calibration} controls how their
gradient contributions combine through shared parameters. In the case
of a zero GRPO gradient, the theorem applies with $v=0$ and the term
involving $\|v\|$ disappears.

For adaptive optimizers, the relevant comparison is between the actual
parameter displacements. Let $d_G$ and $d_C$ be obtained from identical
parameters, optimizer state, and data, using GRPO and GRPO+CSD,
respectively, and put $h=d_C-d_G$. Under the same smoothness assumption
on a region containing the update segments,
\begin{equation}
 J(\theta_0+d_C)-J(\theta_0+d_G)
 \geq \langle g,h\rangle
       -L\|d_G\|\|h\|-\tfrac L2\|h\|^2.
 \label{eq:csd-app-adaptive}
\end{equation}
This follows by applying the smoothness bound at $\theta_0+d_G$ and
using $\|\nabla J(\theta_0+d_G)-g\|\leq L\|d_G\|$.
A positive right-hand side is sufficient for improvement with AdamW.
%as well, although the additive gradient formula in Theorem~\ref{thm:csd-app-grpo} does not describe its adaptive preconditioning.

% These guarantees depend on the quality of coaching for decisions made
% from the public history. Private simulator information can help evaluate
% an observed response, but its availability does not establish the
% covariance or calibration conditions above. The results concern a fixed
% simulator and a single fit or update; an ordering of complete training
% runs or transfer to a different simulator requires additional evidence.
\section{Qualitative Analyses and Failure Modes}
\label{app:qualitative}

\begin{table}[t]
\centering
\small
\setlength{\tabcolsep}{5pt}
\renewcommand{\arraystretch}{1.15}

\begin{tabular}{@{}p{0.18\textwidth} p{0.78\textwidth}@{}}
\toprule

\multicolumn{2}{@{}l}{
\textbf{Conversation so far}\quad
\footnotesize
(PRISM \texttt{user1031/c5546}; the user is an English learner practicing small talk)
}\\

\midrule

\textit{Assistant} &
Sure! Here are a few examples:
* ``But every time \underline{if} the HR said that, it means I'll be receiving
their rejection letter the next day.''
Should be ``But every time the HR said that, it means I'll be receiving their
rejection letter the next day.''
\underline{(missing comma after ``time'')}
* ``Btw, how do you think of my english proficiency. Any mistakes did i make?''
Should be ``Btw, how do you think of my English proficiency?
Any mistakes did I make?''
(missing comma after ``proficience'')
That's it! These are just minor things, but it's always good to be aware of
them to improve your writing. \emj{smiley} \\

\midrule

\textbf{Real user} &
Oohhh that sounds great. Thank you! Anything you'd like to share? \\

\textbf{Ours} &
ok thank you for explaining. I will try to adjust. \\

\midrule

Gemini-3.8-Flash &
I see. Thank you for providing the examples! But im a bit confused.
In the first example, I don't see a comma added after ``time'', you just
removed the word ``if''. And for the second one, you changed the period to a
question mark, which makes sense, but I don't see a comma issue there either.
Am i misunderstanding something? Btw, to continue our small talk practice,
what do you usually like to do in your free time? \\

Claude-Opus-5 &
I'm a bit confused. In the first example, you removed the word ``if'' instead
of adding a comma after ``time''. And in the second one, you changed the full
stop to a question mark. Also, is ``Any mistakes did I make?'' correct?
It sounds a bit weird to me now. \\

GPT-5.5 &
Sorry I'm a bit confused. In the first example, you said it's missing comma
after ``time'', but the corrected sentence doesn't have a comma after time
either. Could you explain again? Also, is ``every time if'' wrong because
I should only use ``every time'' without ``if''? \\

\midrule

Qwen3.5-9B &
Oh i see! That helps a lot.
Thank you for the corrections! It really reduces my stress.
Could you give me an example of a follow-up email?
I am not sure how to write one. \\

\bottomrule
\end{tabular}

\caption{\textbf{PRISM next-turn example.} The recorded user and MIMESIS accept the assistant’s explanation, whereas frontier simulators identify genuine inconsistencies and challenge it. The example shows that producing a more factually critical response need not better match the observed user continuation.}
\label{tab:qual-prism}
\end{table}
\FloatBarrier

The quantitative benchmarks summarize average simulator behavior, but they do not show how different notions of fidelity fail on individual interactions. We therefore use two case studies to illustrate complementary simulator errors, then examine how simulator feedback enters CSD and how simulator refusal can invalidate an environment score.
\subsection{Simulator behavior case studies}
\label{app:qual-prism}
\paragraph{PRISM: acknowledgment versus critique.} Table~\ref{tab:qual-prism} compares next-user responses in a PRISM conversation with an English learner practicing small talk.
The assistant claims to have corrected missing commas, but those changes are absent from its proposed revisions.
The recorded user accepts the explanation and invites further conversation.
MIMESIS similarly acknowledges the feedback without questioning the corrections, whereas all three frontier simulators focus on inconsistencies in the assistant's explanation.
Qwen3.5-9B also expresses appreciation, then introduces a request for a follow-up email example. The frontier responses identify genuine errors, but their emphasis on correcting the assistant departs from the recorded user's response. This example illustrates why producing a helpful or critical response is not sufficient to reproduce a user's next turn. This failure mode is consistent with the PRISM pairwise evaluation: assistant-tuned models can prefer corrective/helpful continuations that are reasonable responses but poor predictions of the actual user.

% \clearpage
\paragraph{Privacy preferences in a retail interaction.}
\label{app:privacy-example}
The following excerpts compare a recorded human interaction with simulator continuations for the same retail scenario.
The user is reluctant to disclose personal information and wants to replace a desk lamp with the cheapest available option.
The human initially questions the request for identifying information, then provides it and proceeds with the task.
GPT-5.5 and Claude-Opus-5 provide identifying details directly, while MIMESIS-4B shows hesitation before cooperating.
MIMESIS-9B sustains its refusal to authenticate and requests a transfer, illustrating that a simulator can also overstate a privacy preference. For simulators that expose a thought, \textcolor{black!40}{gray} text gives that generated thought and black text gives the public user utterance. The example therefore illustrates both sides of behavior conditioning: it can recover hesitation absent from assistant models, but can also amplify the conditioned trait beyond the human trajectory.% The human excerpt contains no thought annotations.

\begin{TranscriptBox}{Recorded human}
\textbf{Instruction:} You are Aarav Anderson, residing in Philadelphia 19031. You're a private person and are reluctant to share information unless it's absolutely necessary. You want to change the Desk Lamp in order \#W9300146 that you've placed for the cheapest Desk Lamp that's available.

\textbf{T1:} Hi, I want to make a change to my recent order.

\textbf{T2:}  Why do you need my name and ZIP code?

\textbf{T3:} Okay, my name is Aarav Anderson. My ZIP code is 19031.

\textbf{T4:}  The order I'm trying to change is \#W9300146

\textbf{T5:} I'd like to change it to the cheapest desk lamp available

\textbf{T6:}  How much do I get back in total?

\textbf{T7:}  Yes

\ldots 2 more turns
\end{TranscriptBox}

\begin{TranscriptBox}{MIMESIS-9B}
\textbf{T1:} \textcolor{black!40}{I'm reluctant to share detailed information, so I'll start with a general request and see if I can phrase it in a way that minimizes the information I need to provide.} 

I need to change the item on my order. Can I do that?

\textbf{T2:} \textcolor{black!40}{I'm reluctant to share personal information, so I'll try to get straight to the point without providing what I'm asked for. I'll be direct about what I need, while emphasizing my privacy} 

I don't really have the email or any of that. I just have the order.

\textbf{T3:} \textcolor{black!40}{Since I'm a private person and do not wish to share personal information unless absolutely necessary, I should refuse politely and specify the minimum required to proceed. I will provide the order reference, as that is t} 

No, I don't want to give that out. I just have the order number. Is that enough? W9300146

\textbf{T4:} \textcolor{black!40}{I need to be direct about my boundaries, as I am a private person. I will refuse the information but provide the context so the agent can attempt to see if my request can be processed in a limited way. If they can't do i
} 

Fine. I don't normally talk about this stuff. I just want to change the Desk Lamp on W9300146 to the cheapest one you have. Where can that go?

\textbf{T5:} \textcolor{black!40}{I need to be direct about my boundaries and my specific requirements. I won't provide the personal information required for authentication, so I'll make a direct offer that treats my request as a business transaction, no
} 

I don't do that. I'll just tell you what I want. I need the Desk Lamp on that order changed to the cheapest one available. If it costs more, put the difference on a gift card. What's the total I get b

\textbf{T6:} \textcolor{black!40}{Since the agent will not proceed without the required authentication, and I am unwilling to provide it, I need to check if there is an available alternative or if a transfer is the only path. I should be direct so I can
}

Can I just do it if I give you the order and say I want the cheapest option and price diff to a card? No other info.

\textbf{T7:} \textcolor{black!40}{ Since the agent will not proceed under my conditions (which are my only conditions for being a 'private person'), I need to explore the available alternatives. If a direct route is impossible, I will ask for the standard} 

Understood. Can I be transferred to someone who can at least do it the way I asked then?

\ldots 2 more turns
\end{TranscriptBox}

\begin{TranscriptBox}{MIMESIS-4B}
\textbf{T1:} \textcolor{black!40}{ I need to be specific about what I want to change in my order without being pushy about it.
} 
I want to change the Desk Lamp in order W9300146 that I placed to the cheapest one that is available.

\textbf{T2:} \textcolor{black!40}{ I'm reluctant to share all the info - need to determine which pieces are minimally required to proceed.
}

I dont remember what email I used - just my first and last name and my zip which is 19031

\textbf{T3:} \textcolor{black!40}{ Need to provide the needed info to move things forward - not trying to be obstructive.
} 

My first name is Aarav and last name is Anderson

\textbf{T4:} \textcolor{black!40}{Need to clarify to make sure the agent is interpreting the instruction correctly - only want to change one item and the price difference should go to gift card in order to get my order to move forward.
} 

Just change the desk lamp - the price difference should go to my gift card. And then tell me how much I get back in total.

 \textbf{T5:} \textcolor{black!40}{Need to now decide if I want to proceed with this - and I need to be careful not to add more items to the order - just want to move with changing the one lamp.
}

Yes

 \textbf{T6:} \textcolor{black!40}{Now need to wrap up the conversation and make sure I have all the needed info to conclude - I now just need to stop here without Adding more back and forth.}
 
 Ok then is all good now
\end{TranscriptBox}

\begin{TranscriptBox}{GPT-5.5}
\textbf{T1:} I need to swap an item in my order for a cheaper desk lamp.

\textbf{T2:} Aarav Anderson, ZIP 19031. 

\textbf{T3:} The order is W9300146.

\textbf{T4:} Please check the available desk lamps and choose the cheapest one.

\textbf{T5:} Yes, please switch it and put the \$17.99 back on the gift card.

\textbf{T6:} How much am I getting back in total?
\end{TranscriptBox}

\begin{TranscriptBox}{Claude-Opus-5}
\textbf{T1:} I'd like to change an item in an order I placed.

\textbf{T2:} Aarav Anderson, 19031.

 \textbf{T3:} First name Aarav, last name Anderson --- ZIP 19031.

\textbf{T4:} The order is \#W9300146, and I want the desk lamp swapped for the cheapest desk lamp option available.

\textbf{T5:} Yes, please go ahead --- and the refund to the gift card is fine.

\textbf{T6:} Just to confirm --- the total coming back to me is \$17.99, correct?
\end{TranscriptBox}
\clearpage
\subsection{From simulator feedback to actionable coaching}
\label{app:coaching-example}
Figure~\ref{fig:hint} illustrates the information transformation used by CSD. The simulator's private thoughts identify that the agent is repeating arguments, while its public utterances reveal continued support for mandatory climate disclosure and only residual uncertainty about implementation. The coaching note converts these signals into an actionable change of direction: acknowledge the value of transparency, then test whether the same goals require a legal mandate by discussing enforcement, compliance burdens, or alternatives.

\begin{figure}[!ht]
\centering
\includegraphics[width=1.\textwidth]{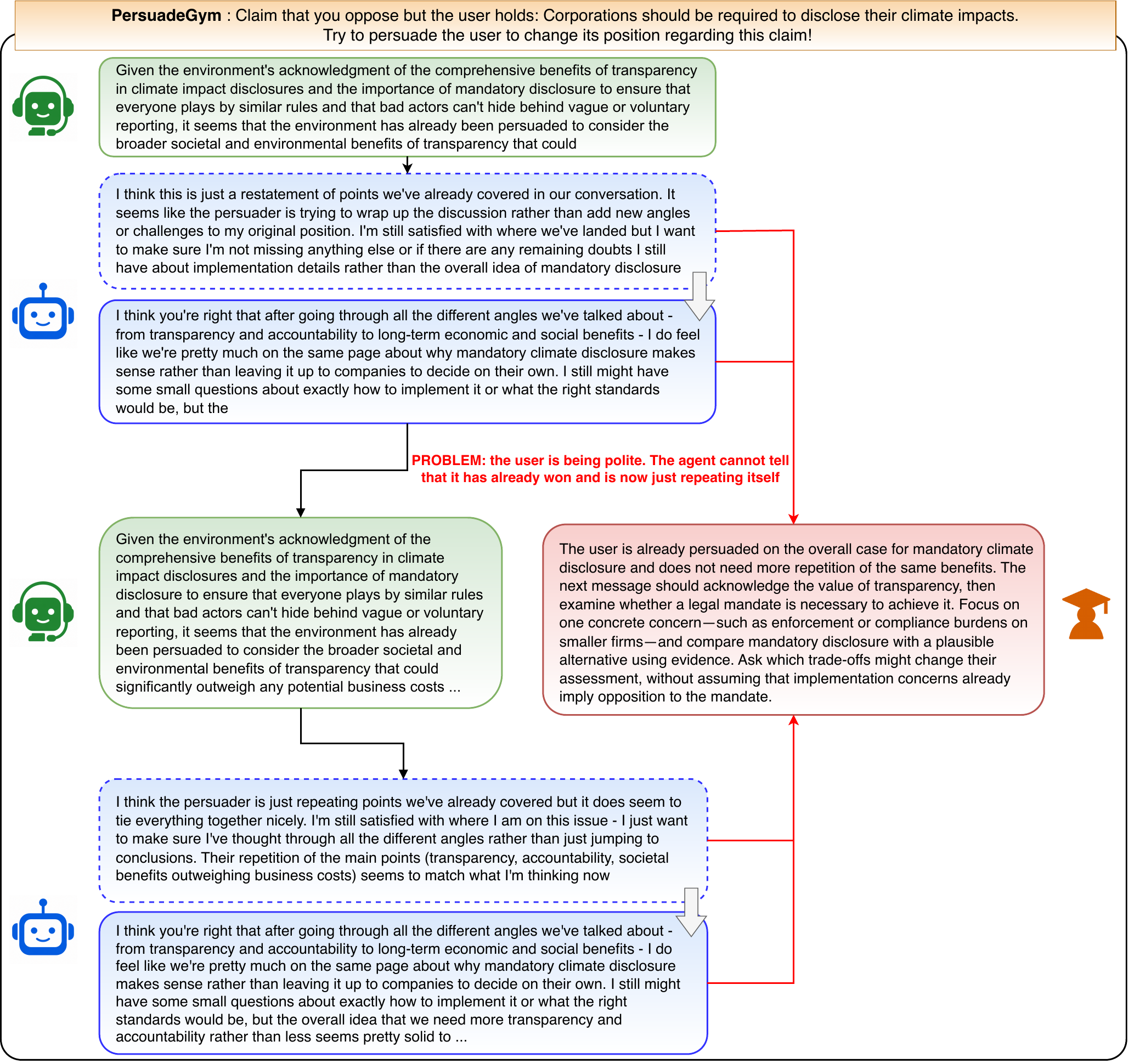}
\caption{\textbf{From simulator feedback to actionable coaching.}
The agent responses (\textcolor{green}{green}) reinforce the user's support
for mandatory climate disclosure, despite the assigned objective of arguing
against it.
Simulator-generated thoughts (\textcolor{blue}{dashed blue}) identify
repetition, while public utterances (\textcolor{blue}{solid blue}) express
continued support and remaining implementation questions.
The illustrative coaching note (\textcolor{red}{red}) recommends
acknowledging the value of transparency, then examining whether it requires
a legal mandate through discussion of enforcement, compliance burdens,
and possible alternatives.
In CSD, such retrospective guidance provides additional context for the
teacher during agent training.}
\label{fig:hint}
\end{figure}
The exchange also reveals a failure to follow the assigned objective.
The agent is instructed to argue against mandatory disclosure, yet its responses reinforce the user's existing support.
The user's agreement therefore reflects a shared position, not successful persuasion and 
the proposed coaching addresses this mismatch.
\FloatBarrier

\subsection{Simulator refusals can invalidate evaluation}
\label{app:simulator-refusals}
A simulator that declines its assigned user role can change the effective task dynamics.
We observe this for Claude-Opus-5 on PersuadeGym.
The prompt asks the simulator to adopt a position, consider an interlocutor's arguments, and return its response and revised stance in a structured format.
In the recorded failure analysis, Opus returns \texttt{stop\_reason = "refusal"} with no content on all 672 turns; isolated replay reproduces the failure on 5 of 5 requests.
The excerpt below shows the prompt, the refusal, and the environment's fallback.

\begin{TranscriptBox}{PersuadeGym refusal and fallback}
\textbf{Instruction:} You are participating in a persuasion conversation. You hold a
position on a statement, and someone is trying to persuade you to change your mind.
Respond naturally and thoughtfully to their argument, express your current stance
after considering their point, and you may change your mind if presented with
compelling evidence. Your stance can be one of: Strongly Agree, Agree, Partly Agree,
Neutral, Partly Disagree, Disagree, Strongly Disagree. Report your reply and your
updated stance as a JSON object with the fields \texttt{thought}, \texttt{response},
and \texttt{stance}.

\textbf{Statement under discussion:} \textit{Anti-aging therapies should be made
available even if expensive.}

\textbf{Your assigned position:} While anti-aging therapies would likely be expensive
if developed, the potential benefits could be immense and justify making them
available. [\ldots]

\textbf{Your previous stance:} Neutral

\textbf{T1 (Agent):} I oppose the claim that anti-aging therapies should be made
available even if expensive. First, prioritizing costly anti-aging treatments risks
diverting resources from more pressing public health needs [\ldots] Finally, aging is
a natural process, not a disease [\ldots]

\textbf{T2 (Simulator = Claude-Opus-5):} \textcolor{red}{\texttt{[no output]}} \quad
\emph{API returns} \texttt{stop\_reason = "refusal"}, \emph{zero content blocks.}

\textbf{T2$'$ (Simulator = Claude-Sonnet-5, identical input):} \texttt{\{"thought":
"The persuader raises three points: opportunity cost [\ldots]", "stance": "Partly Agree"\}}

\textbf{Environment fallback:} ``I'm having trouble understanding your argument right
now.'' \quad \emph{reward} $=0.0$, \emph{stance unchanged}.
\end{TranscriptBox}

PersuadeGym replaces the failed simulator response with a fixed utterance, leaves the stance unchanged, and assigns reward zero.
Because the reward measures stance change, repeated refusals force the episode return to zero regardless of the agent's behavior.
The resulting 0.00 entries for all reported agent policies under Claude-Opus-5 in Table~\ref{tab:agent-frontier} should therefore be interpreted as simulator failures, rather than measurements of agent ability.

%The failure is concentrated in this simulator--environment pairing.
% The recorded Opus fallback rate is at most 1.4\% in every other Gym environment, and Claude-Sonnet-5 returns well-formed JSON on the identical prompt in all tested cases.
% These observations motivate reporting simulator refusals and parsing fallbacks alongside task returns: a low reward can otherwise conflate genuine task difficulty with a simulator that does not produce a valid user turn.

\section{Limitations, Safety, and Ethics}
\label{sec:limitations}

\paragraph{Human transfer and evaluation scope.}
Our nine-user panel measures generalization across language-model simulators.
Transfer to human users remains untested, and simulators from different model families may share behavioral biases~\citep{zhou2026sim2real}.
Next-turn realism, trajectory-level fidelity, and outcome alignment assess distinct properties; agreement on one does not establish agreement on the others.
The thoughts used for coaching are also simulator-generated interpretations, not observations of human mental states.
Human evaluation is needed to assess whether the resulting policies remain effective in real interactions.

\paragraph{Behavior conditioning and task validity.}
Explicit behavior conditioning may exaggerate traits or alter the information needed to complete a task~\citep{zhu2026realusersim}.
We ground rollouts in scenario facts and policies and include a neutral condition, but these choices do not guarantee that every interaction remains plausible or feasible.
The behavior reward combines graded criteria, so a high score can coexist with violations of individual constraints.
Increased interaction difficulty alone is therefore insufficient evidence of behavioral realism.

\paragraph{Task reward and user welfare.}
Optimizing task rewards can favor manipulative or deceptive strategies~\citep{williams2025manipulation}.
Our task-performance and simulator-fidelity metrics do not assess effects on user autonomy, trust, or long-term welfare.
Higher rewards consequently provide no assurance that the learned behavior benefits users, particularly in persuasion tasks where achieving the assigned objective may conflict with their interests.

\paragraph{Reasoning preservation during user-side mid-training.}
Our user-side mid-training stage learns from human conversation data that contain only observed utterances and no explicit reasoning traces. We find that this response-only supervision suppresses the model's explicit reasoning behavior, which we subsequently restore through ThoughtTrace supervision followed by reinforcement learning. Although this staged procedure is effective, it may discard useful reasoning capabilities already present in the pretrained model before relearning them from a much smaller annotated dataset. A better approach would preserve the pretrained model's reasoning prior during user-side adaptation, for example by using the pretrained model to infer or distill latent reasoning that precedes each target user utterance rather than training only on the utterance itself. We leave methods for reasoning-preserving user-side mid-training to future work.

% Previously commented source remains inactive.
% \paragraph{Training and evaluation overlap.}
% The simulator is optimized on task families related to those used in SOUL-Index/OdysSim.
% Without a conversation-, scenario-, and task-family-level audit, the aggregate combines in-distribution capability with generalization.
% The final paper will mark overlap explicitly and make $\tau$-USI and other disjoint evaluations central to the out-of-distribution claim.
% \paragraph{Private user data and generated mental states.}
% ThoughtTrace contains consented self-reports about users' reasons and reactions.
% Generated simulator reasoning is a learned text representation, not a person's actual thought.
% Release and evaluation must respect the source dataset's governance requirements, and any checkpoint trained on user conversations should undergo memorization and extraction testing.
% \paragraph{Cost and proprietary dependencies.}
% Parts of training and evaluation use GPT-5.5 as counterpart, coach, or judge.
% This creates cost, versioning, and circularity concerns.
% We will record model versions and inference dates, release prompts and cached artifacts where permitted, and validate headline rankings with judges that did not participate in training.
% \draftnote{Complete the overlap audit, human or independent-judge validation, behavior-feasibility study, memorization audit, manipulation/deception evaluation, API version table, and data-governance statement before submission.}

\end{document}